\PassOptionsToPackage{backref=page,bookmarks=false}{hyperref}
\documentclass[final]{colt2024} % Include author names

\usepackage[utf8]{inputenc}				% allow utf-8 input
\usepackage[T1]{fontenc}				% use 8-bit T1 fonts
\usepackage{xr-hyper}                   % hyperlink for xr

\usepackage{hyperref}	% hyperlink 
\usepackage{times}
\usepackage{amsfonts}					% blackboard math symbols
\usepackage{nicefrac}					% compact symbols for 1/2, etc.
\usepackage{microtype}					% microtypography

\usepackage{graphicx}
\usepackage{caption}
\graphicspath{{Figures/}}
\DeclareGraphicsExtensions{.pdf}

\newcommand{\myTitle}{Exact Mean Square Linear Stability Analysis for SGD}

\makeatletter
\newcommand*{\addFileDependency}[1]{% argument=file name and extension
	\typeout{(#1)}
	\@addtofilelist{#1}
	\IfFileExists{#1}{\typeout{File #1 O.K.}}{\typeout{No file #1.}}
}
\makeatother

\newcommand{\yy}{ {\boldsymbol{y} } }

\newcommand{\zz}{ {\boldsymbol{z} } }
\newcommand{\bs}{ {\boldsymbol{s} } }
\newcommand{\rr}{ {\boldsymbol{r} } }
\newcommand{\vv}{ {\boldsymbol{v} } }
\newcommand{\uu}{ {\boldsymbol{u} } }

\newcommand{\zeroVec}{ {\boldsymbol{0} } }

\newcommand{\grad}{ {\boldsymbol{g} } }
\newcommand{\params}{ {\boldsymbol{\theta} } }

\newcommand{\Identity}{ \boldsymbol{I} }

\newcommand{\HH}{ {\boldsymbol{H} } }
\newcommand{\PP}{ {\boldsymbol{P} } }
\newcommand{\MM}{ {\boldsymbol{M} } }

\newcommand{\XX}{ {\boldsymbol{X} } }
\newcommand{\YY}{ {\boldsymbol{Y} } }
\newcommand{\ZZ}{ {\boldsymbol{Z} } }
\newcommand{\VV}{ {\boldsymbol{V} } }
\newcommand{\TT}{ {\boldsymbol{T} } }
\newcommand{\UU}{ {\boldsymbol{U} } }
\newcommand{\bS}{ {\boldsymbol{S} } }
\newcommand{\CC}{ {\boldsymbol{C} } }
\newcommand{\DD}{ {\boldsymbol{D} } }
\newcommand{\LL}{ {\boldsymbol{L} } }
\newcommand{\myMat}{ {\boldsymbol{Y} } }

\newcommand{\OO}{ {\boldsymbol{O} } }
\newcommand{\bPsi}{ {\boldsymbol{\Psi} } }
\newcommand{\bXi}{ {\boldsymbol{\Xi} } }

\newcommand{\CovEvo}{ {\boldsymbol{Q} } }
\newcommand{\bE}{ {\boldsymbol{E} } }

\newcommand{\bLambda}{ {\boldsymbol{\Lambda} } }

\newcommand{\TransMat}{ {\boldsymbol{A} } }

\newcommand{\E}{ \mathbb{E} }
\newcommand{\EE}{ \mathbb{E} }
\newcommand{\prob}{ \mathbb{P} }

\newcommand{\CovMat}{\boldsymbol{\Sigma}}
\newcommand{\MeanVec}{\boldsymbol{\mu}}

\newcommand{\R}{ {\mathbb{R} } }

\newcommand{\Sb}{ {\mathbb{S} } }

\newcommand{\batch}{ {\mathfrak{B} } }

\newcommand{\PSDset}[1]{ {\mathcal{S}_{+}(\mathbb{R}^{#1 \times #1}) } } 
\newcommand{\myNull}[1]{ \mathcal{N}(#1) }
\newcommand{\myNullOrtho}[1]{ \mathcal{N}^{\myPerp}(#1) }
\newcommand{\myRange}[1]{ \mathcal{R}(#1) }
\newcommand{\myRangeOrtho}[1]{ \mathcal{R}^{\myPerp}(#1) }

\newcommand{\loss}{ {\mathcal{L} } }
\newcommand{\stochasticLoss}{ \hat{\mathcal{L}} }
\newcommand{\lossApprox}{\tilde{\mathcal{L}}}

\newcommand{\norm}[1]{ \left\| #1 \right\| }
\newcommand{\normF}[1]{ \left\| #1 \right\|_{\mathrm{F}}}

\newcommand{\vectorization}[1]{ \mathrm{vec} \left( #1 \right) }
\newcommand{\ie}{\emph{i.e.,}~}
\newcommand{\eg}{\emph{e.g.,}~}

\newcommand\remove[1]{}
\newcommand{\Tr}{ \mathrm{Tr} }
\newcommand{\transpose}{ { \mathrm{T}} }
\newcommand{\etaMean}{\eta^*_{\mathrm{mean}}}
\newcommand{\etaVar}{\eta^*_{\mathrm{var}}}
\newcommand{\myPerp}{{\scriptscriptstyle \boldsymbol{\perp}}}
\newcommand{\myPar}{{\scriptscriptstyle \boldsymbol{\|}}}
\definecolor{myRed}{RGB}{162, 20, 47}
\definecolor{myYellow}{RGB}{237, 177, 32}
\definecolor{myBlue}{RGB}{0, 114, 189}
\definecolor{myPurple}{RGB}{124, 45, 140}

\title[\myTitle]{\myTitle}

\coltauthor{%
 \Name{Rotem Mulayoff} \Email{rotem.mulayof@gmail.com}\\
 \addr Technion -- Israel Institute of Technology
 \AND
 \Name{Tomer Michaeli} \Email{tomer.m@ee.technion.ac.il}\\
 \addr Technion -- Israel Institute of Technology%
}

\begin{document}

\maketitle

% Add Abstract
\begin{abstract}
The dynamical stability of optimization methods at the vicinity of minima of the loss has recently attracted significant attention. For gradient descent (GD), stable convergence is possible only to minima that are sufficiently flat w.r.t.~the step size, and those have been linked with favorable properties of the trained model. However, while the stability threshold of GD is well-known, to date, no explicit expression has been derived for the exact threshold of stochastic GD (SGD). In this paper, we derive such a closed-form expression. Specifically, we provide an explicit condition on the step size that is both necessary and sufficient for the linear stability of SGD in the mean square sense. Our analysis sheds light on the precise role of the batch size $B$. In particular, we show that the stability threshold is monotonically non-decreasing in the batch size, which means that reducing the batch size can only decrease stability. Furthermore, we show that SGD's stability threshold is equivalent to that of a mixture process which takes in each iteration a full batch gradient step w.p.~$1-p$, and a single sample gradient step w.p.~$p$, where $p \approx 1/B $. This indicates that even with moderate batch sizes, SGD's stability threshold is very close to that of GD's. We also prove simple necessary conditions for linear stability, which depend on the batch size, and are easier to compute than the precise threshold. Finally, we derive the asymptotic covariance of the dynamics around the minimum, and discuss its dependence on the learning rate. We validate our theoretical findings through experiments on the MNIST dataset.
\end{abstract}


\section{Introduction}\label{sec:intro}
% Motivation - stability in optimization
The dynamical stability of optimization methods has been shown to play a key role in shaping the properties of trained models. For instance, gradient descent (GD) can stably converge only to minima that are sufficiently flat with respect to the step size \citep{cohen2021gradient}, and in the context of neural networks, such minima were shown to correspond to models with favorable properties. These include smoothness of the predictor function \citep{ma2021on,nacson2023the, mulayoffNeurips}, balancedness of the layers \citep{mulayoff2020unique}, and arguably better generalization \citep{hochreiter1997flat,keskar2016large,jastrzkebski2017three,wu2017towards,ma2021on}. While the stability threshold of GD is well-known, that of stochastic GD (SGD) has yet to be fully understood. Several empirical works studied SGD's stability  \citep{jastrzebski2018on,Jastrzebski2020The,cohen2021gradient,gilmer2022a}, yet they did not determine a definitive stability condition. Various theoretical works studied SGD's dynamics using linear stability, \ie via second-order Taylor expansion at the vicinity of minima, focusing on stability in the mean square sense \citep{wu2018sgd, granziol2022learning, velikanov2023a}, higher moments \citep{ma2021on}, and in probability \citep{ziyin2023probabilistic}. However, these works either do not provide explicit stability conditions or rely on strong assumptions. For example, \citet{wu2018sgd,ma2021on} present the condition as a complex optimization problem. Similarly, \citet{granziol2022learning} consider infinite network widths and make strong assumptions on the nature of the batching noise. Likewise, \citet{velikanov2023a} analyze SGD with momentum, assuming momentum parameter close to $1$ and ``spectrally expressible'' dynamics, and present their result in terms of a moment generating function. Overall, the exact stability threshold of SGD in the general case remains unknown.

% Our setting and results
In this paper, we analyze the linear stability of SGD in the mean square sense. We start by considering interpolating minima, which are common in training of overparametrized models. In this case, we provide an explicit threshold on the step size~$\eta$ that is both necessary and sufficient for stability. Our analysis sheds light on the precise role of the batch size $B$. In particular, we show that the maximal step size allowing stable convergence is monotonically non-decreasing in the batch size. Namely, decreasing the batch size can only decrease the stability threshold of SGD. Moreover, we show that this threshold is equivalent to that of a mixture process that takes in each iteration a full batch gradient step w.p.~$1-p$, and a single sample gradient step w.p.~$p$, where~$p~\approx~1/B $. This suggests that even with moderate batch sizes, SGD's stability threshold is very close to that of GD's.  Although our result gives an explicit condition on the step size for stability, its computation may still be challenging in practical applications. Thus, we also prove simple necessary criteria for stability, which depend on the batch size and are easier to compute.

Next, we turn to study a broader class of minima which we call \emph{regular}. Specifically, in interpolating minima, the loss of each individual sample has zero gradient and a positive semi-definite (PSD) Hessian. In regular minima, the individual Hessians are still required to be PSD, but the gradients can be arbitrary. Only the average of the gradients over all samples has to vanish (as in any minimum). In this setting, the dynamics can wander within the null-space of the Hessian, if the gradients have nonzero components in that subspace. However, the interesting question is whether the process is stable within the orthogonal complement of the null space. Here we again provide an explicit condition on the step size that is both necessary and sufficient for linear stability. We further derive the theoretical limit of the covariance matrix of the dynamics, as well as the limit values of the expected squared distance to the minimum, the expected loss, and the expected squared norm of the gradient, and show how they all decrease when reducing the learning rate. This provides a theoretical explanation of the behavior encountered in common learning rate scheduling strategies.

Finally, we validate our theoretical results through experiments on the MNIST dataset \citep{lecun1998mnist}. These confirm that our theory correctly predicts the stability threshold of SGD, and its dependence on the batch size. Furthermore, the experiments suggest that SGD converges at the edge of its (mean-square) stability region at least in certain training regimes, which is an interesting subject for future research.

\paragraph{Contributions.} To summarize, our main contributions are as follows.
\begin{itemize}
    \setlength\itemsep{-0.3em} % -0.45em
    \item We derive a closed-form expression for the maximal step size with which mini-batch SGD is linearly stable in the mean square sense. The threshold depends on the batch size (Thm.~\ref{thm:stability threshold expectation}).
    \item We prove that the stability threshold is monotonically non-decreasing with the batch size, so that smaller batches can only compromise stability (Prop.~\ref{thm:Monotonicity}).
    \item We show that the stability threshold of mini-batch SGD is the same as the stability threshold of an algorithm that randomly chooses in each iteration whether to perform a GD step or a single-sample SGD step, where the probability is roughly one over the batch size (Prop.~\ref{prop:equivalent algorithm}).
    \item We determine a lower bound on the batch size such that the stability threshold of SGD is close to that of GD (Prop.~\ref{prop:stability gap}).    
    \item We provide simpler necessary conditions for the linear stability of mini-batch SGD (Prop.~\ref{thm:stability expectation necessary}).
    \item Apart from the common setting of interpolating minima, we also study a large family of non-interpolating minima.  For interpolating minima, SGD converges below the stability threshold. In contrast, for non-interpolating minima, SGD randomly wanders around the minimum. Here we again derive a closed-form for the stability threshold, as well as an expression for the covariance matrix of the dynamics (Thm.~\ref{thm:stability threshold expectation regular minima} and Thm.~\ref{thm:covariance limit proposition}).
    \item Key to our derivations is a fundamental algebraic result, which we prove for sums of Kroneker products of symmetric matrices (Thm.~\ref{thm:Symmetric Kronecker systems}).
\end{itemize}

\section{Background: Linearized dynamics}\label{sec:linearization}
Let $\ell_i:\R^d\to\R$ be differentiable almost everywhere for all $i\in [n]$. We consider the minimization of a loss function
\begin{equation}
	% \smash
 {\loss(\params)=\frac{1}{n}\sum_{i=1}^n \ell_i(\params)}
\end{equation}
using the SGD iterations
\begin{equation}\label{eq:UpdateRule}
	\params_{t+1} = \params_t - \eta \nabla \stochasticLoss_t(\params_t).
\end{equation}
Here, $ \eta $ is the step size and $\stochasticLoss_t$ is a stochastic approximation of $\loss$ obtained as
\begin{equation}
\stochasticLoss_t(\params) = \frac{1}{B}\sum_{i \in \batch_t } \ell_i(\params),
\end{equation}
where $\batch_t$ is a batch of size $B$ sampled at iteration $t$. We assume that the batches $ \{ \batch _t \} $ are drawn uniformly at random from the $\smash{\binom{n}{B}}$ possible options, independently across iterations. Namely, there are distinct samples within each batch and possible repetitions between different batches.

Analyzing the full dynamics of this process is intractable in most cases. Yet near minima, accurate characterization of the stability of the iterates can be obtained via linearization \citep{wu2018sgd,ma2021on,mulayoffNeurips}, %nar2018step
as is common in the analysis of nonlinear systems. 
\begin{definition}[Linearized dynamics]\label{def:Linearization}
    Let $ \params^* $ be a twice differentiable minimum of $ \loss $, and denote
    \begin{equation}
        \grad_i \triangleq \nabla\ell_i(\params^*), \qquad \HH_i \triangleq \nabla^2 \ell_i(\params^*).
    \end{equation}
    Then the linearized dynamics of SGD near $ \params^*  $ is given by
    \begin{equation}\label{eq:linearized dynamics}
        \params_{t+1} = \params_{t} - \frac{\eta}{B} {\sum_{i \in  \batch_t }} \HH_i (\params_t - \params^* ) -\frac{\eta}{B} {\sum_{i \in  \batch_t }} \grad_i.
    \end{equation}
\end{definition}
Note that since $ \params^* $ is a minimum point of $ \loss $ we have that
\begin{equation}\label{eq:gradEqZero}
	\nabla \loss(\params^*) = \frac{1}{n}% \smash
 {\sum_{i= 1}^n} \grad_i = \zeroVec.
\end{equation}
Furthermore, the Hessian of the loss, which we denote by $\HH$, is given by
\begin{equation}\label{eq:Hessian Def}
    \HH \triangleq \nabla^2 \loss(\params^*) = \frac{1}{n} %\smash
    {\sum_{i= 1}^n} \HH_i . 
\end{equation}
Thus, the linearized dynamics are in fact SGD iterates on the second-order Taylor expansion of $ \loss $ at~$ \params^* $,
\begin{equation}\label{eq:TaylorLtilde}
    \lossApprox(\params) = \loss(\params^*)+%\smash
    {\frac{1}{2}} (\params-\params^* )^\transpose \HH (\params - \params^*).
\end{equation}

\section{Stability of first and second moments}\label{sec:expectation results}
Our focus is on the stability of SGD's dynamics at the vicinity of minima. We specifically examine the dynamics within two subspaces: the null space of the Hessian $\HH$ at the minimum, and its orthogonal complement. We denote the projection of any vector $ \vv \in \R^{d} $ onto the null space of $\HH$ by~$ \vv^{\myPar} $, and its projection onto the orthogonal complement of the null space by $ \vv^{\myPerp} $.

Multiple works studied the stability of SGD's dynamics. Commonly, this was done by analyzing the evolution of the moments of the linearized dynamics (see Sec.~\ref{sec:linearization}) over time, with a specific emphasis on the second moment, which is the approach we take here. However, before discussing the evolution of the second moment, let us summarize the behavior of the first moment. Specifically, it is easy to demonstrate that the first moment of SGD's linearized trajectory $ \{\E[  \params_t ]\} $ is the same as GD's. Since GD is stable if and only if $ \eta \leq 2/\lambda_{\max}(\HH) $, we have the following (see proof in App.~\ref{app:Mean dynamics}).
\begin{theorem}[Stability of the mean]\label{thm:stability first moment}
    Assume that $ \params^* $ is a twice differentiable minimum. Consider the linear dynamics of $ \{ \params_t \} $ from Def.~\ref{def:Linearization} and let
    \begin{equation}
        \etaMean \triangleq \frac{2}{\lambda_{\max}(\HH)}.
    \end{equation}
    Then 
    \begin{enumerate}
        \item $ \E\big[\params_t^{\myPar}\big] = \E\big[ \params_0^{\myPar} \big] $ for all $t\geq0$;
        \item $ \underset{t\to \infty}{\mathrm{limsup}} \big\| \E\big[  \params_t \big]-\params^* \big\| $ is finite if and only if $\eta \leq \etaMean$;
        \item $ \underset{t\to \infty}{\lim} \big\| \E\big[  \params_t^{\myPerp} \big]-\params^{*\myPerp}  \big\| =0  $ if $\eta < \etaMean$.
    \end{enumerate}
\end{theorem}

We next proceed to analyze the dynamics of the second moment, which determine stability in the mean square sense. Note that boundedness of the first moment is a necessary condition for boundedness of the second moment. Therefore, the condition $ \eta \leq \etaMean$ is a prerequisite for stability in the mean square sense. However, how much smaller than $\etaMean $ is SGD's mean square stability threshold, is not currently known in closed-form. Here, we determine the precise threshold for the mean square stability of SGD's linearized dynamics. To achieve this, we leverage the approach taken by \citet{ma2021on}, who investigated the stability of SGD in the context of interpolating minima.

\subsection{Interpolating minima}\label{sec:Interpolating minima}
We begin by studying interpolating minima, which are prevalent in the training of overparametrized models. In this case, the model fits the training set perfectly, %\footnote{The important minima from a practical standpoint are the ones that benign overfit.}, 
which means that these global minima are also minima for each sample individually. This is expressed mathematically as follows.
\begin{definition}[Interpolating minima]\label{ass:Interpolating minimum}~A twice differentiable minimum $ \params^* $ is said to be \emph{interpolating} if for each sample $i \in [n] $ the gradient $ \grad_i = \zeroVec $ and the Hessian $\HH_i $ is PSD. 
\end{definition}
In this setting, \citet{ma2021on} showed that the evolution over time of any moment of SGD's linearized dynamics is fully tractable. Specifically, for the second moment, they proved the following.
\begin{theorem}[\citet{ma2021on}, Thm.~1 + Cor.~3]\label{thm:stability expectation general}
    Assume that $ \params^* $ is a twice differentiable interpolating minimum. Consider the linear dynamics of $ \{ \params_t \}  $ from Def.~\ref{def:Linearization}, and let
    \begin{equation}\label{eq:Covariance evolution matrix}
        \CovEvo(\eta, B) \triangleq (\Identity-\eta\HH)\otimes(\Identity-\eta\HH) + \frac{n-B}{B(n-1)} \frac{\eta^2}{n}  \smash{\sum_{i = 1}^n} (\HH_i \otimes \HH_i -\HH \otimes \HH ),
    \end{equation}
    where $ \otimes $ denotes the Kronecker product. Then $ \smash{\underset{t\to \infty}{\mathrm{limsup}} \E[ \| \params_t -\params^*  \|^2 ] }$ is finite if and only if
    \begin{equation}\label{eq:impicit SGD stability condition}
    	\max_{\CovMat \in \PSDset{d}} \frac{  \norm{\CovEvo(\eta, B) \; \vectorization{\CovMat}}}{\normF{\CovMat}} \leq 1,
    \end{equation}
    where $ \PSDset{d} $ denotes the set of all PSD matrices over $ \R^{d \times d} $. 
    Furthermore, if the spectral radius~$ \rho (\CovEvo(\eta, B)) \leq 1 $ then $ \underset{t\to \infty}{\mathrm{limsup}} \E\big[ \| \params_t -\params^*  \|^2 \big] $ is finite.
\end{theorem}
Below, we omit the dependence of $\CovEvo$ on $\eta$ and $B$ whenever these are not essential for the discussion. In this theorem, $ \CovMat $ represents the second-moment matrix of $\params_t-\params^*$. Specifically, the matrix $\CovMat_t = \E[(\params_t-\params^*)(\params_t-\params^*)^\transpose] $ evolves over time as $ \mathrm{vec}(\CovMat_{t+1}) = \CovEvo \, \mathrm{vec}(\CovMat_{t}) $. Therefore, the stability condition of \eqref{eq:impicit SGD stability condition} simply states that if the dynamics of the dominant initial state of the system (which is restricted to PSD matrices) is bounded, then $ \CovMat_t $ is bounded and vice versa. However, this characterization leaves us with a constrained optimization problem over a $d^2$-dimensional space, which is hard to solve numerically. Therefore, this approach does not reduce the problem into a condition from which we can gain any theoretical insight into SGD's stability.

Our first key result is that the constrained optimization problem in \eqref{eq:impicit SGD stability condition} can be reduced to an eigenvalue problem. Specifically, we establish (see Sec.~\ref{sec:Stability threshold proof outline}) that when the eigenvectors of the $d^2\times d^2$ matrix $ \CovEvo $ are reshaped into $d\times d$ matrices, they correspond to either symmetric or skew-symmetric matrices\footnote{Eigenbases corresponding to eigenvalues of multiplicity greater than one, always have a basis consisting of symmetric and skew-symmetric matrices.}. Moreover, the top eigenvalue of $ \CovEvo $ is a dominant eigenvalue, and always corresponds to a PSD matrix. Consequently, the maximizer of \eqref{eq:impicit SGD stability condition} is the top eigenvector of~$\CovEvo$, which we use, along with some algebraic manipulation, to derive the following result (see proof in App.~\ref{app:stability threshold expectation proof}).
\begin{theorem}[Mean square stability for interpolating minima]\label{thm:stability threshold expectation}
    Assume that $ \params^* $ is a twice differentiable interpolating minimum. Consider the linear dynamics of $ { \{ \params_t \} }  $ from Def.~\ref{def:Linearization}, and let
    \begin{equation}
        \CC \triangleq \frac{1}{2} \HH \oplus \HH , \qquad
        \DD \triangleq (1-p)\, \HH \otimes \HH + p \,\frac{1}{n}\smash{\sum_{i=1}^n} \HH_i \otimes \HH_i,
    \end{equation}
    where $ \oplus $ denotes the Kronecker sum and $ \smash{ p \triangleq \frac{n-B}{B(n-1)} \in [0, 1] }$ . Define
    \begin{equation}\label{eq:etaVar definition}
        \etaVar\triangleq \frac{2}{\lambda_{\max}\left(\CC^{\dagger} \DD \right)},
    \end{equation}
    where $ { \CC^{\dagger} } $ denotes the Moore-Penrose inverse of $ \CC $. 
    Then 
    \begin{enumerate}
    \setlength\itemsep{0.3em}
        \item $ {\params_t^{\myPar}=\params_0^{\myPar}}$ (surely) for all $t\geq0$;
        \item $ { \underset{t\to \infty}{\mathrm{limsup}}\, \E\big[ \| \params_t^\myPerp -\params^{*\myPerp}  \|^2 \big] } $ is finite if and only if $\eta \leq \etaVar$;
        \item $ { \underset{t\to \infty}{\lim} \E\big[ \| \params_t^\myPerp -\params^{*\myPerp}  \|^2 \big] =0 }$ if $\eta < \etaVar$.
    \end{enumerate}
\end{theorem}
This result provides an explicit characterization of the mean square stability of SGD. Here we see that the set of step sizes that are stable in the mean square sense, is an interval. This is in contrast to stability in probability, where the stable learning rates can comprise of several disjoint intervals~\citep{ziyin2023probabilistic}. Moreover, SGD's threshold, $ \etaVar $, has the same form as the threshold for GD, $ 2/\lambda_{\max}$, but with a different matrix. In App.~\ref{app:Recovering GD's stability condition} we show how Thm.~\ref{thm:stability threshold expectation} recovers GD's condition when $ B=n $.

The dependence of $ \etaVar $ on the batch size $B$ may not be immediate to see from the theorem. However, we can prove the following (see proof in App.~\ref{app:MonotonicityProof}).
\begin{proposition}[Monotonicity of the stability threshold]\label{thm:Monotonicity}
    Assume that $ \params^* $ is a twice differentiable interpolating minimum. Then $ \etaVar $ is a non-decreasing function of $ B $.  
\end{proposition}%
This result implies that decreasing the batch size can only impair stability, which settles with empirical observations, \eg in \citep[Fig.~1]{wu2018sgd}. Additionally, since $ \etaVar $ is non-decreasing with $B$, and for $ B = n$  it equals $ \etaMean $, we have that the gap between $ \lambda_{\max}(\CC^\dagger\DD) $ and $  \lambda_{\max}(\HH) $ is nonnegative for all $ B \in [1,n] $ and non-increasing in $B$. For stable minima, $ \lambda_{\max}(\CC^\dagger\DD) $ is bounded from above by $ 2/\eta $. This suggests that training with smaller batches leads to lower $ \lambda_{\max}(\HH) $, \ie flatter minima, which aligns with experimental results \citep{keskar2016large,jastrzkebski2017three}.

At what rate does $ \etaVar $ increase with $B$ towards $ \etaMean $? To understand this, note that $\DD$ is a convex combination of two matrices, where $p$ represents the combination weight. The first matrix, $\HH \otimes \HH$, is associated with full batch SGD ($ B = n $), while the second matrix, $\frac{1}{n}\sum_{i=1}^n\HH_i \otimes \HH_i$, is related to single sample SGD ($ B = 1 $). We can use this fact to explain the effect of the batch size on dynamical stability by presenting an equivalent stochastic process that has the same stability threshold as SGD (see proof in App.~\ref{app:equivalent algorithm proof}).
\begin{proposition}[Equivalent mixture process]\label{prop:equivalent algorithm}
Let \emph{\texttt{ALG}}$ (p) $ be a stochastic optimization algorithm in which
\begin{equation}
     \params_{t+1} = 
     \begin{cases}
         \params_t - \eta \nabla \ell_{i_t}(\params_t) &  \text{w.p.} \quad  p, \\
         \params_t - \eta \nabla \loss(\params_t) &  \text{w.p.} \quad  1-p,
     \end{cases}  
\end{equation}
where $ \{ i_t \} $ are i.i.d.\ random indices distributed uniformly over the training set. Assume that $ \params^* $ is a twice differentiable interpolating minimum. Then when $ p=\frac{n-B}{B(n-1)} $, \emph{\texttt{ALG}}$ (p) $ has the same stability threshold in the vicinity of $ \params^* $ as SGD with batch size $ B $.
\end{proposition}
In simpler terms, \texttt{ALG}$ (p) $ is a mixture process that takes in each iteration a gradient step with a batch of one sample ($B=1$) with probability $p$ and with a full batch ($B=n$) with probability $1-p$. This result shows that the stability conditions of SGD and of \texttt{ALG}$ (p) $ are the same for $ p=\frac{n-B}{B(n-1)}  $. For~$ n\gg B $, we get $ p \approx 1/B$. Thus, Prop.~\ref{prop:equivalent algorithm} implies that, in the context of stability, even moderate values of $ B $ make mini-batch SGD behave like GD. We next quantify how large $B$ needs to be in order for the gap between $ \etaVar $ and $ \etaMean $ to be small  (see proof in App.~\ref{app:stability gap proof}).
\begin{proposition}[Stability gap]\label{prop:stability gap} Define $ \bE = \frac{1}{n} \sum_{i = 1}^n(\HH_i-\HH) \otimes (\HH_i-\HH) $ and let  $ \varepsilon \in (0,1)$. If
\begin{equation}
    B \geq \frac{1-\varepsilon}{\varepsilon} \  \frac{\lambda_{\max}(\CC^{\dagger}\bE)}{\lambda_{\max}(\HH)},
\end{equation}
then
\begin{equation}
    (1-\varepsilon)\etaMean \leq \etaVar \leq  \etaMean .
\end{equation}
\end{proposition}
Here $ \bE $ captures the variance of the per-sample Hessians. Thus, this result suggests that when these Hessians are similar (\ie the entries of $ \bE $ are small), moderate batch sizes are sufficient to guarantee a small gap between $ \etaMean $ and $ \etaVar $. On the other hand, if the variance of the Hessians is large, then the batch size $ B $ is expected to be large for the stability thresholds of SGD and GD to be close. We note that while propositions~\ref{thm:Monotonicity}-\ref{prop:stability gap} were presented in the context of interpolating minima, they also apply to regular minima (see Sec.~\ref{sec:Non-interpolating minima}).

Although Thm.~\ref{thm:stability threshold expectation} provides an explicit threshold for the step size, its computation may be challenging in practical applications, as it requires inverting, multiplying, and computing the spectral norm of large ($d^2\times d^2$) matrices. Still, we can obtain necessary criteria for stability that are simple and easier to verify, and which also depend on the batch size. To do so, we compute quadratic forms over $ \CC^{\dagger} \DD $ with non-optimal yet interesting vectors. In this way, we bound $ \lambda_{\max}(\CC^{\dagger} \DD) $ from below to get the following result
(see detail and proof in App.~\ref{app:stability necessary proof}).
\begin{proposition}[Necessary conditions for stability]\label{thm:stability expectation necessary}
    Let $ \vv_{\max} $ be a top eigenvector of $ \HH $. Then the step size $\etaVar$ satisfies
    \begin{equation}\label{eq:necessary condition lambda max}
     \etaVar \leq \frac{ 2\lambda_{\max}(\HH) }{ \lambda_{\max}^2(\HH) + \frac{p}{n}\sum_{i = 1}^n ( \vv_{\max}^\transpose \HH_i \vv_{\max} - \lambda_{\max}(\HH))^2 },
    \end{equation}
    as well as
    \begin{equation}\label{eq:necessary condition Frobenius norm}
     % \smash[t]
     {\etaVar \leq \frac{ 2 \Tr(\HH) }{ (1-p)\normF{\HH}^2  + \frac{p}{n}\sum_{i = 1}^n  \normF{\HH_i}^2}}.
    \end{equation}
\end{proposition}
From \eqref{eq:necessary condition lambda max}, we can deduce a lower bound on the gap between the stability thresholds of GD and SGD. Specifically, when the variance of $\HH_i$ along the direction of the top eigenvector of $ \HH $ is large, $\etaVar$ is far from $ \etaMean $ for moderate $p$. In general, this condition is expected to be quite tight when there is a clear dominant direction in $ \HH $ caused by some $\HH_i$. In contrast, condition \eqref{eq:necessary condition Frobenius norm} is expected to be tight if all $ \{ \HH_i \} $ have roughly the same spectrum but with different bases, \ie when no sample is dominant and the samples are incoherent.

It is worthwhile mentioning that if the stability condition of Thm.~\ref{thm:stability threshold expectation} is not met, then the linearized dynamics diverge. However, in practice, the full (non-linearized) dynamics can just move to a different point on the loss landscape, where the generalized sharpness $ \lambda_{\max}(\CC^{\dagger} \DD) $  is lower. It was shown that GD possesses such a stabilizing mechanism \citep{damian2023selfstabilization}. An interesting open question is whether a similar mechanism exists in SGD.

\subsection{Non-interpolating minima}\label{sec:Non-interpolating minima}
While for interpolating minima, we saw that $\params_t^\myPerp$ can converge to $\params^{*\myPerp}$, this is generally not the case for non-interpolating minima. In this section, we explore the dynamics of SGD at the vicinity of a broader class of minima. Specifically, we consider the following definition.
\begin{definition}[Regular minima]\label{ass:Regular minimum}~A twice differentiable minimum $ \params^* $ is said to be \emph{regular} if for each sample $i \in [n] $ the Hessian $\HH_i $ is PSD. 
\end{definition}
This definition encompasses a broader class of minima than Def.~\ref{ass:Interpolating minimum}, as it allows for arbitrary (nonzero) gradients $\grad_i$. Only the mean of the gradients has to vanish (as in any minimum). Intuitively speaking, although a regular minimum does not necessarily fit all the training points, it does not involve a major disagreement among them. This can be understood through the second-order Taylor expansion of the loss per sample, which can have descent directions in the parameter space, yet it can only decrease (on behalf of raising the loss to other samples) linearly with the parameters, and not quadratically.

Clearly, having gradients with nonzero components in the null space of the Hessian pushes the dynamics to diverge. Interestingly, for regular minima, the dynamics of SGD in the null space and in its orthogonal complement are separable. Thus, despite having a random walk in the null space, we can give a condition for stability within its orthogonal complement (see proof in App.~\ref{app:stability threshold expectation regular minima proof}).
\begin{theorem}[Mean square stability for regular minima]\label{thm:stability threshold expectation regular minima}
    Assume that $ \params^* $ is a twice differentiable regular minimum. Consider the linear dynamics of $ \{ \params_t \}  $ from Def.~\ref{def:Linearization}. Then
    \begin{enumerate}
        \item $ %\smash
        { \underset{t\to \infty}{\lim} \E\big[ \| \params_t^{\myPar} - \params^{*\myPar} \|^2 \big] = \infty \ } $ if and only if $ %\smash
        { \ \sum_{i=1}^n \| \grad_i^{\myPar} \|^2 >0  }$;
        \item If $ %\smash
        { \eta < \etaVar }$ then $ %\smash
        {\underset{t\to \infty}{\mathrm{limsup}}\, \E\big[ \| \params_t^\myPerp -\params^{*\myPerp}  \|^2 \big] } $ is finite;
        \item If $ \ %\smash
        { \underset{t\to \infty}{\mathrm{limsup}}\, \E\big[ \| \params_t^\myPerp -\params^{*\myPerp}  \|^2 \big] } $ is finite then $ %\smash
        { \eta \leq \etaVar }$.
    \end{enumerate}
\end{theorem}
We see that $ \etaVar $ is the stability threshold also for regular minima. Recall that when $ \eta < \etaVar $, we also have stability of the first moment, and thus $ \E[\params_t^{\myPar}] = \E[\params_0^{\myPar}]$ for any $ t \geq 0$. Namely, SGD's dynamics in the null space is a random walk without drift. Note that moving in the null space does not increase the loss, however it might change the trained model. Furthermore, in the proof, we show that under a mild assumption, $ \underset{t\to \infty}{\mathrm{limsup}} \E[ \| \params_t^{\myPerp} -\params^{*\myPerp}  \|^2 ] $ is finite if and only if $ 0 \leq \eta < \etaVar $.

Next, we turn to compute the limit of the second moment of the dynamics (see proof in App.~\ref{app:covariance limit proposition proof}).
\begin{theorem}[Covariance limit]\label{thm:covariance limit proposition}
    Assume that $ \params^* $ is a twice differentiable regular minimum. Consider the linear dynamics of $ \{ \params_t \}  $ from Def.~\ref{def:Linearization}. If $ 0 < \eta < \etaVar $ then
    \begin{equation}
        \lim_{t \to \infty}  \vectorization{\CovMat_{t}^{\myPerp}} = \eta p \left(2\CC - \eta \DD \right)^{\dagger} \vectorization{\CovMat_{\grad}^{\myPerp}},
    \end{equation}
    where
    \begin{equation}
        % \smash[u]
        {\CovMat_{\grad}^{\myPerp} = \frac{1}{n}\sum_{i=1}^n \grad_i^{\myPerp}\left(\grad_i^{\myPerp}\right)^\transpose}.
    \end{equation}
\end{theorem}
Using this result we can obtain the mean squared distance to the minimum, the mean of the second-order Taylor expansion of the loss, and the mean of the squared norm of the expansion's gradient at large times (see proof in App.~\ref{app:covariance limit proof}).
\begin{corollary}\label{thm:covariance limit}
    Let $ \params^* $ be a twice differentiable regular minimum. Consider the second-order Taylor expansion of the loss, $\lossApprox$ of \eqref{eq:TaylorLtilde} and the linear dynamics of $ \{ \params_t \}  $ from Def.~\ref{def:Linearization}. If $ \eta < \etaVar $ then
    \begin{enumerate}
        \item $ { \underset{t\to \infty}{\lim} \E\big[ \| \params_t^\myPerp -\params^{*\myPerp} \|^2 \big] = \eta p (\vectorization{\Identity})^\transpose \big( 2 \CC -\eta \DD \big)^{\dagger} \vectorization{\CovMat_{\grad}^{\myPerp}} }$;
        \item $ { \underset{t\to \infty}{\lim} \E \big[\lossApprox(\params_t) \big] -\lossApprox(\params^*) = \frac{1}{2} \eta p (\vectorization{\HH})^\transpose \big( 2 \CC -\eta \DD \big)^{\dagger} \vectorization{\CovMat_{\grad}^{\myPerp}} } $;
        \item $ { \underset{t\to \infty}{\lim} \E\big[ \|\nabla \lossApprox (\params_t) \|^2 \big] =  \eta p \left(\vectorization{\HH^2}\right)^\transpose \big( 2 \CC -\eta \DD \big)^{\dagger} \vectorization{\CovMat_{\grad}^{\myPerp}} } $.
    \end{enumerate}
\end{corollary}
We see that these values depend linearly on the covariance matrix of the gradients. Specifically, if~$ \CovMat_{\grad} = \zeroVec  $ then we recover the results of interpolating minima. Moreover, note that for $ \eta \ll \etaVar $, we have that $ 2 \CC -\eta \DD \approx 2\CC$. Therefore, the main dependence on $\eta$ comes from the factor of $ \eta $ preceding these expressions. We thus get that when decreasing the learning rate, the loss level drops, and the parameters $ \params_t $ get closer to the minimum. This explains the empirical behavior observed when decreasing the learning rate in neural network training,  which causes the loss level to drop.

\subsection{Derivation of the stability threshold $ \etaVar $}\label{sec:Stability threshold proof outline}
In this section, we give a sketch of the derivation of the stability threshold $ \etaVar $ in the simple case of interpolating minima (\ie the second statement of Thm.~\ref{thm:stability threshold expectation}). For the full proof, please see App.~\ref{app:stability threshold expectation proof}. In interpolating minima, the gradient vanishes for each sample, \ie $\grad_i = \zeroVec$ for all $i \in [n]$. Therefore, from the linearized dynamics \eqref{eq:linearized dynamics} we get
\begin{equation}\label{eq:dynamics proof sketch}
    \params_{t+1} - \params^* =  \Big( \Identity - \smash{\frac{\eta}{B} \sum_{i \in  \batch_t }} \HH_i \Big) (\params_t - \params^* ) = \TransMat_t (\params_t - \params^* ),
\end{equation}
where $ \TransMat_t \triangleq \Identity - \frac{\eta}{B} \sum_{i \in  \batch_t } \HH_i $. Note that $\{\TransMat_t \}$ are i.i.d.\ and that $\params_{t}$ is constructed from $\TransMat_0,\ldots,\TransMat_{t-1}$, so that $\params_{t}$ and $\TransMat_t$ are statistically independent. Therefore, the covariance of the dynamics evolves as
\begin{align}
    \CovMat_{t+1} = \EE\left[ \left(\params_{t+1} - \params^*\right)\left(\params_{t+1} - \params^*\right)^\transpose \right]
    & = \EE\left[ \TransMat_t \left(\params_{t} - \params^*\right)\left(\params_{t} - \params^*\right)^\transpose \TransMat_t^\transpose \right] \nonumber\\
    & = \EE\left[ \TransMat_t \E\left[ \left(\params_{t} - \params^*\right)\left(\params_{t} - \params^*\right)^\transpose \middle| \TransMat_t  \right] \TransMat_t^\transpose \right]  \\
    & = \EE\left[ \TransMat_t \E\left[ \left(\params_{t} - \params^*\right)\left(\params_{t} - \params^*\right)^\transpose\right] \TransMat_t^\transpose \right]
    = \smash[b]{\EE\left[ \TransMat_t \CovMat_t \TransMat_t^\transpose \right]}, \nonumber
\end{align}
where in the first line we used \eqref{eq:dynamics proof sketch}, in the second we used the law of total expectation, and in the third we used the fact that $ 
\params_t $ is statistically independent of $ \TransMat_t $. Using vectorization we get
\begin{equation}\label{eq:covariance evolution sketch}
    \smash{\vectorization{\CovMat_{t+1}} = \EE\left[ \TransMat_t \otimes \TransMat_t \right] \vectorization{\CovMat_t}.}
\end{equation}
\citet{ma2021on} showed that $ \EE[ \TransMat_t \otimes \TransMat_t ] = \CovEvo $, where $ \CovEvo $ is given in \eqref{eq:Covariance evolution matrix}. Since $ \CovMat_t$ is PSD by definition, we only care about the effect of $\CovEvo$ on vectorizations of PSD matrices. Thus, $ \{ \CovMat_{t} \} $ are bounded if and only if (see proof in \citep{ma2021on})
\begin{equation}\label{eq:impicit SGD stability condition restated}
    \max_{\CovMat \in \PSDset{d}} \frac{  \norm{\CovEvo(\eta, B) \; \vectorization{\CovMat}}}{\normF{\CovMat}} \leq 1.
\end{equation}
This constrained optimization problem is hard to solve. Yet, if we ignore the constraint then the solution becomes simple -- it is the spectral radius of $ \CovEvo $ (since $ \CovEvo $ is symmetric). Surprisingly, it turns out that removing the constraint does not affect the solution because the matrix $ \CovMat \in \R^{d \times d} $ that maximizes the objective in \eqref{eq:impicit SGD stability condition restated} without constraints is guaranteed to be PSD, so that it also maximizes the objective under the constraint~$ \CovMat \in \PSDset{d} $. Proving this fundamental algebraic property is a main challenge and a key contribution of our work (see proof in App.~\ref{app:Symmetric Kronecker systems}).
\begin{theorem}[Symmetric Kronecker systems]\label{thm:Symmetric Kronecker systems}
    Let $ \smash{\{\myMat_i\}} $ be symmetric matrices in $ \smash{\R^{d\times d}}$. Define
    \begin{equation}\label{eq:Symmetric Kronecker systems}
        \CovEvo = \Big. \smash{\sum_{i=1}^M} \myMat_i\otimes \myMat_i, \Big.
    \end{equation}
    and let $ \smash{\zz_{\max}}$ be a top eigenvector of $ \smash{\CovEvo} $. Then
    \begin{enumerate}
        \setlength\itemsep{-0.2em}
        \item there always exists a complete set of eigenvectors $ \smash{\{  \zz_j \} }$ for $ \CovEvo $ such that each $ \smash{\ZZ_j = \mathrm{vec}^{-1}(\zz_j) }$ is either a symmetric or a skew-symmetric matrix;
        \item the top eigenvalue is a dominant eigenvalue\footnote{By ``top'' we refer to the largest eigenvalue, and by ``dominant'' we refer to the largest eigenvalues in absolute value.}, i.e., the spectral radius $ \smash{ \rho(\CovEvo) = \lambda_{\max} (\CovEvo) } $;
        \item there exists a top eigenvector corresponding to a PSD matrix, i.e., $ \smash{\mathrm{vec}^{-1}(\zz_{\max}) \in \PSDset{d} }$.
    \end{enumerate}
\end{theorem}
Taking $ \{\myMat_i\} $ to be the $ \binom{n}{B} $ realizations that $\TransMat_t$ can take, we obtain that $ \CovEvo = \EE[ \TransMat_t \otimes \TransMat_t ] $  is of the form~\eqref{eq:Symmetric Kronecker systems}. Note that the realizations of $\TransMat_t$ are symmetric, so that Thm.~\ref{thm:Symmetric Kronecker systems} applies. Therefore, the matrix that attains the maximum in the optimization problem \eqref{eq:impicit SGD stability condition restated} without the constraint, whose vectorization is generally a dominant eigenvector of $\CovEvo$, is guaranteed to be a PSD matrix. Furthermore, the corresponding objective value is $ \lambda_{\max}(\CovEvo)$. This implies that the linear system in \eqref{eq:covariance evolution sketch} is stable if and only if $ \lambda_{\max}(\CovEvo) \leq 1 $. Since~$ \CovEvo $ is symmetric, $ \lambda_{\max}(\CovEvo) \leq 1 $ is equivalent to $ \uu^\transpose  \CovEvo \uu \leq 1  $ for all $ \uu \in \Sb^{d^2-1} $. It is easy to show that~$ \CovEvo = \Identity - 2\eta \CC + \eta^2 \DD $ (see \eqref{eq:Q in terms of C and D}). In App.~\ref{app:stability threshold expectation proof} we prove that 	$ \smash{\uu^\transpose  \CovEvo \uu 	 = 1 - 2\eta \uu^\transpose\CC\uu + \eta^2 \uu^\transpose\DD\uu \leq 1} $ holds for all $ \uu \in \Sb^{d^2-1} $ if and only if 
\begin{equation}
    \eta \leq {\frac{2}{\lambda_{\max}\left(\CC^{\dagger} \DD \right)}} = \etaVar.
\end{equation}

\section{Experiments}\label{sec:Experiments}

In this section, we experimentally validate our theoretical results in a setting with nonlinear dynamics. We trained single hidden-layer ReLU networks with varying step sizes and batch sizes on a subset of the MNIST dataset (see App.~\ref{App:Additional experimental results}). Since training with cross-entropy in overparametrized networks results in infima rather than minima, we used the quadratic loss. Specifically, each class was labeled with a one-hot vector, and the network was trained to predict the label without softmax. Our primary goal in this experiment is to test the stability threshold of SGD; hence, we initialized the training with large weights to ensure that the minimum closest to the starting point is unstable (large weights imply large Hessians, and are thus more likely to violate the stability criterion). We used the same initial point for all the training runs to eliminate initialization effects. To avoid divergence, we started with a very small step size and gradually increased it until it reached its designated value (\ie learning rate warm-up). Together, large initialization and warm-up force SGD out of the unstable region until it finds a stable minimum and converges as closely as possible to the stability threshold. Convergence was determined when the loss remained below~$10^{-6}$ for $ 200 $ consecutive epochs.

\begin{figure}[t]%
    \subfigure[Sharpness vs. step size][b]{\label{subfig:Sharpness vs. step size}%
    \includegraphics[width=0.5\linewidth]{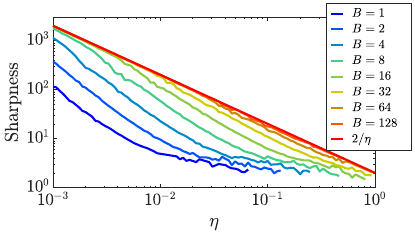}}%
    \subfigure[Sharpness vs. batch size][b]{\label{subfig:Sharpness vs. batch size}%
    \includegraphics[width=0.5\linewidth]{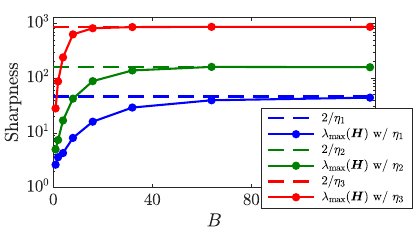}}%
    \caption{\textbf{Sharpness vs. step size and batch size.} We trained single hidden-layer ReLU networks using varying step sizes and batch sizes on a subset of MNIST. Panel~\protect\subfigref{subfig:Sharpness vs. step size} visualizes the sharpness of the converged minima versus learning rate for different batch sizes. For small batch sizes, $\lambda_{\max}(\HH)$ deviates significantly from $2/\eta$. Yet as the batch size increases to a moderate value, these curves coincide, indicating that in terms of stability, SGD behaves similarly to GD. Panel~\protect\subfigref{subfig:Sharpness vs. batch size} plots the sharpness against the batch size for three different learning rates $ \eta_1 = 0.043, \eta_2 = 0.012, \eta_3 = 0.002$. Here we see a similar trend where SGD behaves like GD for~$ B \geq 32 $.}
    \label{Fig:B}%
    \vspace{-12pt}
\end{figure}

Figure~\ref{subfig:Sharpness vs. step size} visualizes the sharpness of the converged minima versus the learning rate for several values of $ B $. Here we observe that for small batch sizes, $ \lambda_{\max}(\HH) $ is far from $2/\eta$. Yet for moderate batch sizes and above (\eg $ B \geq 32 $), these curves virtually coincide, indicating that, in the context of stability, SGD behaves like GD. Figure~\ref{subfig:Sharpness vs. batch size} shows the sharpness versus the batch size for three step sizes. Here the stability threshold of SGD rapidly converges to that of GD as the batch size increases.

Apart for the sharpness $ \lambda_{\max}( \HH ) $, we also want to compare the generalized sharpness $ \lambda_{\max}(\CC^{\dagger} \DD ) $ to $2/\eta$. Since computing the generalized sharpness is impractical in this task, we underestimate it via a lower bound, which results in a tighter necessary condition than \eqref{eq:necessary condition lambda max}. The bound corresponds to restricting the optimization problem in \eqref{eq:impicit SGD stability condition} to rank one PSD matrices, and is given by (see App.~\ref{sec: evaluating v otimes v on Q})
\begin{equation}\label{eq:optimized bound}
    \frac{2}{\etaVar} = \lambda_{\max}\left( \CC^{\dagger} \DD \right) 
    \geq \max_{\vv : \|\vv\| = 1 } \left\{\vv^\transpose \HH \vv+p \frac{\frac{1}{n} \sum_{i = 1}^n ( \vv^\transpose \HH_i \vv - \vv^\transpose \HH \vv)^2 }{ \vv^\transpose \HH \vv } \right\}.
\end{equation}
We solve this optimization problem numerically, by using GD on the unit sphere with predetermined scheduled geodesic step size. In the following, we present graphs of the sharpness $\lambda_{\max}(\HH)$ at the minima to which we converged, as well as the bounds \eqref{eq:optimized bound} and \eqref{eq:necessary condition lambda max} on the generalized sharpness~$\lambda_{\max}(\CC^{\dagger}\DD)$. Using the color coding of Fig.~\ref{Fig:A}, these correspond to
\begin{align}\label{eq:relationship}
    {\color{myRed}  \frac{2}{\eta}}
    & \geq 
    \lambda_{\max}\left( \CC^{\dagger} \DD \right) \quad\left( = \frac{2}{\etaVar} \right)\nonumber \\
    & \geq
    {\color{myPurple}\max_{\vv : \|\vv\| = 1 }\left\{\vv^\transpose \HH \vv+p \frac{\frac{1}{n} \sum_{i = 1}^n ( \vv^\transpose \HH_i \vv - \vv^\transpose \HH \vv)^2 }{ \vv^\transpose \HH \vv } \right\}} \nonumber\\
    & \geq
    {\color{myBlue} \lambda_{\max}(\HH)+ p \frac{\frac{1}{n}\sum_{i = 1}^n ( \vv_{\max}^\transpose \HH_i \vv_{\max} - \lambda_{\max}(\HH))^2}{\lambda_{\max}(\HH)} }\nonumber \\
    & \geq \smash{\color{myYellow}\lambda_{\max}(\HH)},
\end{align}
where $ \vv_{\max} $ denotes the top eigenvector of $ \HH $.

\begin{figure}[t]%
    \centering%
    \rotatebox[origin=l]{90}{\hspace{1.1cm}\fontsize{8}{10}\selectfont (Generalized) Sharpness}%
    \subfigure[$ B = 1 $][b]{%
    \includegraphics[trim=0 0.01in 0 0.1in, clip, width=1.8in]{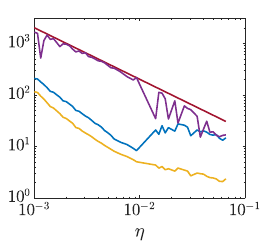}}%
    \subfigure[$ B = 2 $][b]{%
    \includegraphics[trim=0 0.01in 0 0.1in, clip, width=1.8in]{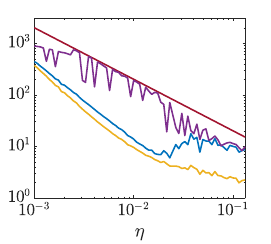}}%
    \subfigure[$ B = 4 $][b]{%
    \includegraphics[trim=0 0.01in 0 0.1in, clip, width=1.8in]{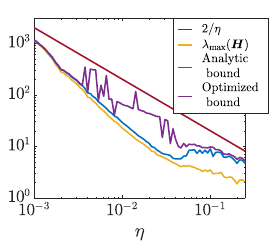}}\\%
    \rotatebox[origin=l]{90}{\hspace{1.1cm}\fontsize{8}{10}\selectfont (Generalized) Sharpness}%
    \subfigure[$ B = 8 $][b]{%
    \includegraphics[trim=0 0.01in 0 0.1in, clip, width=1.8in]{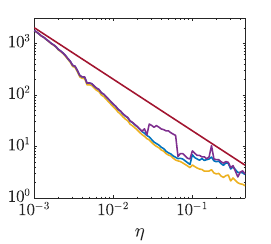}}%
    \subfigure[$ B = 16 $][b]{%
    \includegraphics[trim=0 0.01in 0 0.1in, clip, width=1.8in]{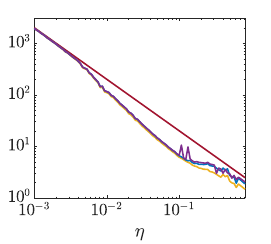}}%
    \subfigure[$ B = 32 $][b]{%
    \includegraphics[trim=0 0.01in 0 0.1in, clip, width=1.8in]{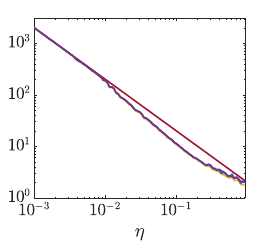}}
    \caption{\textbf{(Generalized) Sharpness vs. step size.} We trained single hidden-layer ReLU networks using varying step sizes and batch sizes on MNIST dataset. For each pair of hyper-parameters $ (\eta, B) $, we measured the sharpness of the minimum (yellow), our necessary condition for stability (blue), and the optimized bound (purple), which their relations are given in \eqref{eq:relationship}. We see that for small batch sizes $ B = 1 $ and $ B = 2 $, the optimized bound \eqref{eq:optimized bound} coincides with $ 2/\eta $, confirming that SGD converged at the edge of stability ($\eta = \etaVar$). For additional insights and detail, see  Sec.~\ref{sec:Experiments}.}%
    \label{Fig:A}%
\end{figure}

Figure~\ref{Fig:A} depicts the expressions in \eqref{eq:relationship} versus the step size for six batch sizes. We see that for~$ B~=~1 $ and $B=2$, the gap between $ 2/\eta $ (red) and the optimized bound \eqref{eq:optimized bound} (purple) upon convergence is small. Particularly, they coincide over a wide range of step sizes $\eta$. Since the generalized sharpness $ {\lambda_{\max}( \CC^{\dagger} \DD ) }$ must reside between those two curves, we can deduce two things: (a)~Our theory correctly predicts the stability threshold, while SGD converged at the edge of stability (as designed in our experiment); (b)~For small batches, the second order moment matrix that maximizes \eqref{eq:impicit SGD stability condition} is rank one. As the batch size increases, the two curves draw apart, indicating that the rank of the dominant second-order moment matrix becomes larger. Furthermore, the gap between our simple necessary condition \eqref{eq:necessary condition lambda max} (blue) and the trivial bound of $ 2/\lambda_{\max}(\HH) $ (yellow) is large for high learning rates and small for small step sizes. This gap represents the variance of the widths of the minima of the per-sample losses (corresponding to the widths of the quadratic functions $ \{  (\params-\params^* )^\transpose \HH_i (\params - \params^*) \} $) in the direction of $ \vv_{\max} $, the top eigenvector of~$\HH$. Thus we find that for small learning rates, this variance is small and the model is aligned in this direction, and for large learning rates, this variance is high.

A comment is in place regarding the fluctuation of the optimized bound. As described above, the value of this bound is obtained by an optimization problem which we solved using GD for each pair of step size and batch size. It may be that we have not found the global optimum for every step size, and got stuck at local maxima for some set of hyperparameters. This can explain why the curve falls down and then comes up again at some of the learning rates. Additionally, as we mentioned, the optimized bound is equivalent to restricting the optimization problem from \eqref{eq:impicit SGD stability condition} to rank one symmetric matrices. It is possible that for some minima of the loss, the optimal matrix of \eqref{eq:impicit SGD stability condition} is rank one, and therefore the bound is tight, while for others the rank is higher and thus the bound is not tight. Further study of SGD is needed to determine the cause of this behavior, which we leave for future work. For more details and experimental results, please see App.~\ref{App:Additional experimental results}.

\section{Related work}\label{sec:RelatedWork}

% theoretical
The stability of SGD in the vicinity of minima has been previously studied in multiple works. On the theoretical side, \citet{wu2018sgd} examined stability in the mean square sense and gave an implicit sufficient condition. \citet{granziol2022learning} used random matrix theory to find the maximal stable learning rate as a function of the batch size. Their work assumes some conditions on the Hessian's noise caused by batching, and the result holds in the limit of an infinite number of samples and batch size. \citet{velikanov2023a} examined SGD with momentum and derived a bound on the maximal learning rate. Their derivation uses ``spectrally expressible'' approximations and the result is given implicitly through a moment-generating function. \citet{ma2021on} studied the dynamics of higher moments of SGD and gave an implicit necessary and sufficient condition for stability (see Thm.~\ref{thm:stability expectation general} and the discussion following it). \citet{wu2022alignment} gave a necessary condition for stability via the alignment property. However, the result assumes and uses a lower bound on a property they coin ``alignment'' but an analytic bound for this alignment property is lacking for the general case. \citet{ziyin2023probabilistic} studied the stability of SGD in probability, rather than in mean square. Since convergence in probability is a weaker requirement, theoretically, SGD can converge with high probability to minima which are unstable in the mean square sense. Indeed, their theory predicts that SGD can converge far beyond GD's threshold. Yet this did not happen in extensive experiments done in \citet[App.~G]{cohen2021gradient} and \citet{gilmer2022a}. Finally, \citet{mulayoffNeurips} analyzed the stability in non-differentiable minima, and gave a necessary condition for a minimum to be ``strongly stable'', \ie such that SGD does not escape a ball with a given radius from the minimum.

% related stuff
\citet{pmlr-v139-liu21ad} studied the covariance matrix of the stationary distribution of the iterates in the vicinity of minima. \citet{ziyin2022strength} improved their results while deriving an implicit equation that relates this covariance to the covariance of the gradient noise. However, these works do not discuss the conditions under which the dynamics converge to the stationary state. \citet{lee2023a} studied the stability of SGD along its trajectory and gave an explicit exact condition. Yet their result does not apply to minima, since the denominator in their condition vanishes at minima.

% empirical
On the empirical side, \citet{cohen2021gradient} examined the behavior of GD, and showed that it typically converges at the edge of stability. Additionally, for SGD (see their App.~G) they found that with large batches, the sharpness behaves similarly to full-batch gradient descent. Moreover, they found that the smaller the batch size, the lower the sharpness at the converged minimum. \citet{gilmer2022a} studied how the curvature of the loss affects the training dynamics in multiple settings. They observed that SGD \emph{with momentum} is stable only when the optimization trajectory primarily resides in a region of parameter space where $ \eta \lesssim 2/\lambda_{\max}(\HH) $. Further experimental results in \citet{Jastrzebski2020The, jastrzebski2018on} show that the sharpness along the trajectory of SGD is implicitly regularized.

\section{Conclusion}\label{sec:Conclusion}
We presented an explicit threshold on SGD's step size, which is both necessary and sufficient for guaranteeing mean-square stability. We showed that this threshold is a monotonically non-decreasing function of the batch size, which implies that decreasing the batch size can only make the process less stable. Additionally, we interpreted the role of the batch size~$B$ through an equivalent process that takes in each iteration either a full batch gradient step or a single sample gradient step. Our interpretation highlights that even with moderate batch sizes, SGD's stability threshold is very close to that of GD. We also proved simpler necessary conditions for stability, which depend on the batch size, and are easier to compute. Finally, we verified our theory through experiments on MNIST.

% Acknowledgments---Will not appear in anonymized version
\acks{The research of Rotem Mulayoff was partially supported by the Planning and Budgeting Committee of the Israeli Council for Higher Education, and by the Andrew and Erna Finci Viterbi Graduate Fellowship. The research of Tomer Michaeli was partially supported by the Israel Science Foundation (grant no. 2318/22) and by the Ollendorff Minerva Center, ECE faculty, Technion.}

% Bibliography
\bibliography{StableMinimaRef}

\clearpage
\appendix
\section{Notations and the Kronecker product}\label{app:notations and Kronecker}
Throughout our derivations, we use the following notations.
{\renewcommand{\arraystretch}{1}
\begin{table}[h]
	\begin{tabular}{c l}
		$ a $ 								& Lower case non-bold letters for scalars \\ 
		$ \boldsymbol{a} $ 					& Lower case bold letters for vectors \\ 
		$ \boldsymbol{A} $ 					& Upper case bold for matrices \\
  	$ \boldsymbol{a}_{[p]}, \ \boldsymbol{A}_{[\ell,p]} $ & $p$'th element of $\boldsymbol{a}$, $(\ell,p)$ element of $ \boldsymbol{A} $  \\
		$ \boldsymbol{A}^{\transpose} $ 	& Transpose of $ \boldsymbol{A} $ \\
		$ \boldsymbol{A}^{\dagger} $ 		& Moore–Penrose inverse of $ \boldsymbol{A} $ \\
		$ \PP_{\mathcal{V}} $ 				& Orthogonal projection matrix onto the subspace $ \mathcal{V} $ \\
		$ \myNull{\boldsymbol{A}} $ 		& Null space of $ \boldsymbol{A} $ \\
		$ \myNullOrtho{\boldsymbol{A}} $ 	& Orthogonal complement of the null space of $ \boldsymbol{A} $ \\
		$ \myRange{\boldsymbol{A}} $ 		& Range of $ \boldsymbol{A} $ \\
		$ \myRangeOrtho{\boldsymbol{A}} $ 	& Orthogonal complement of the range of $ \boldsymbol{A} $ \\
		$ \otimes $ 						& Kronecker product \\
		$ \oplus $ 							& Kronecker sum \\
        $ ^{ \odot k } $                    & $k$'th Hadamard power \\
		$ \E $ 								& Expectation \\
		$ \prob $ 							& Probability \\
		$ \norm{\boldsymbol{a}} $ 			& Euclidean norm of $\boldsymbol{a} $ \\
		$ \norm{\boldsymbol{A}} $ 			& Top singular value of $\boldsymbol{A} $ \\
		$ \normF{\boldsymbol{A}} $ 			& Frobenius norm of $\boldsymbol{A} $ \\
		$ \rho(\boldsymbol{A}) $ 			& Spectral radius of $\boldsymbol{A} $ \\
		$ \mathrm{vec}(\boldsymbol{A}) $ 	& Vectorization of $\boldsymbol{A} $ (column stack) \\
		$ \mathrm{vec}^{-1}(\boldsymbol{a})$& Reshaping $\boldsymbol{a} $ back to $ d \times d $ matrix \\
		$ \loss $ 							& Loss function \\
		$ \params $ 						& Parameters vector of the loss \\
		$ \params^{*} $ 					& Minimum point of the loss \\
		$ d $ 								& Dimension of $ \params $ \\
		$ n $ 								& Number of training samples  \\
		$ \eta $ 							& Step size \\
		$ B $ 								& Batch size \\
		$ p $ 								& Defined to be $(n-B)/\big(B(n-1)\big) $ \\
        $\Identity_d $                      & The $d \times d $ identity matrix (when the dimensions are clear, the subscript is omitted) \\
		$ \CovMat $ 						& Second moment matrix \\
		$ \HH $ 							& Hessian of the full loss at $ \params^{*} $ \\
		$ \HH_i $ 							& Hessian of the loss of the sample $ i $ at $ \params^{*} $ \\
		$ \grad_i $ 						& Gradient of the loss of the sample $ i $ at $ \params^{*} $ \\
		$ \boldsymbol{a}^{\myPar} $ 		& Projection of $ \boldsymbol{a} $ onto the null space of $ \HH $ \\ 
		$ \boldsymbol{a}^{\myPerp} $ 		& Projection of $ \boldsymbol{a} $ onto the orthogonal complement of the null space of $ \HH $ \\
		$ \PSDset{d} $ 						& The set of all positive semi-definite (PSD) matrices over $ \R^{d \times d} $ \\
		$ \Sb^{d-1} $						& Unit sphere in $ \R^{d} $
	\end{tabular}
	\caption{Table of notations}
\end{table}
}

{\noindent Further notations that we use are given below.}
\begin{align}\label{eq:CovDef}
    \MeanVec_t \triangleq \EE\left[\params_{t} - \params^*\right], \qquad & \qquad \CovMat_t \triangleq \EE \left[\left(\params_{t} - \params^*\right)\left(\params_{t} - \params^*\right)^\transpose\right],\nonumber\\
    \MeanVec_t^{\myPerp} \triangleq \EE\left[\params_{t}^{\myPerp} - \params^{*\myPerp} \right], \qquad & \qquad \CovMat_t^{\myPerp} \triangleq \EE \left[\left(\params_{t}^{\myPerp} - \params^{*\myPerp}\right)\left(\params_{t}^{\myPerp} - \params^{*\myPerp}\right)^\transpose\right],\nonumber\\
    \MeanVec_t^{\myPar} \triangleq \EE\left[\params_{t}^{\myPar} - \params^{*\myPar}\right], \qquad & \qquad \CovMat_t^{\myPar} \triangleq \EE \left[\left(\params_{t}^{\myPar} - \params^{*\myPar}\right)\left(\params_{t}^{\myPar} - \params^{*\myPar}\right)^\transpose\right].
\end{align}
Additionally, we make extensive use of the following properties of the Kronecker product throughout the derivations. For any matrices $\MM_1,\MM_2,\MM_3,\MM_4$,
\begin{align*}
    \vectorization{\MM_1 \MM_2 \MM_3}  & = \big( \MM_3^\transpose \otimes \MM_1 \big) \vectorization{ \MM_2} \tag{P1}, \label{eq:KroneckerProperty1} \\
    \big( \MM_1 \otimes \MM_2 \big)^\transpose  & =  \MM_1^\transpose \otimes \MM_2^\transpose \tag{P2} \label{eq:KroneckerProperty2}, \\
    \big( \MM_1 \otimes \MM_2 \big) \big( \MM_3 \otimes \MM_4 \big) & = \big( \MM_1 \MM_3 \big) \otimes \big( \MM_2 \MM_4 \big) \tag{P3}, \label{eq:KroneckerProperty3} \\
    \left[\vectorization{\MM_1}\right]^\transpose (\MM_2 \otimes \MM_3) \vectorization{\MM_4}
    & = \Tr\left(\MM_1^\transpose \MM_3 \MM_4 \MM_2^\transpose \right) \tag{P4} .  \label{eq:quadratic form over kronecker}
\end{align*}
Finally, we give here the definition of Kronecker sum. If $\MM_1$ is $ d_1 \times d_1 $, $ \MM_2 $ is $d_2 \times d_2$ and $\Identity_d $ denotes the $d \times d $ identity matrix then
\begin{equation} \label{eq:Kronecker sum}
    \MM_1 \oplus \MM_2 = \MM_1 \otimes \Identity_{d_2} + \Identity_{d_1} \otimes \MM_2.
\end{equation}

\section{Stability of the first and second moments}\label{app:stability conditions proof}
Using our notation (see App.~\ref{app:notations and Kronecker}), for all $ \vv \in \R^d $  we have $ \vv^{\myPerp} = \PP_{\myNullOrtho{\HH}} \vv $ and $ \vv^{\myPar} = \PP_{\myNull{\HH}} \vv $ . Since $ \HH $ is symmetric,
\begin{equation}\label{eq:projection and Hessian}
    \PP_{\myNull{\HH}}\HH = \HH \PP_{\myNull{\HH}} = \zeroVec, \qquad \text{and} \qquad \PP_{\myNullOrtho{\HH}}\HH = \HH \PP_{\myNullOrtho{\HH}} = \HH.
\end{equation}
If $ \HH_i \in \PSDset{d} $ for all $ i \in [n] $, then the null space of $\HH$ is contained in the null space of each $\HH_i$, and therefore we also have that
\begin{equation}\label{eq:projection and PSD Hessians}
    \PP_{\myNull{\HH}}\HH_i = \HH_i \PP_{\myNull{\HH}} = \zeroVec, \qquad \text{and} \qquad \PP_{\myNullOrtho{\HH}}\HH_i = \HH_i \PP_{\myNullOrtho{\HH}} = \HH_i.
\end{equation}

\subsection{Linearized dynamics}
The linearized dynamics near $ \params^* $ is
\begin{equation}
    \params_{t+1} = \params_{t} - \frac{\eta}{B} \sum_{i \in  \batch_t } \HH_i (\params_t - \params^* ) -\frac{\eta}{B} \sum_{i \in  \batch_t } \grad_i.
\end{equation}
Therefore,
\begin{align}\label{eq:Linearized full dynamics}
    \params_{t+1} - \params^*
    & = \params_t -\params^* - \frac{\eta}{B} \sum_{i \in  \batch_t } \HH_i (\params_t - \params^* ) -\frac{\eta}{B} \sum_{i \in  \batch_t } \grad_i \nonumber \\
    & = \left( \Identity - \frac{\eta}{B} \sum_{i \in  \batch_t } \HH_i \right) (\params_t - \params^* ) -\frac{\eta}{B} \sum_{i \in  \batch_t } \grad_i.
\end{align}
Here we assume that the batches are chosen uniformly at random, independently across iterations.

\paragraph{Linearized dynamics in the orthogonal complement.}
Under the assumption that $ \HH_i \in \PSDset{d} $ for all $ i \in [n] $, the linearized dynamics in the orthogonal complement is given by
\begin{align}\label{eq:Linearized dynamics over the orthogonal complement}
    \params^{\myPerp}_{t+1} - \params^{*\myPerp}
    & = \PP_{\myNullOrtho{\HH}} \left(\params_{t+1} - \params^* \right) \nonumber \\
    & = \PP_{\myNullOrtho{\HH}} \left( \Identity - \frac{\eta}{B} \sum_{i \in  \batch_t } \HH_i \right) (\params_t - \params^* ) -\frac{\eta}{B} \sum_{i \in  \batch_t } \PP_{\myNullOrtho{\HH}} \grad_i \nonumber \\
    & =  \left( \PP_{\myNullOrtho{\HH}} - \frac{\eta}{B} \sum_{i \in  \batch_t } \PP_{\myNullOrtho{\HH}} \HH_i \right) (\params_t - \params^* ) -\frac{\eta}{B} \sum_{i \in  \batch_t } \grad_i^{\myPerp} \nonumber \\
    & =  \left( \PP_{\myNullOrtho{\HH}} - \frac{\eta}{B} \sum_{i \in  \batch_t } \HH_i \PP_{\myNullOrtho{\HH}} \right) (\params_t - \params^* ) -\frac{\eta}{B} \sum_{i \in  \batch_t } \grad_i^{\myPerp} \nonumber \\
    & =  \left( \Identity - \frac{\eta}{B} \sum_{i \in  \batch_t } \HH_i  \right) \PP_{\myNullOrtho{\HH}}(\params_t - \params^* ) -\frac{\eta}{B} \sum_{i \in  \batch_t } \grad_i^{\myPerp} \nonumber \\
    & =  \left( \Identity - \frac{\eta}{B} \sum_{i \in  \batch_t } \HH_i  \right) \left( \params^{\myPerp}_{t+1} - \params^{*\myPerp} \right) -\frac{\eta}{B} \sum_{i \in  \batch_t } \grad_i^{\myPerp}.
\end{align}
Here, in the second step, we used \eqref{eq:Linearized full dynamics}, and in the fourth we used \eqref{eq:projection and PSD Hessians}.

\paragraph{Linearized dynamics in the null space.}
Under the assumption that $ \HH_i \in \PSDset{d} $ for all $ i \in [n] $, the linearized dynamics in the null space is given by
\begin{align}\label{eq:Linearized dynamics over the null space}
    \params^{\myPar}_{t+1} - \params^{*\myPar}
    & = \PP_{\myNull{\HH}} \left(\params_{t+1} - \params^* \right) \nonumber \\
    & = \PP_{\myNull{\HH}} \left( \Identity - \frac{\eta}{B} \sum_{i \in  \batch_t } \HH_i \right) (\params_t - \params^* ) -\frac{\eta}{B} \sum_{i \in  \batch_t } \PP_{\myNull{\HH}} \grad_i \nonumber \\
    & = \left( \PP_{\myNull{\HH}} - \frac{\eta}{B} \sum_{i \in  \batch_t } \PP_{\myNull{\HH}} \HH_i \right) (\params_t - \params^* ) -\frac{\eta}{B} \sum_{i \in  \batch_t } \grad_i^{\myPar} \nonumber \\
    & = \left( \PP_{\myNull{\HH}} - \zeroVec \right) (\params_t - \params^* ) -\frac{\eta}{B} \sum_{i \in  \batch_t } \grad_i^{\myPar} \nonumber \\
    & = \PP_{\myNull{\HH}} (\params_t - \params^* ) -\frac{\eta}{B} \sum_{i \in  \batch_t } \grad_i^{\myPar} \nonumber \\
    & = \params^{\myPar}_{t} - \params^{*\myPar} -\frac{\eta}{B} \sum_{i \in  \batch_t } \grad_i^{\myPar}.
\end{align}
Again, in the second step, we used \eqref{eq:Linearized full dynamics}, and in the fourth we used \eqref{eq:projection and PSD Hessians}. Overall,
\begin{equation}\label{eq:Linearized dynamics over the null space general case}
    \params^{\myPar}_{t+1} = \params^{\myPar}_{t} -\frac{\eta}{B} \sum_{i \in  \batch_t } \grad_i^{\myPar}.
\end{equation}
Note that if $ \grad_i^{\myPar} = \zeroVec $ for all $ i \in [n] $ then
\begin{equation}\label{eq:linearized dynamics in the null space}
    \params^{\myPar}_{t+1} = \params^{\myPar}_{t}.
\end{equation}

\subsection{Mean dynamics (proof of Theorem~\ref{thm:stability first moment})}\label{app:Mean dynamics}
First, we compute the mean of the linearized dynamics.
\begin{align}\label{eq:Mean dynamics}
    \MeanVec_{t+1} = \EE\left[\params_{t+1} - \params^*\right]
    & = \EE\left[\left( \Identity - \frac{\eta}{B} \sum_{i \in  \batch_t } \HH_i \right) (\params_t - \params^* ) \right] - \EE\left[ \frac{\eta}{B} \sum_{i \in  \batch_t } \grad_i \right] \nonumber \\
    & = \EE\left[\left( \Identity - \frac{\eta}{B} \sum_{i \in  \batch_t } \HH_i \right) \EE\left[ (\params_t - \params^* ) \middle|  \batch_t \right] \right] - \frac{\eta}{n} \sum_{i= 1}^n \grad_i \nonumber \\
    & = \left( \Identity - \eta \HH \right) \EE\left[ (\params_t - \params^* ) \right] \nonumber \\
    & = \left( \Identity - \eta \HH \right) \MeanVec_{t},
\end{align}
where in the second step we used the law of total expectation, and in the third step we used \eqref{eq:gradEqZero}. This system is stable if and only if the spectral radius $ \rho(\Identity - \eta \HH) \leq 1 $. This condition is equivalent to $ \lambda_{\max}(\HH) \leq 2/\eta $ (see proof in, \eg \cite{cohen2021gradient,mulayoffNeurips}), thus proving point 2 of Thm.~\ref{thm:stability first moment}.

\paragraph{Mean dynamics in the orthogonal complement.}
In a similar manner, taking the expectation of both sides of \eqref{eq:Linearized dynamics over the orthogonal complement} and using \eqref{eq:gradEqZero}, we get
\begin{equation}
    \MeanVec_{t+1}^{\myPerp} = \left( \Identity - \eta \HH \right)\MeanVec_{t}^{\myPerp},
\end{equation}
Note that for all $ t \geq 0 $, 
\begin{equation}\label{eq:projected transition matrix first moment}
    \MeanVec_{t+1}^{\myPerp} = \PP_{\myNullOrtho{\HH}} \MeanVec_{t+1}^{\myPerp} = \PP_{\myNullOrtho{\HH}} \left( \Identity - \eta \HH \right)\MeanVec_{t}^{\myPerp} =  \left( \PP_{\myNullOrtho{\HH}} - \eta \HH \right)\MeanVec_{t}^{\myPerp} .
\end{equation}
Namely, $ \MeanVec_{t}^{\myPerp} = \left( \PP_{\myNullOrtho{\HH}} - \eta \HH \right)^t\MeanVec_{0}^{\myPerp} $, and thus
\begin{equation}
    \norm{ \MeanVec_{t}^{\myPerp} } = \norm{  \left( \PP_{\myNullOrtho{\HH}} - \eta \HH \right)^t\MeanVec_{0}^{\myPerp}} \leq  \norm{   \PP_{\myNullOrtho{\HH}} - \eta \HH }^t\norm{\MeanVec_{0}^{\myPerp}}.
\end{equation}
It is easy to show that
\begin{equation}
    \norm{ \PP_{\myNullOrtho{\HH}} - \eta \HH } = \max_{\lambda_i(\HH) \neq 0 } \left\{ \left|1-\eta \lambda_{i}(\HH) \right| \right\}.
\end{equation}
Therefore, if $ 0 < \eta < 2/\lambda_{\max} $, we have that $ \norm{  \PP_{\myNullOrtho{\HH}} - \eta \HH } < 1  $ and thus
\begin{equation}
    \lim_{t \to \infty} \norm{ \MeanVec_{t}^{\myPerp} } \leq  \lim_{t \to \infty} \norm{\PP_{\myNullOrtho{\HH}} - \eta \HH }^t\norm{\MeanVec_{0}^{\myPerp}} = 0.
\end{equation}
This proves point 3 of Thm.~\ref{thm:stability first moment}.

\paragraph{Mean dynamics in the null space.}
Taking the expectation of both sides of \eqref{eq:Linearized dynamics over the null space} and using \eqref{eq:gradEqZero}, we obtain
\begin{equation}    
    \MeanVec_{t+1}^{\myPar} = \MeanVec_{t}^{\myPar}.
\end{equation}
This demonstrates that for all $t \geq 0$,
\begin{equation}
    \EE\left[\params_{t}^{\myPar} - \params^{*\myPar}\right]
    = \MeanVec_{t}^{\myPar}
    = \MeanVec_{0}^{\myPar}
    = \EE\left[\params_{0}^{\myPar} - \params^{*\myPar}\right],
\end{equation}
so that
\begin{equation}
    \EE\left[\params_{t}^{\myPar}\right] = \EE\left[\params_{0}^{\myPar}\right].
\end{equation}
This proves Point 1 of Thm.~\ref{thm:stability first moment}.

\subsection{Covariance dynamics for the orthogonal complement}
Before providing a complete proof for Thm.~\ref{thm:stability threshold expectation} (see App.~\ref{app:stability threshold expectation proof}) and Thm.~\ref{thm:stability threshold expectation regular minima} (see App.~\ref{app:stability threshold expectation regular minima proof}), we next examine the evolution over time of the covariance of the parameter vector. We start by focusing on the orthogonal complement space. Define
\begin{equation}\label{eq:transition matrix and vector v}
    \TransMat_t =  \Identity - \frac{\eta}{B} \sum_{i \in  \batch_t } \HH_i \qquad \text{and} \qquad \vv_t =  \frac{\eta}{B} \sum_{i \in  \batch_t } \grad_i,
\end{equation}
so that \eqref{eq:Linearized dynamics over the orthogonal complement} can be compactly written as
\begin{equation}\label{eq:Linearized dynamics over the orthogonal complement new}
    \params^{\myPerp}_{t+1} - \params^{*\myPerp} = \TransMat_t \left( \params^{\myPerp}_{t} - \params^{*\myPerp} \right) -\vv_t^{\myPerp}.
\end{equation}
Recall that this holds under the assumption that $ \HH_i \in \PSDset{d} $ for all $ i \in [n] $. Note that $\{\TransMat_t \}$ are i.i.d.\ and that $\params^{\myPerp}_{t}$ is constructed from $\TransMat_0,\ldots,\TransMat_{t-1}$, so that $\params^{\myPerp}_{t}$ is independent of $\TransMat_t$. We therefore have
\begin{align}
    \CovMat_{t+1}^{\myPerp}
    & = \EE\left[ \left( \params^{\myPerp}_{t+1} - \params^{*\myPerp} \right) \left( \params^{\myPerp}_{t+1} - \params^{*\myPerp}\right)^\transpose \right] \nonumber \\
    &  = \EE\left[ \left( \TransMat_t \left( \params^{\myPerp}_{t} - \params^{*\myPerp} \right) -\vv_t^{\myPerp} \right) \left(  \TransMat_t \left( \params^{\myPerp}_{t} - \params^{*\myPerp} \right) -\vv_t^{\myPerp} \right)^\transpose \right] \nonumber\\
    & = \EE\left[ \TransMat_t \left( \params^{\myPerp}_{t} - \params^{*\myPerp}\right) \left( \params^{\myPerp}_{t} - \params^{*\myPerp}\right)^\transpose \TransMat^\transpose_t \right]
    - \EE \left[ \TransMat_t \left( \params^{\myPerp}_{t} - \params^{*\myPerp} \right) (\vv^\myPerp_t)^{\transpose} \right]\nonumber\\
    & \quad - \EE \left[ \vv_t^{\myPerp}\left( \params^{\myPerp}_{t} - \params^{*\myPerp} \right)^\transpose \TransMat^\transpose_t \right]
    +  \EE\left[ \vv_t^{\myPerp}(\vv^\myPerp_t)^{\transpose} \right] \nonumber \\
    & = \EE\left[ \TransMat_t \EE \left[\left( \params^{\myPerp}_{t} - \params^{*\myPerp}\right) \left( \params^{\myPerp}_{t} - \params^{*\myPerp}\right)^\transpose \right] \TransMat^\transpose_t \right]
    - \EE \left[ \TransMat_t \EE\left[ \params^{\myPerp}_{t} - \params^{*\myPerp} \right] (\vv^\myPerp_t)^{\transpose} \right]\nonumber\\
    & \quad - \EE \left[ \vv_t^{\myPerp}\EE \left[ \params^{\myPerp}_{t} - \params^{*\myPerp} \right]^\transpose \TransMat^\transpose_t \right]
    +  \CovMat_{\vv}^{\myPerp}  \nonumber \\
    & = \EE\left[ \TransMat_t \CovMat_{t}^{\myPerp} \TransMat^\transpose_t \right] 
    - \EE \left[ \TransMat_t \MeanVec^{\myPerp}_{t} (\vv^\myPerp_t)^{\transpose} \right]
    - \EE \left[ \vv_t^{\myPerp} (\MeanVec^{\myPerp}_{t})^\transpose \TransMat^\transpose_t \right]
    +  \CovMat_{\vv}^{\myPerp} ,
\end{align}
where in the second equality we used \eqref{eq:Linearized dynamics over the orthogonal complement new}, and in the fourth the fact that $\TransMat_t $ is independent of $\params^{\myPerp}_{t} $. Using vectorization, the above equation can be written as
\begin{align}\label{eq:second moment evolution}
    \vectorization{\CovMat_{t+1}^{\myPerp}}
    & = \EE\left[ \vectorization{\TransMat_t \CovMat_{t}^{\myPerp} \TransMat^\transpose_t} \right]
    - \EE \left[ \vectorization{\TransMat_t \MeanVec_{t}^{\myPerp} (\vv^\myPerp_t)^{\transpose}} \right]
    - \EE \left[ \vectorization{ \vv^\myPerp_t (\MeanVec^{\myPerp}_{t})^\transpose \TransMat^\transpose_t} \right]
    + \vectorization{\CovMat_{\vv}^{\myPerp}} \nonumber \\
    & = \EE\left[\TransMat_t  \otimes \TransMat_t \right]  \vectorization{ \CovMat_{t}^{\myPerp}}
    - \EE \left[ \vv^\myPerp_t \otimes \TransMat_t \right] \MeanVec_{t}^{\myPerp}
    - \EE \left[ \TransMat_t \otimes \vv^\myPerp_t \right] \MeanVec_{t}^{\myPerp}
    + \vectorization{\CovMat_{\vv}^{\myPerp}} \nonumber \\
    & = \CovEvo \vectorization{ \CovMat_{t}^{\myPerp}}
    - \left(\EE \left[ \vv^\myPerp_t \otimes \TransMat_t \right] + \EE \left[ \TransMat_t \otimes \vv^\myPerp_t \right] \right) \MeanVec_{t}^{\myPerp}
    + \vectorization{\CovMat_{\vv}^{\myPerp}},
\end{align}
where we denoted 
\begin{equation}\label{eq:Q}
	\CovEvo \triangleq  \EE\left[\TransMat_t  \otimes \TransMat_t \right]. 
\end{equation}
Overall, the joint dynamics of $ \CovMat_{t}^{\myPerp} $ and $ \MeanVec_{t}^{\myPerp} $ is given by
\begin{equation}\label{eq:joint dynamics}
    \begin{pmatrix}
        \MeanVec_{t+1}^{\myPerp} \\
        \vectorization{\CovMat_{t+1}^{\myPerp}}
    \end{pmatrix}
    =
    \begin{pmatrix}
        \Identity - \eta \HH  &  \zeroVec \\
        - \EE \left[ \vv_t^{\myPerp} \otimes \TransMat_t \right] - \EE \left[ \TransMat_t \otimes \vv_t^{\myPerp} \right] & \CovEvo
    \end{pmatrix}
    \begin{pmatrix}
        \MeanVec_{t}^{\myPerp} \\
        \vectorization{\CovMat_{t}^{\myPerp}}
    \end{pmatrix}
    +
    \begin{pmatrix}
        \zeroVec \\
        \vectorization{\CovMat_{\vv}^{\myPerp}}
    \end{pmatrix}.
\end{equation}
In some cases, it is easier to look at a projected version of the transition matrix. In \eqref{eq:projected transition matrix first moment} we showed that 
\begin{equation}
	\MeanVec_{t+1}^{\myPerp}  = \left( \PP_{\myNullOrtho{\HH}} - \eta \HH \right)\MeanVec_{t}^{\myPerp} .
\end{equation}
Moreover, from \eqref{eq:second moment evolution},
\begin{align}\label{eq:covariance evolution in orthogonal subspace}
	\vectorization{\CovMat_{t+1}^{\myPerp}}
	& = \vectorization{\PP_{\myNullOrtho{\HH}} \CovMat_{t+1}^{\myPerp} \PP_{\myNullOrtho{\HH}}} \nonumber \\
	& = \left(\PP_{\myNullOrtho{\HH}} \otimes \PP_{\myNullOrtho{\HH}} \right) \vectorization{\CovMat_{t+1}^{\myPerp}} \nonumber \\
	& = \left(\PP_{\myNullOrtho{\HH}} \otimes \PP_{\myNullOrtho{\HH}} \right) \left(\CovEvo \vectorization{ \CovMat_{t}^{\myPerp}}
	- \left(\EE \left[ \vv^\myPerp_t \otimes \TransMat_t \right] + \EE \left[ \TransMat_t \otimes \vv^\myPerp_t \right] \right) \MeanVec_{t}^{\myPerp}
	+ \vectorization{\CovMat_{\vv}^{\myPerp}}\right) \nonumber \\
	& = \left(\PP_{\myNullOrtho{\HH}} \otimes \PP_{\myNullOrtho{\HH}} \right) \CovEvo \vectorization{ \CovMat_{t}^{\myPerp}} \nonumber \\
	& \quad - \left(\PP_{\myNullOrtho{\HH}} \otimes \PP_{\myNullOrtho{\HH}} \right) \left(\EE \left[ \vv^\myPerp_t \otimes \TransMat_t \right] + \EE \left[ \TransMat_t \otimes \vv^\myPerp_t \right] \right) \MeanVec_{t}^{\myPerp}
	+ \vectorization{\CovMat_{\vv}^{\myPerp}} .
\end{align}
Therefore, the linear system in \eqref{eq:joint dynamics} can be written as
\begin{align}\label{eq:projected joint dynamics}
    &\begin{pmatrix}
        \MeanVec_{t+1}^{\myPerp} \\
        \vectorization{\CovMat_{t+1}^{\myPerp}}
    \end{pmatrix}
    =\nonumber\\
    &\begin{pmatrix}
        \PP_{\myNullOrtho{\HH}} - \eta \HH  &  \zeroVec \\
        - \left(\PP_{\myNullOrtho{\HH}} \otimes \PP_{\myNullOrtho{\HH}} \right) (\EE \left[ \vv_t^{\myPerp} \otimes \TransMat_t \right] + \EE \left[ \TransMat_t \otimes \vv_t^{\myPerp} \right] )& \left(\PP_{\myNullOrtho{\HH}} \otimes \PP_{\myNullOrtho{\HH}} \right) \CovEvo
    \end{pmatrix}
    \begin{pmatrix}
        \MeanVec_{t}^{\myPerp} \\
        \vectorization{\CovMat_{t}^{\myPerp}}
    \end{pmatrix} \nonumber	\\
    & \qquad +
    \begin{pmatrix}
        \zeroVec \\
        \vectorization{\CovMat_{\vv}^{\myPerp}}
    \end{pmatrix}.
\end{align}

\subsection{The transition matrix of the covariance dynamics}
We now proceed to develop an explicit expression for the covariance transition matrix $\CovEvo$ of \eqref{eq:Q}. We have
\begin{align}
	\CovEvo = \EE\left[\TransMat_t  \otimes \TransMat_t \right]
	& = \EE\left[ \left(\Identity - \frac{\eta}{B} \sum_{i \in  \batch_t } \HH_i \right) \otimes \left(\Identity - \frac{\eta}{B} \sum_{i \in  \batch_t } \HH_i \right) \right] \nonumber \\
	& = \EE\left[ \Identity - \frac{\eta}{B} \sum_{i \in  \batch_t } \left( \Identity \otimes \HH_i + \HH_i \otimes \Identity \right) + \frac{\eta^2}{B^2} \sum_{i,j \in  \batch_t } \HH_i \otimes \HH_j \right] \nonumber \\
	& = \Identity - \eta \left( \Identity \otimes \HH + \HH \otimes \Identity \right) + \eta^2  \EE\left[ \frac{1}{B^2} \sum_{i,j \in  \batch_t }  \HH_i \otimes \HH_j \right].
\end{align}
Note that
\begin{align}
	\EE\left[ \frac{1}{B^2} \sum_{i,j \in  \batch_t }  \HH_i \otimes \HH_j \right]
	& = \EE\left[ \frac{1}{B^2} \sum_{i \neq j \in  \batch_t }  \HH_i \otimes \HH_j + \frac{1}{B^2}\sum_{i \in  \batch_t } \HH_i \otimes \HH_i \right] \nonumber \\
	& = \frac{1}{B^2} \times B(B-1) \EE\left[\HH_i \otimes \HH_j  \middle| i \neq j  \in \batch_t \right] + \frac{1}{B^2} \E\left[\sum_{i \in  \batch_t } \HH_i \otimes \HH_i \right] \nonumber \\
	& =  \frac{B-1}{B} \EE\left[ \HH_i \otimes \HH_j \middle| i \neq j  \in \batch_t \right]  + \frac{1}{nB}\sum_{i = 1}^n  \HH_i \otimes \HH_i.
\end{align}
Specifically using symmetry and \eqref{eq:Hessian Def}, 
\begin{align}
	\EE\left[ \HH_i \otimes \HH_j \middle| i \neq j  \in \batch_t \right]
	&  = \sum_{i \neq j = 1 }^n  \frac{1}{n(n-1)} \HH_i \otimes \HH_j \nonumber \\
	& = \frac{n}{(n-1)} \frac{1}{n^2} \sum_{i \neq j = 1 }^n \HH_i \otimes \HH_j \nonumber \\
	& = \frac{n}{(n-1)} \left( \HH \otimes \HH - \frac{1}{n^2} \sum_{i = 1 }^n \HH_i \otimes \HH_i \right).
\end{align}
Hence,
\begin{align}
	\frac{B-1}{B} \EE\left[ \HH_i \otimes \HH_j \middle| i \neq j  \in \batch_t \right]
	& = \frac{n(B-1)}{B(n-1)} \left( \HH \otimes \HH - \frac{1}{n^2} \sum_{i = 1 }^n \HH_i \otimes \HH_i \right)  \nonumber \\
	& = \frac{n(B-1)}{B(n-1)}  \HH \otimes \HH - \frac{B-1}{Bn(n-1)} \sum_{i = 1 }^n \HH_i \otimes \HH_i \nonumber \\
	& = \HH \otimes \HH - \frac{n-B}{B(n-1)}  \HH \otimes \HH - \frac{B-1}{Bn(n-1)} \sum_{i = 1 }^n \HH_i \otimes \HH_i.
\end{align}
Overall,
\begin{align}
	\EE\left[ \frac{1}{B^2} \sum_{i,j \in  \batch_t }  \HH_i \otimes \HH_j \right]
	& = \HH \otimes \HH - \frac{n-B}{B(n-1)}  \HH \otimes \HH - \frac{B-1}{Bn(n-1)} \sum_{i = 1 }^n \HH_i \otimes \HH_i  \nonumber \\ 
	& \quad + \frac{1}{nB}\sum_{i = 1}^n  \HH_i \otimes \HH_i \nonumber \\
	& = \HH \otimes \HH - \frac{n-B}{B(n-1)}  \HH \otimes \HH + \frac{n-B}{B(n-1)}\times \frac{1}{n}\sum_{i = 1}^n  \HH_i \otimes \HH_i \nonumber \\
	& = \HH \otimes \HH  + \frac{n-B}{B(n-1)} \left( \frac{1}{n}\sum_{i = 1}^n  \HH_i \otimes \HH_i - \HH \otimes \HH \right).
\end{align}
Therefore, we have that  $ \CovEvo $ is given by
\begin{align}\label{eq:Covariance evolution matrix app}
    \CovEvo(B, \eta)
    & = \Identity - \eta \left( \Identity \otimes \HH + \HH \otimes \Identity \right) + \eta^2  \HH \otimes \HH  + \eta^2 \frac{n-B}{B(n-1)} \left( \frac{1}{n}\sum_{i = 1}^n  \HH_i \otimes \HH_i - \HH \otimes \HH \right) \nonumber \\
    & = (\Identity-\eta\HH)\otimes(\Identity-\eta\HH) + \frac{n-B}{B(n-1)} \frac{\eta^2}{n}  \smash{\sum_{i = 1}^n} (\HH_i \otimes \HH_i -\HH \otimes \HH ),
    \end{align}
which reproduce the result of \eqref{eq:Covariance evolution matrix}. Here we give an alternative form of $\CovEvo$, which is useful in many derivations. We set
\begin{equation}
	p = \frac{n-B}{B(n-1)},
\end{equation}
and continue from the first line in \eqref{eq:Covariance evolution matrix app}.
\begin{align}\label{eq:Covariance evolution matrix alternative}
	\CovEvo(B, \eta)
	& = \smash[t]{\Identity - \eta \left( \Identity \otimes \HH + \HH \otimes \Identity \right) + \eta^2  \HH \otimes \HH  + \eta^2 p \left( \frac{1}{n}\sum_{i = 1}^n  \HH_i \otimes \HH_i - \HH \otimes \HH \right) } \nonumber \\
	& = \Identity - \eta \left( \Identity \otimes \HH + \HH \otimes \Identity \right) + (1-p) \eta^2 \HH \otimes \HH  +  p \times \frac{\eta^2}{n}\sum_{i = 1}^n  \HH_i \otimes \HH_i \nonumber \\
    & = (1-p)\Identity - (1-p)\eta \left( \Identity \otimes \HH + \HH \otimes \Identity \right) + (1-p) \eta^2 \HH \otimes \HH \nonumber \\
    & \quad + p \Identity - p \eta \left( \Identity \otimes \HH + \HH \otimes \Identity \right) + p \times \frac{\eta^2}{n}\sum_{i = 1}^n  \HH_i \otimes \HH_i \nonumber \\
    & = (1-p)\bigg[\Identity - \eta \left( \Identity \otimes \HH + \HH \otimes \Identity \right) + \eta^2 \HH \otimes \HH\bigg] \nonumber \\
    & \quad + p \left[ \Identity - \eta \left( \Identity \otimes \HH + \HH \otimes \Identity \right) +  \frac{\eta^2}{n}\sum_{i = 1}^n  \HH_i \otimes \HH_i \right]\nonumber \\
    & = (1-p)\left(\Identity - \eta \HH \right) \otimes\left(\Identity - \eta \HH \right)  \nonumber \\
    & \quad + p \left[ \Identity - \eta \left( \Identity \otimes \Big(\frac{1}{n}\sum_{i = 1}^n  \HH_i \Big) + \Big(\frac{1}{n}\sum_{i = 1}^n  \HH_i \Big) \otimes \Identity \right) + \frac{\eta^2}{n}\sum_{i = 1}^n  \HH_i \otimes \HH_i \right]\nonumber \\
    & = (1-p)\left(\Identity - \eta \HH \right) \otimes\left(\Identity - \eta \HH \right)  
    % \nonumber \\  & \quad 
    + p \frac{1}{n}\sum_{i = 1}^n \bigg[ \Identity - \eta \left( \Identity \otimes \HH_i +   \HH_i \otimes \Identity \right) +  \eta^2 \HH_i \otimes \HH_i \bigg]\nonumber \\
	& = (1-p) (\Identity-\eta\HH)\otimes(\Identity-\eta\HH) + p \frac{1}{n} \sum_{i = 1}^n (\Identity-\eta\HH_i)\otimes(\Identity-\eta\HH_i),
\end{align}
where in the third step we add and subtract $ p \Identity - p \eta \left( \Identity \otimes \HH + \HH \otimes \Identity \right) $, and in the fifth we used \eqref{eq:Hessian Def}. Another useful form is the following. First, observe that
\begin{align}\label{eq:kronecker variance}
     \frac{1}{n}\sum_{i = 1}^n  \HH_i \otimes \HH_i - \HH \otimes \HH
    & = \frac{1}{n}\sum_{i = 1}^n  \HH_i \otimes \HH_i - \HH \otimes \HH -  \HH \otimes \HH +\HH \otimes \HH\nonumber\\
    & = \frac{1}{n}\sum_{i = 1}^n  \HH_i \otimes \HH_i - \HH \otimes \HH -  \HH \otimes \HH +\HH \otimes \HH\nonumber\\
    & = \frac{1}{n}\sum_{i = 1}^n  \HH_i \otimes \HH_i  -  \left(\frac{1}{n}\sum_{i = 1}^n  \HH_i\right) \otimes \HH - \HH \otimes \left(\frac{1}{n}\sum_{i = 1}^n  \HH_i\right) +\HH \otimes \HH\nonumber\\
    & = \frac{1}{n}\sum_{i = 1}^n  \HH_i \otimes \HH_i  - \frac{1}{n}\sum_{i = 1}^n  \HH_i \otimes \HH - \frac{1}{n}\sum_{i = 1}^n \HH \otimes \HH_i +\HH \otimes \HH\nonumber\\
    & = \frac{1}{n}\sum_{i = 1}^n \left( \HH_i - \HH\right)\otimes \left( \HH_i - \HH\right)
\end{align}
Then, starting from the first line in \eqref{eq:Covariance evolution matrix app} and using \eqref{eq:kronecker variance}, we have
\begin{align}\label{eq:Covariance evolution matrix alternative 2}
    \CovEvo(B, \eta)
    & = \Identity - \eta \left( \Identity \otimes \HH + \HH \otimes \Identity \right) + \eta^2  \HH \otimes \HH  + p \eta^2  \left( \frac{1}{n}\sum_{i = 1}^n  \HH_i \otimes \HH_i - \HH \otimes \HH \right) \nonumber \\
    & = \left( \Identity - \eta\HH \right) \otimes \left( \Identity - \eta\HH \right)  + \eta^2 p \left( \frac{1}{n}\sum_{i = 1}^n \left( \HH_i - \HH\right)\otimes \left( \HH_i - \HH\right) \right)
\end{align}
Finally, we give the $ \CovEvo $ in terms of $\CC$ and $\DD$. Again, we start from the first line in \eqref{eq:Covariance evolution matrix app}.
\begin{align}\label{eq:Q in terms of C and D}
	\CovEvo
	& = \Identity - \eta \left( \Identity \otimes \HH + \HH \otimes \Identity \right) + \eta^2  \HH \otimes \HH  + \eta^2 p \left( \frac{1}{n}\sum_{i = 1}^n  \HH_i \otimes \HH_i - \HH \otimes \HH \right) \nonumber \\
	& = \Identity - \eta \left( \Identity \otimes \HH + \HH \otimes \Identity \right) + \eta^2 \left[ (1-p) \times \HH \otimes \HH  + p \times \frac{1}{n}\sum_{i = 1}^n  \HH_i \otimes \HH_i \right] \nonumber \\
	& = \Identity + 2\eta \CC + \eta^2 \DD.
\end{align}

\subsection{Covariance matrix of the gradient noise}
We now develop an explicit expression for the covariance matrix of the gradient noise $\vv$ of \eqref{eq:transition matrix and vector v}. We have 
\begin{align}
    \CovMat_{\vv} = \EE\left[ \vv_t \vv_t^\transpose \right]
    & = \left(\frac{\eta}{B} \right)^2 \EE\left[ \sum_{ i,j \in  \batch_t } \grad_i\grad_j^\transpose\right] \nonumber\\
    & = \left(\frac{\eta}{B} \right)^2 \EE\left[ \sum_{ i \neq j \in  \batch_t } \grad_i\grad_j^\transpose +\sum_{ i \in  \batch_t } \grad_i\grad_i^\transpose \right] \nonumber\\
    & = \left(\frac{\eta}{B} \right)^2 \left( B(B-1) \EE\left[ \grad_i \grad_j^\transpose \middle| i \neq j \in  \batch_t  \right] + \frac{B}{n} \sum_{ i =  1}^n \grad_i\grad_i^\transpose\right).
\end{align}
Observe that
\begin{align}
    \EE\left[ \grad_i \grad^\transpose_j \middle| i \neq j  \in \batch_t \right]
    & = \sum_{i \neq j = 1 }^n  \frac{1}{n(n-1)} \grad_i \grad^\transpose_j \nonumber \\
    & = \frac{1}{n(n-1)} \left(\sum_{i, j = 1 }^n \grad_i \grad^\transpose_j -\sum_{ i =  1}^n \grad_i \grad^\transpose_i \right) \nonumber \\
    & = \frac{1}{n(n-1)} \left( \left( \sum_{i = 1 }^n \grad_i\right)\left( \sum_{i = 1 }^n \grad_i\right)^\transpose -\sum_{ i =  1}^n \grad_i \grad^\transpose_i \right) \nonumber \\
    & = -\frac{1}{n(n-1)}\sum_{ i =  1}^n \grad_i \grad^\transpose_i,
\end{align}
where in the last step we used \eqref{eq:gradEqZero}. Thus,
\begin{align}\label{eq:covariance of v}
    \CovMat_{\vv}
    & = \left(\frac{\eta}{B} \right)^2 \left(\frac{B}{n} -  \frac{B(B-1)}{n(n-1)}\right) \sum_{i=1}^n \grad_i\grad_i^\transpose \nonumber\\
    & = \left(\frac{\eta}{B} \right)^2 \times \frac{B(n-B)}{n(n-1)} \sum_{i=1}^n \grad_i\grad_i^\transpose\nonumber\\
    & = \eta^2\frac{n-B}{B(n-1)} \times \frac{1}{n}\sum_{ i =  1}^n \grad_i\grad_i^\transpose \nonumber\\
    & = \eta^2 p \CovMat_{\grad},
\end{align}
where we denoted
\begin{equation}
    \CovMat_{\grad} = \frac{1}{n}\sum_{i=1}^n \grad_i\grad_i^\transpose.
\end{equation}

\subsection{The Null spaces of $ \CC , \ \DD $ and $ \bE $}\label{app:Null space}
Let us now analyze the relation between the null spaces of  $ \CC$ ,  $\DD $ and $ \bE $. First, it is easy to see that under the assumption that $ \HH_i \in \PSDset{d} $ for all $ i \in [n] $, 
\begin{equation}\label{eq:null space of H and H_i s}
    \myNull{\HH} = \bigcap_{i = 1}^n \myNull{\HH_i},
\end{equation}
where $ \myNull{\cdot} $ denotes null space of a matrix. Here we show the following.
\begin{lemma}\label{lemma:Null space lemma}
    Assume that $ \HH_i \in \PSDset{d} $ for all $ i \in [n] $ and let
    \begin{align}
        \CC & = \frac{1}{2} \HH \oplus \HH, \nonumber \\
        \DD & = (1-p)\, \HH \otimes \HH + p \,\frac{1}{n}\sum_{i=1}^n\HH_i \otimes \HH_i, \nonumber \\
        \bE & = \frac{1}{n}\sum_{i=1}^n (\HH_i -\HH) \otimes (\HH_i -\HH).
    \end{align}
    Then $ \myNull{\CC} \subseteq \myNull{\DD} $ and $ \myNull{\CC} \subseteq \myNull{\bE} $.
\end{lemma}
\begin{proof}   
Let $ \uu \in \myNull{\CC} $ and denote $ \UU = \mathrm{vec}^{-1}(\uu) $, then
\begin{equation}
    \zeroVec = 2 \CC \uu = \HH \oplus \HH \uu = (\HH \otimes\Identity+\Identity \otimes \HH) \uu.
\end{equation}
In matrix form we get
\begin{equation}
    \UU \HH + \HH \UU = \zeroVec.
\end{equation}
Let us take the Frobenius norm, then
\begin{equation}
    \normF{\UU \HH + \HH \UU}^2 = \normF{\UU \HH}^2 + \normF{\HH \UU}^2+2\Tr\left((\UU \HH)^\transpose \HH \UU\right) = 0,
\end{equation}
where
\begin{equation}
    \Tr\left((\UU \HH)^\transpose \HH \UU\right) = \Tr\left(\HH \UU^\transpose \HH \UU\right)= \Tr\left(\HH^{\frac{1}{2}} \UU^\transpose \HH \UU \HH^{\frac{1}{2}}\right)\geq 0
\end{equation}
because $\HH^{\frac{1}{2}} \UU^\transpose \HH \UU \HH^{\frac{1}{2}}=(\HH^{\frac{1}{2}} \UU \HH^{\frac{1}{2}})^\transpose (\HH^{\frac{1}{2}}\UU \HH^{\frac{1}{2}})$ is PSD. This implies that
\begin{equation}
    \normF{\UU \HH}^2  = \normF{\HH \UU}^2 = 0.
\end{equation}
Thus, $ \uu \in \myNull{\CC} $ if and only if $ \UU \HH = \zeroVec $ and $ \HH \UU = \zeroVec $. Since the null space of $ \HH $ is the intersection of $ \{ \HH_i \} $ \eqref{eq:null space of H and H_i s} , we have that $ \UU $ also satisfies $ \HH_i \UU = \UU \HH_i = \zeroVec $ for all $ i \in [n] $. Now,
\begin{equation}
    \DD \uu  = (1-p)\, \HH \otimes \HH\uu + p \,\frac{1}{n}\sum_{i=1}^n\HH_i \otimes \HH_i\uu,
\end{equation}
and in matrix form,
\begin{equation}
    \mathrm{vec}^{-1}(\DD \uu) = (1-p)\, \HH \UU \HH + p \,\frac{1}{n}\sum_{i=1}^n\HH_i \UU \HH_i  = \zeroVec. 
\end{equation}
Namely, $ \uu \in \myNull{\DD} $. Similarly,
\begin{equation}
    \bE \uu  = \frac{1}{n}\sum_{i=1}^n (\HH_i -\HH) \otimes (\HH_i -\HH)\uu,
\end{equation}
and in matrix form,
\begin{equation}
    \mathrm{vec}^{-1}(\bE \uu) = \frac{1}{n}\sum_{i=1}^n (\HH_i -\HH) \UU (\HH_i -\HH)  = \zeroVec. 
\end{equation}
Namely, $ \uu \in \myNull{\bE} $.
\end{proof}

\subsection{Positivity of $ \CC $ and $ \DD $}\label{app:Positivity of C and D}

\subsubsection{Positivity of $ \CC $}
The eigenvalues of a Kronecker \emph{sum} are the pairwise sums of the eigenvalues of the summands \citep[Thm. 13.16]{laub2004matrix}. In App.~\ref{app:Recovering GD's stability condition} we explicitly show this for $\CC$, where we derive that the eigenvalues of $ \CC = \frac{1}{2} \HH \oplus \HH  $ are $ \frac{1}{2} \big( \lambda_i(\HH) + \lambda_j(\HH)  \big), \  i=1,\ldots,d ,\, j=1,\ldots,d $. Note that $ \HH $ is the Hessian of the loss at a minimum, and is therefore PSD. Therefore all eigenvalues of $ \HH $ are nonnegative, and as a consequence, the eigenvalues of $ \CC $ are nonnegative, \ie $ \CC $ is PSD.

\subsubsection{Positivity of $ \DD$}
The eigenvalues of a Kronecker \emph{product} are the pairwise products of the eigenvalues of the multiplicands \citep[Thm. 13.12]{laub2004matrix}. This property asserts that for any PSD matrix $ \MM $, namely with nonnegative eigenvalues, the Kronecker product $ \MM \otimes \MM $ is PSD. Note that $ \DD $ is defined as
\begin{equation}
	\DD = (1-p) \HH \otimes \HH + p \frac{1}{n}\sum_{i=1}^n  \HH_i \otimes \HH_i,
\end{equation}
with $ p \in [0,1] $. In our settings, \ie regular and interpolating minima, we consider Hessian matrices $ \{ \HH_i \} $ that are PSD. By the property above, all $ \{ \HH_i \otimes \HH_i \} $ are PSD, and also $ \{ \HH \otimes \HH \} $ is PSD. Therefore, $ \DD $ is a convex combination of PSD matrices, which is PSD.

\subsection{Proof of Theorem~\ref{thm:stability threshold expectation}}\label{app:stability threshold expectation proof}
We are now ready to prove Thm.~\ref{thm:stability threshold expectation}. 

\paragraph{First statement.} In \eqref{eq:linearized dynamics in the null space} we showed that for interpolating minima $ \params^{\myPar}_{t+1} = \params^{\myPar}_{t} $, which completes the proof for the first statement of the theorem.

\paragraph{Second statement.}
\citet{ma2021on} showed that the second moment $ \CovMat_t = [(\params_t-\params^*)(\params_t-\params^*)^\transpose] $ for interpolating minima evolves over time as
\begin{equation}
    \vectorization{\CovMat_{t+1}} = \CovEvo \ \vectorization{\CovMat_{t}},
\end{equation}
where $ \CovEvo $ is given in \eqref{eq:Covariance evolution matrix}. Since $ \CovMat_t$ is PSD by definition, we only care about the effect of $\CovEvo$ on vectorizations of PSD matrices. Therefore, we have that $ \{ \CovMat_{t} \} $ are bounded if and only if (see proof in \citep{ma2021on})
\begin{equation}\label{eq:impicit SGD stability condition restated 3}
    \max_{\CovMat \in \PSDset{d}} \frac{  \norm{\CovEvo(\eta, B) \; \vectorization{\CovMat}}}{\normF{\CovMat}} \leq 1.
\end{equation}
To obtain the stability threshold of SGD in the mean square sense we first rearrange the terms in $ \CovEvo $ as (see \eqref{eq:Covariance evolution matrix alternative})
\begin{equation}
    \CovEvo(\eta, B) = (1-p) \times (\Identity-\eta\HH)\otimes(\Identity-\eta\HH) + p \times \frac{1}{n} \sum_{i = 1}^n (\Identity-\eta\HH_i)\otimes(\Identity-\eta\HH_i).
\end{equation}
Here we explicitly see that $ \CovEvo $ can be written as a sum of Kronecker products, where each product is of a symmetric matrix with itself, as required by Thm.~\ref{thm:Symmetric Kronecker systems}. Applying this theorem, we have that the spectral radius of $ \CovEvo $ equals its top eigenvalue, and the corresponding top eigenvector is a vectorization of a PSD matrix. Note that since $ \CovEvo $ is symmetric, its spectral radius $ \rho(\CovEvo) $ is given by the \emph{unconstrained} optimization problem
\begin{equation}
    \rho(\CovEvo) = \max_{\CovMat \in \R^{d\times d}} \frac{  \norm{\CovEvo(\eta, B) \; \vectorization{\CovMat}}}{\normF{\CovMat}}.
\end{equation}
Theorem~\ref{thm:Symmetric Kronecker systems} tells us that the top eigenvector of $ \CovEvo $ maximizes this unconstrained problem, and more importantly, it always corresponds to a PSD matrix. Therefore, this top eigenvector also maximizes the objective while restricting to the subset of PSD matrices, which is given by the constraint in \eqref{eq:impicit SGD stability condition restated 3}. Thus, we have that the maximizer for the constrained optimization problem in \eqref{eq:impicit SGD stability condition restated 3} is, in fact, the top eigenvalue of $ \CovEvo $. Hence, the linear system is stable if and only if $ \lambda_{\max}(\CovEvo) \leq 1 $. Writing $ \CovEvo $ in terms of $\CC$ and $\DD$ gives (see \eqref{eq:Q in terms of C and D})
\begin{equation} \label{eq:RegatheringQ}
    \CovEvo = \Identity - 2\eta \CC + \eta^2 \DD.
\end{equation}
Because $ \CovEvo $ is symmetric, the condition $ \lambda_{\max}(\CovEvo) \leq 1 $ is equivalent to the requirement that $ \uu^\transpose  \CovEvo \uu \leq 1  $ for all $ \uu \in \Sb^{d^2-1} $. In App.~\ref{app:Null space} we show that $ \myNull{\CC} \subseteq \myNull{\DD} $. Therefore, if $ \uu \in \myNull{\CC} $ then also~$ \uu \in \myNull{\DD} $ and we get
\begin{equation}
    \uu^\transpose  \CovEvo \uu
    = 1 - 2\eta \uu^\transpose\CC\uu + \eta^2 \uu^\transpose\DD\uu = 1.
\end{equation}
Namely, directions in the null space of $ \CC $ do not impose any constraint on the learning rate, and thus can be ignored. Additionally, if $ \uu \in \myNull{\DD}  $ but $ \uu \notin \myNull{\CC} $, then
\begin{equation}
    \uu^\transpose  \CovEvo \uu
    = 1 - 2\eta \uu^\transpose\CC\uu + \eta^2 \uu^\transpose\DD\uu = 1 - 2\eta \uu^\transpose\CC\uu \leq 1,
\end{equation}
holds for all $ \eta \geq 0 $, because $ \CC $ is PSD (see App.~\ref{app:Positivity of C and D}). Now,
\begin{equation}\label{eq:proof start}
    \uu^\transpose  \CovEvo \uu
     = 1 - 2\eta \uu^\transpose\CC\uu + \eta^2 \uu^\transpose\DD\uu \leq 1
\end{equation}
holds for all $ \uu \notin \myNull{\DD} $ (which also results in $  \uu \notin \myNull{\CC} $), if and only if
\begin{equation}
    \forall \uu \notin \myNull{\DD} \qquad \eta  \uu^\transpose  \DD \uu \leq 2 \uu^\transpose  \CC \uu.
\end{equation}
Since $ \DD $ is PSD (see App.~\ref{app:Positivity of C and D}), and we assume that $ \uu \notin \myNull{\DD} $, we can divide both sides of this equation by $ 
\uu^\transpose  \DD \uu > 0 $ to get a condition on the learning rate as
\begin{equation}
    0 \leq \eta \leq 2 \inf_{\uu \in \Sb^{d^2-1} : \uu \notin \myNull{\DD} } \left\{ \frac{  \uu^\transpose  \CC \uu}{\uu^\transpose \DD \uu} \right\}.
\end{equation}
Therefore, the stability threshold $ \etaVar $ is given by
\begin{equation}\label{eq:StabilityConditionFirstMove}
    \etaVar
    = 2 \inf_{\uu \in \Sb^{d^2-1} : \uu \notin \myNull{\DD} } \left\{ \frac{  \uu^\transpose  \CC \uu}{\uu^\transpose \DD \uu} \right\}
    = 2\left(\sup_{\uu \in \Sb^{d^2-1} : \uu \notin \myNull{\DD}} \left\{ \frac{\uu^\transpose \DD \uu}{  \uu^\transpose  \CC \uu} \right\}\right)^{-1}.
\end{equation}
Note that the norm of $ \uu $ has no effect, and therefore we can remove the constraint $ \uu \in \Sb^{d^2-1} $.
Additionally, we can also relax the constraint $ \uu \notin \myNull{\DD} $ to $ \uu \notin \myNull{\CC} $, because the supremum in~\eqref{eq:StabilityConditionFirstMove} is over a nonnegative function (both $\CC$ and $\DD$ are PSD, see App.~\ref{app:Positivity of C and D}), and will not be affected by adding to the domain points at which the function vanishes. Since $ \myNull{\CC} \subseteq \myNull{\DD} $ we have that~$ \myNullOrtho{\DD} \subseteq \myNullOrtho{\CC} $ and therefore
\begin{equation}\label{eq:D and projection on C}
    \qquad \PP_{\myNullOrtho{\CC}} \DD = \DD \PP_{\myNullOrtho{\CC}} = \DD,
\end{equation}
where $ \PP_{\myNullOrtho{\CC}} $ is the projection matrix onto the orthogonal complement of the null space of $ \CC $.  Additionally, $ \CC $ is PSD (see App.~\ref{app:Positivity of C and D}), and therefore $ \CC^{\frac{1}{2}} $ exists and is also PSD, so that
\begin{equation}\label{eq:projection on C and its pseudo inverse}
    \PP_{\myNullOrtho{\CC}} = \left(\CC^{\frac{1}{2}}\right)^{\dagger} \CC^{\frac{1}{2}} = \CC^{\frac{1}{2}} \left(\CC^{\frac{1}{2}}\right)^{\dagger}.
\end{equation}
Therefore,
\begin{align}\label{eq:From frac to lambda max}
    \sup_{\uu \in \Sb^{d^2-1} : \uu \notin \myNull{\DD}} \left\{ \frac{\uu^\transpose \DD \uu}{  \uu^\transpose  \CC \uu} \right\}
    & = \sup_{\uu \notin \myNull{\CC}} \left\{ \frac{\uu^\transpose \DD \uu}{  \uu^\transpose  \CC \uu} \right\} \nonumber \\
    & = \sup_{\uu \notin \myNull{\CC}} \left\{ \frac{\uu^\transpose \PP_{\myNullOrtho{\CC}} \DD \PP_{\myNullOrtho{\CC}} \uu}{  \uu^\transpose  \CC \uu} \right\} \nonumber \\
    & = \sup_{\uu \notin \myNull{\CC}} \left\{ \frac{\uu^\transpose \CC^{\frac{1}{2}} \left(\CC^{\frac{1}{2}}\right)^{\dagger} \DD \left(\CC^{\frac{1}{2}}\right)^{\dagger} \CC^{\frac{1}{2}} \uu}{  \uu^\transpose  \CC^{\frac{1}{2}} \CC^{\frac{1}{2}} \uu} \right\} \nonumber \\
    & = \sup_{\uu \notin \myNull{\CC}} \left\{ \frac{\left(\CC^{\frac{1}{2}} \uu \right)^\transpose  \left(\CC^{\frac{1}{2}}\right)^{\dagger} \DD \left(\CC^{\frac{1}{2}}\right)^{\dagger} \left(\CC^{\frac{1}{2}} \uu \right)}{  \left(\CC^{\frac{1}{2}} \uu \right)^\transpose \left(\CC^{\frac{1}{2}} \uu \right)} \right\},
\end{align}
where in the second step we used \eqref{eq:D and projection on C}, and in the third step we used \eqref{eq:projection on C and its pseudo inverse}. By a simple change of variables $ \yy = \CC^{\frac{1}{2}}\uu \in \myNullOrtho{\CC} $ we get
\begin{align}\label{eq:From frac to lambda max 2}
    \max_{\yy \in \myNullOrtho{\CC}} \left\{ \frac{\yy^\transpose \left(\CC^{\frac{1}{2}}\right)^{\dagger} \DD \left(\CC^{\frac{1}{2}}\right)^{\dagger} \yy}{  \yy^\transpose \yy} \right\}
    & = \max_{\yy \in \R^{d^2}} \left\{ \frac{\yy^\transpose \left(\CC^{\frac{1}{2}}\right)^{\dagger} \DD \left(\CC^{\frac{1}{2}}\right)^{\dagger} \yy}{  \yy^\transpose \yy} \right\} \nonumber \\
    & = \lambda_{\max}\left( \left(\CC^{\frac{1}{2}}\right)^{\dagger} \DD \left(\CC^{\frac{1}{2}}\right)^{\dagger}\right),
\end{align}
where in the first step we used the fact that adding to $ \yy$ a component in $\myNull{\CC} $ will increase the denominator by $ \| \PP_{\myNull{\CC}} \yy \|^2 $ but will not affect the numerator. Namely, the optimum cannot be attained by $ \yy \notin \myNullOrtho{\CC}  $. Now, let  $ (\lambda_{i}, \ \yy_i) $ be an eigenpair of $  (\CC^{\frac{1}{2}})^{\dagger} \DD (\CC^{\frac{1}{2}})^{\dagger} $, then we have
\begin{equation}
    \lambda_{i} \yy_{i} = \left(\CC^{\frac{1}{2}}\right)^{\dagger} \DD \left(\CC^{\frac{1}{2}}\right)^{\dagger} \yy_{i}.
\end{equation}
Since we only care about nonzero eigenvalues, we can assume that $ \lambda_{i} \neq 0 $, and therefore $ \yy_{i} \notin \myNull{\CC} $. Multiplying by $ (\CC^{\frac{1}{2}})^{\dagger} $ from the left we get
\begin{equation}
    \lambda_{i} \left(\CC^{\frac{1}{2}}\right)^{\dagger} \yy_{i}
    = \left(\CC^{\frac{1}{2}}\right)^{\dagger} \left(\CC^{\frac{1}{2}}\right)^{\dagger} \DD \left(\CC^{\frac{1}{2}}\right)^{\dagger} \yy_{i}
    = \CC^{\dagger} \DD  \left(\CC^{\frac{1}{2}}\right)^{\dagger} \yy_{i}.
\end{equation}
Namely, $ (\CC^{\frac{1}{2}})^{\dagger} \yy_{i} \neq \zeroVec $ is an eigenvector of $ \CC^{\dagger} \DD $ with eigenvalue $ \lambda_i $.
Thus we have that if  $ \lambda_{i} \neq 0 $ is an eigenvalue of $ (\CC^{\frac{1}{2}})^{\dagger} \DD (\CC^{\frac{1}{2}})^{\dagger} $, then it is also an eigenvalue of $ \CC^{\dagger} \DD $. Similarly, we can prove vice versa, \ie that if $ \lambda_{i} \neq 0 $ is an eigenvalue of $ \CC^{\dagger} \DD $, then it is also an eigenvalue of $ (\CC^{\frac{1}{2}})^{\dagger} \DD (\CC^{\frac{1}{2}})^{\dagger} $. This means that $ (\CC^{\frac{1}{2}})^{\dagger} \DD (\CC^{\frac{1}{2}})^{\dagger} $ and $ \CC^{\dagger} \DD  $ have the same
eigenvalues. Therefore,
\begin{equation}\label{eq:eigenvalue equality}
    \lambda_{\max}\left( \left(\CC^{\frac{1}{2}}\right)^{\dagger} \DD \left(\CC^{\frac{1}{2}}\right)^{\dagger} \right) = 	\lambda_{\max}\left( \CC^{\dagger} \DD \right).
\end{equation}
Overall, we showed that the condition in \eqref{eq:impicit SGD stability condition} is equivalent to
\begin{equation}
    \eta \leq \frac{2}{ \lambda_{\max}\left( \CC^{\dagger} \DD  \right) }.
\end{equation}
This completes the proof for the second statement of the theorem.

\paragraph{Third statement.}
For the third statement of the theorem, note that from \eqref{eq:covariance evolution in orthogonal subspace} we have that $\forall t\geq 0$
\begin{equation}
    \vectorization{\CovMat_{t+1}^{\myPerp}}
    = \left(\PP_{\myNullOrtho{\HH}} \otimes \PP_{\myNullOrtho{\HH}} \right) \CovEvo \ \vectorization{\CovMat_{t}^{\myPerp}} .
\end{equation}
Namely, $ \vectorization{\CovMat_{t}^{\myPerp}}  =  \big[(\PP_{\myNullOrtho{\HH}} \otimes \PP_{\myNullOrtho{\HH}}) \CovEvo \big]^t \vectorization{\CovMat_{0}^{\myPerp}}  $, and thus
\begin{equation}\label{eq:dynamics upper bound interpolating minima}
    \norm{ \vectorization{\CovMat_{t}^{\myPerp}} } = \norm{  [(\PP_{\myNullOrtho{\HH}} \otimes \PP_{\myNullOrtho{\HH}}) \CovEvo ]^t  \vectorization{\CovMat_{0}^{\myPerp}}} \leq  \norm{(\PP_{\myNullOrtho{\HH}} \otimes \PP_{\myNullOrtho{\HH}}) \CovEvo }^t\norm{\vectorization{\CovMat_{0}^{\myPerp}}}.
\end{equation}
Here
\begin{align}
    \big(\PP_{\myNullOrtho{\HH}} \otimes \PP_{\myNullOrtho{\HH}}\big)\CovEvo
    & = \big(\PP_{\myNullOrtho{\HH}} \otimes \PP_{\myNullOrtho{\HH}}\big) \nonumber \\
    & \quad \left[(1-p) (\Identity-\eta\HH)\otimes(\Identity-\eta\HH) + p \times \frac{1}{n} \sum_{i = 1}^n (\Identity-\eta\HH_i)\otimes(\Identity-\eta\HH_i) \right] \nonumber \\
    & = (1-p) \big(\PP_{\myNullOrtho{\HH}}-\eta\HH\big)\otimes\big(\PP_{\myNullOrtho{\HH}}-\eta\HH\big) \nonumber \\
    & \quad + p \times \frac{1}{n} \sum_{i = 1}^n \big(\PP_{\myNullOrtho{\HH}}-\eta\HH_i\big)\otimes\big(\PP_{\myNullOrtho{\HH}}-\eta\HH_i\big),
\end{align}
where we used \eqref{eq:Covariance evolution matrix alternative} for the value of $ \CovEvo $. We see that $  \big(\PP_{\myNullOrtho{\HH}} \otimes \PP_{\myNullOrtho{\HH}}\big)\CovEvo $ is a sum of Kronecker products, where each product is a symmetric matrix multiplied by itself. This means that Thm.~\ref{thm:Symmetric Kronecker systems} applies to $ \big(\PP_{\myNullOrtho{\HH}} \otimes \PP_{\myNullOrtho{\HH}}\big) \CovEvo $, and thus we have that $ \big\| \big(\PP_{\myNullOrtho{\HH}} \otimes \PP_{\myNullOrtho{\HH}}\big) \CovEvo \big\| = \lambda_{\max}\big( (\PP_{\myNullOrtho{\HH}} \otimes \PP_{\myNullOrtho{\HH}}) \CovEvo \big) $. Moreover, it is easy to show that $ \PP_{\myNullOrtho{\HH}} \otimes \PP_{\myNullOrtho{\HH}} = \PP_{\myNullOrtho{\DD}} $, and $ \PP_{\myNullOrtho{\DD}} \CC = \CC \PP_{\myNullOrtho{\DD}} $. Combining this with~\eqref{eq:Q in terms of C and D}, we have
\begin{equation}
    (\PP_{\myNullOrtho{\HH}} \otimes \PP_{\myNullOrtho{\HH}})\CovEvo = \PP_{\myNullOrtho{\DD}} - 2\eta \CC \PP_{\myNullOrtho{\DD}} + \eta^2 \DD .
\end{equation}
Thus, for all $ \uu \in \myNull{\DD} $ we have
\begin{equation}\label{eq:null space of the projected evolution matrix of second moment}
    (\PP_{\myNullOrtho{\HH}} \otimes \PP_{\myNullOrtho{\HH}})\CovEvo \uu = \PP_{\myNullOrtho{\DD}} \uu  - 2\eta \CC \PP_{\myNullOrtho{\DD}} \uu  + \eta^2 \DD \uu = \zeroVec.
\end{equation}
Since the eigenvectors of symmetric matrices are orthogonal, and $ \myNull{\DD} $ is an eigenspace, we get that the top eigenvector of $ (\PP_{\myNullOrtho{\HH}} \otimes \PP_{\myNullOrtho{\HH}})\CovEvo $ should be in $ \myNullOrtho{\DD} $. Now, for $ \uu \in \myNullOrtho{\DD} \cap \Sb^{d^2-1} $
\begin{align}\label{eq:quadratic form over projection of Q}
    \uu^\transpose  (\PP_{\myNullOrtho{\HH}} \otimes \PP_{\myNullOrtho{\HH}})\CovEvo  \uu 
    & = \uu^\transpose\PP_{\myNullOrtho{\DD}} \uu  - 2\eta \uu^\transpose \CC \PP_{\myNullOrtho{\DD}} \uu  + \eta^2 \uu^\transpose \DD \uu \nonumber \\
    & = 1 - 2\eta \uu^\transpose\CC\uu + \eta^2 \uu^\transpose\DD\uu,
\end{align}
where in the second step we used the fact that $ \uu \in \myNullOrtho{\DD} $, and therefore $ \PP_{\myNullOrtho{\DD}} \uu = \uu $.
Additionally, note that
\begin{equation}
    \inf_{\uu \in \Sb^{d^2-1} : \uu \notin \myNull{\DD} } \left\{ \frac{  \uu^\transpose  \CC \uu}{\uu^\transpose \DD \uu} \right\}
    = \inf_{\uu \in \Sb^{d^2-1} \cap \myNullOrtho{\DD} } \left\{ \frac{  \uu^\transpose  \CC \uu}{\uu^\transpose \DD \uu} \right\}.
\end{equation}
Namely, having a component of $ \uu $ in $ \myNull{\DD} $ can only be non-optimal, since the denominator is invariant to vectors in $ 
\myNull{\DD} $, while the numerator can only increase ($\CC$ is PSD, see App.~\ref{app:Positivity of C and D}). Now, assuming $ \eta > 0 $ we have from the derivation of $ \etaVar $ in the second statement (see \eqref{eq:StabilityConditionFirstMove})
\begin{align}\label{eq:eta iff lambda}
    & \eta < \etaVar \nonumber \\
    \Leftrightarrow \quad  & \eta <  2 \inf_{\uu \in \Sb^{d^2-1} : \uu \notin \myNull{\DD} } \left\{ \frac{  \uu^\transpose  \CC \uu}{\uu^\transpose \DD \uu} \right\}  \nonumber \\
    \Leftrightarrow \quad  & \eta <  2 \inf_{\uu \in \Sb^{d^2-1}\cap \myNullOrtho{\DD} } \left\{ \frac{  \uu^\transpose  \CC \uu}{\uu^\transpose \DD \uu} \right\}  \nonumber \\
    \Leftrightarrow \quad & \eta < 2\frac{\uu^\transpose\CC\uu }{\uu^\transpose\DD\uu}  \qquad \forall \uu \in \Sb^{d^2-1} \cap \myNullOrtho{\DD}   \nonumber \\
    \Leftrightarrow \quad & \eta \uu^\transpose\DD\uu < 2\uu^\transpose\CC\uu \qquad \forall \uu \in \Sb^{d^2-1} \cap \myNullOrtho{\DD} \qquad (\DD \ \text{is PSD})   \nonumber \\
    \Leftrightarrow \quad & \eta^2 \uu^\transpose\DD\uu < 2\eta \uu^\transpose\CC\uu \qquad \forall \uu \in \Sb^{d^2-1} \cap \myNullOrtho{\DD} \qquad (\eta>0)   \nonumber \\
    \Leftrightarrow \quad & -2\eta \uu^\transpose\CC\uu+ \eta^2 \uu^\transpose\DD\uu <0  \qquad \forall \uu \in \Sb^{d^2-1} \cap \myNullOrtho{\DD}   \nonumber \\
    \Leftrightarrow \quad & 1-2\eta \uu^\transpose\CC\uu+ \eta^2 \uu^\transpose\DD\uu <1 \qquad \forall \uu \in \Sb^{d^2-1} \cap \myNullOrtho{\DD}   \nonumber \\ 
    \Leftrightarrow \quad & \uu^\transpose \left(\PP_{\myNullOrtho{\HH}} \otimes \PP_{\myNullOrtho{\HH}}\right)\CovEvo \uu < 1 \qquad \forall \uu \in \Sb^{d^2-1} \cap \myNullOrtho{\DD} \nonumber \\
    \Leftrightarrow \quad & \lambda_{\max}\left( \left(\PP_{\myNullOrtho{\HH}} \otimes \PP_{\myNullOrtho{\HH}}\right)\CovEvo \right) < 1
\end{align}
where in the fourth step we used the fact that $\DD$ is PSD (see App.~\ref{app:Positivity of C and D}), and in the penultimate step we used \eqref{eq:quadratic form over projection of Q}. Overall we have that $ 0 < \eta < \etaVar $ if and only if $ \lambda_{\max}\big( (\PP_{\myNullOrtho{\HH}} \otimes \PP_{\myNullOrtho{\HH}})\CovEvo \big) < 1 $ (we will use this fact in later sections).
Therefore,  when $ \eta < \etaVar$ then
\begin{equation} 
\big\|  \big(\PP_{\myNullOrtho{\HH}} \otimes \PP_{\myNullOrtho{\HH}}\big)\CovEvo \big\| = \lambda_{\max}\big( (\PP_{\myNullOrtho{\HH}} \otimes \PP_{\myNullOrtho{\HH}})\CovEvo \big) < 1 .
\end{equation}
Hence, from \eqref{eq:dynamics upper bound interpolating minima} we get
\begin{equation}
    \lim_{t \to \infty} \norm{ \vectorization{\CovMat_{t}^{\myPerp}} } \leq  \norm{\big(\PP_{\myNullOrtho{\HH}} \otimes \PP_{\myNullOrtho{\HH}}\big) \CovEvo }^t\norm{\vectorization{\CovMat_{0}^{\myPerp}}} = 0,
\end{equation}
which proves the statement.

\subsection{Proof of Theorem~\ref{thm:stability threshold expectation regular minima}}\label{app:stability threshold expectation regular minima proof}

\paragraph{First statement.} 
Let us start by proving the first statement. In \eqref{eq:Linearized dynamics over the null space general case} we showed that if the minimum is regular then
\begin{equation}
    \params^{\myPar}_{t+1} - \params^{*\myPar} = \params^{\myPar}_{t} - \params^{*\myPar} -\frac{\eta}{B} \sum_{i \in  \batch_t } \grad_i^{\myPar}.
\end{equation}
Let us compute the expected squared norm. We have
\begin{align}\label{eq:second moment for null space}
    \E\left[ \norm{\params^{\myPar}_{t+1} - \params^{*\myPar}}^2\right]
    & = \E\left[ \norm{\params^{\myPar}_{t} - \params^{*\myPar} -\frac{\eta}{B} \sum_{i \in  \batch_t } \grad_i^{\myPar} }^2\right] \nonumber \\
    & = \E\left[ \norm{\params^{\myPar}_{t} -\params^{*\myPar} }^2\right] + \E \left[\norm{\frac{\eta}{B} \sum_{i \in  \batch_t } \grad_i^{\myPar} }^2\right] - 2\E\left[ \left( \params^{\myPar}_{t} -\params^{*\myPar} \right)^\transpose \left(  \frac{\eta}{B} \sum_{i \in  \batch_t } \grad_i^{\myPar} \right) \right] \nonumber \\
    & = \E\left[ \norm{\params^{\myPar}_{t} -\params^{*\myPar} }^2\right] + \E \left[\norm{\frac{\eta}{B} \sum_{i \in  \batch_t } \grad_i^{\myPar} }^2\right] - 2\E\left[ \params^{\myPar}_{t} -\params^{*\myPar} \right]^\transpose \E \left[  \frac{\eta}{B} \sum_{i \in  \batch_t } \grad_i^{\myPar} \right] \nonumber \\
    & = \E\left[ \norm{\params^{\myPar}_{t} -\params^{*\myPar} }^2\right] + \E \left[\norm{\frac{\eta}{B} \sum_{i \in  \batch_t } \grad_i^{\myPar} }^2\right],
\end{align}
where in the third step we used the fact that $\params^{\myPar}_{t}$ is independent of $\batch_t$ and in the last we used the fact that
\begin{equation}
    \E\left[  \frac{\eta}{B} \sum_{i \in  \batch_t } \grad_i^{\myPar} \right]
    = \frac{\eta}{n} \sum_{i = 1}^n  \grad_i^{\myPar}
    = \PP_{\myNull{\HH}} \frac{\eta}{n} \sum_{i = 1}^n  \grad_i
    = \zeroVec.
\end{equation}
Calculating the right term in the last line of \eqref{eq:second moment for null space} using the definition of $ \vv_t $ (see \eqref{eq:transition matrix and vector v}) gives
\begin{align}
    \E \left[\norm{\frac{\eta}{B} \sum_{i \in  \batch_t } \grad_i^{\myPar} }^2\right]
    & = \E \left[\norm{ \PP_{\myNull{\HH}} \vv_t  }^2\right] \nonumber \\
    & = \Tr\left(\PP_{\myNull{\HH}}  \E \left[\vv_t  \vv_t^{\transpose} \right]\PP_{\myNull{\HH}} \right) \nonumber \\
    & = \eta^2\frac{n-B}{B(n-1)} \frac{1}{n}\sum_{i=1}^n \Tr\left(\PP_{\myNull{\HH}}  \grad_i\grad_i^\transpose \PP_{\myNull{\HH}} \right) \nonumber \\
    & = \eta^2 p \frac{1}{n}\sum_{i=1}^n  \norm{ \PP_{\myNull{\HH}}  \grad_i}^2\nonumber \\
    & = \eta^2 p \frac{1}{n}\sum_{i=1}^n  \norm{ \grad_i^{\myPar}}^2,
\end{align}
where in the third step we used \eqref{eq:covariance of v}. Unrolling \eqref{eq:second moment for null space} we have that 
\begin{equation}
    \E\left[ \norm{\params^{\myPar}_{t} -\params^{*\myPar} }^2\right]
    = \E\left[ \norm{\params^{\myPar}_{0} -\params^{*\myPar} }^2\right] + t \times \eta^2 p \frac{1}{n}\sum_{i=1}^n  \norm{ \grad_i^{\myPar}}^2.
\end{equation}
Thus, $\underset{t\to \infty}{\lim} \E[ \| \params_t^{\myPar} - \params^{*\myPar} \|^2 ] = \infty $ if and only if $ \ \sum_{i=1}^n \big\| \grad_i^{\myPar} \big\|^2 >0  $.

\paragraph{Second and third statements.}
Next, we turn to prove the second and third statements of the theorem. In App.~\ref{app:stability equivalence proof} we show the following.
\begin{lemma}\label{lemma:stability equivalence}
    Assume that $ \params^* $ is a twice differentiable regular minimum. Consider the linear dynamics of $ \{ \params_t \}  $ from Def.~\ref{def:Linearization}.
    \begin{enumerate}
        \item If $ \lambda_{\max}\big((\PP_{\myNullOrtho{\HH}} \otimes \PP_{\myNullOrtho{\HH}})\CovEvo\big)<1 $ then $ \underset{t\to \infty}{\mathrm{limsup}} \E[ \| \params_t^{\myPerp} -\params^{*\myPerp}  \|^2 ]  $ is finite.
        \item If $ \underset{t\to \infty}{\mathrm{limsup}} \E[ \| \params_t^{\myPerp} -\params^{*\myPerp}  \|^2 ]  $ is finite then $ \lambda_{\max}\big((\PP_{\myNullOrtho{\HH}} \otimes \PP_{\myNullOrtho{\HH}})\CovEvo\big) \leq 1 $. 
        \item Let $ \zz_{\max} $ denote the top eigenvector of $ (\PP_{\myNullOrtho{\HH}} \otimes \PP_{\myNullOrtho{\HH}})\CovEvo $, and assume that\\
        $ \zz_{\max}^{\transpose}\mathrm{vec}(\CovMat_{\grad}^{\myPerp}) \neq 0 $.
        If $ \underset{t\to \infty}{\mathrm{limsup}} \E[ \| \params_t^{\myPerp} -\params^{*\myPerp}  \|^2 ]  $ is finite then\\ $ \lambda_{\max}\big((\PP_{\myNullOrtho{\HH}} \otimes \PP_{\myNullOrtho{\HH}})\CovEvo\big) < 1 $. 
    \end{enumerate}
\end{lemma}
In \eqref{eq:eta iff lambda} we showed that $ \lambda_{\max}\big((\PP_{\myNullOrtho{\HH}} \otimes \PP_{\myNullOrtho{\HH}})\CovEvo\big) < 1 $ if and only if $ 0 < \eta < \etaVar $, which proves the second and third statements. Note that under the mild assumption that $ \zz_{\max}^{\transpose}\mathrm{vec}(\CovMat_{\grad}^{\myPerp}) \neq 0 $ we get that $ \underset{t\to \infty}{\mathrm{limsup}} \E[ \| \params_t^{\myPerp} -\params^{*\myPerp}  \|^2 ]  $ is finite if and only if $ 0 \leq \eta < \etaVar $.

\subsection{Proof of Lemma~\ref{lemma:stability equivalence}}\label{app:stability equivalence proof}
\paragraph{First statement.}
Here we assume $ \lambda_{\max}\big((\PP_{\myNullOrtho{\HH}} \otimes \PP_{\myNullOrtho{\HH}}) \CovEvo\big)<1 $, and show this implies that $\underset{t\to\infty}{\mathrm{limsup}}\E[\|\params_t^{\myPerp}~-~\params^{*\myPerp}~\|^2]$ is finite. The (projected) transition matrix that governs the dynamics of $ \CovMat_{t}^{\myPerp} $ and $  \MeanVec_{t}^{\myPerp} $ in \eqref{eq:projected joint dynamics} is given by
\begin{equation}
    \bXi = 
    \begin{pmatrix}
        \PP_{\myNullOrtho{\HH}} - \eta \HH  &  \zeroVec \\
        - \left(\PP_{\myNullOrtho{\HH}} \otimes \PP_{\myNullOrtho{\HH}} \right) (\EE \left[ \vv_t^{\myPerp} \otimes \TransMat_t \right] + \EE \left[ 	\TransMat_t \otimes \vv_t^{\myPerp} \right] )& \left(\PP_{\myNullOrtho{\HH}} \otimes \PP_{\myNullOrtho{\HH}} \right) \CovEvo
    \end{pmatrix}.
\end{equation}
Since this matrix is a block lower triangular matrix, its eigenvalues are
\begin{equation}
    \big\{\lambda_j (\bXi) \big\} = \bigg\{ \lambda_j\big(\PP_{\myNullOrtho{\HH}} - \eta \HH\big) \bigg\} \bigcup \bigg\{ \lambda_j\left(\big(\PP_{\myNullOrtho{\HH}} \otimes \PP_{\myNullOrtho{\HH}} \big) \CovEvo\right) \bigg\}.
\end{equation}
In Lemma~\ref{lemma:spectral radii} we show that if $ \rho\big((\PP_{\myNullOrtho{\HH}} \otimes \PP_{\myNullOrtho{\HH}}) \CovEvo\big)<1 $ then $ \rho(\PP_{\myNullOrtho{\HH}} - \eta \HH)<1 $ (see proof in App.~\ref{app:spectral radii}). Therefore, all the eigenvalues of $ \bXi $ are less than $ 1 $ in absolute value. Therefore, $ \| \mathrm{vec}(\CovMat_t^{\myPerp}) \|_2 = \| \CovMat_t^{\myPerp} \|_{\mathrm{F}} $ is bounded. Since $ \CovMat_t^{\myPerp} $ is PSD we have
\begin{equation}
    \normF{\CovMat_t^{\myPerp}} = \sqrt{\sum_{j = 1}^d \lambda^2_j\big(\CovMat_t^{\myPerp}\big) } \geq \frac{1}{\sqrt{d}} \sum_{j = 1}^d \lambda_j\big(\CovMat_t^{\myPerp}\big) = \frac{1}{\sqrt{d}} \Tr\big(\CovMat_t^{\myPerp}\big) = \frac{1}{\sqrt{d}} \E\left[ \left\| \params_t^{\myPerp} -\params^{*\myPerp}  \right\|^2 \right].
\end{equation}
Therefore, $ \E[ \| \params_t^{\myPerp} -\params^{*\myPerp}  \|^2 ] $ is bounded.

\paragraph{Second statement.}
Here we assume that $ \underset{t\to \infty}{\mathrm{limsup}} \E[ \| \params_t^{\myPerp} -\params^{*\myPerp}  \|^2 ]  $ is finite, then we show \\ $ \lambda_{\max}\big((\PP_{\myNullOrtho{\HH}}~\otimes~\PP_{\myNullOrtho{\HH}})~\CovEvo\big) \leq 1 $. The matrix $ \big(\PP_{\myNullOrtho{\HH}} \otimes \PP_{\myNullOrtho{\HH}}\big) \CovEvo $ can be written as
\begin{align}
    \left(\PP_{\myNullOrtho{\HH}} \otimes \PP_{\myNullOrtho{\HH}}\right) \CovEvo
    & = \left( \PP_{\myNullOrtho{\HH}} - \eta\HH \right) \otimes \left(\PP_{\myNullOrtho{\HH}} - \eta\HH \right) \nonumber \\
    & \quad + \eta^2 p \left( \frac{1}{n}\sum_{i = 1}^n \left( \HH_i - \HH\right)\otimes \left( \HH_i - \HH\right) \right),
\end{align}
where we used \eqref{eq:Covariance evolution matrix alternative 2} for the value of $ \CovEvo $.
This expression is a sum of Kronecker products, where each product is a symmetric matrix with itself. Therefore, according to Thm.~\ref{thm:Symmetric Kronecker systems}, we get 
\begin{equation}
    \lambda_{\max}\big((\PP_{\myNullOrtho{\HH}} \otimes~\PP_{\myNullOrtho{\HH}}) \CovEvo\big) = \rho\big((\PP_{\myNullOrtho{\HH}} \otimes~\PP_{\myNullOrtho{\HH}}) \CovEvo\big) 
\end{equation}
and $\ZZ_{\max} = \mathrm{vec}^{-1}(\zz_{\max}) $ is a PSD matrix, where $ \zz_{\max} $ is a normalized top eigenvector of $ \big(\PP_{\myNullOrtho{\HH}}~\otimes~\PP_{\myNullOrtho{\HH}}\big)~\CovEvo$. Now, set\footnote{
Since the eigenvectors of symmetric matrices are orthogonal, and $ \myNull{\DD} $ is an eigenspace, we get that the top eigenvector of $ (\PP_{\myNullOrtho{\HH}} \otimes \PP_{\myNullOrtho{\HH}})\CovEvo $ should be in $ \myNullOrtho{\DD} $, \ie $ \PP_{\myNullOrtho{\HH}}\ZZ_{\max} \PP_{\myNullOrtho{\HH}} = \ZZ_{\max} $. Therefore this initialization is possible.} $ \CovMat^{\myPerp}_0 = \ZZ_{\max} $ and $ \MeanVec_0 = \zeroVec $, then in this case $ \MeanVec^{\myPerp}_t =  (\PP_{\myNullOrtho{\HH}}-\eta\HH)^t \MeanVec_0^{\myPerp} = \zeroVec $ for all $ t > 0 $. Therefore, from~\eqref{eq:projected joint dynamics}
\begin{equation}
    \vectorization{\CovMat_{t+1}} = (\PP_{\myNullOrtho{\HH}}~\otimes~\PP_{\myNullOrtho{\HH}})\CovEvo\vectorization{\CovMat_{t}^{\myPerp}} + \vectorization{\CovMat_{\vv}^{\myPerp}}.
\end{equation}
Thus, taking an inner product w.r.t.\ $ \zz_{\max} $ on both sides we get
\begin{align}\label{eq:1}
    \zz_{\max}^\transpose \vectorization{\CovMat_{t+1}^{\myPerp}}
    & = \zz_{\max}^\transpose (\PP_{\myNullOrtho{\HH}}~\otimes~\PP_{\myNullOrtho{\HH}}) \CovEvo\vectorization{\CovMat_{t}^{\myPerp}} + \zz_{\max}^\transpose \vectorization{\CovMat_{\vv}^{\myPerp}}\nonumber\\
    & =  \lambda_{\max}\big((\PP_{\myNullOrtho{\HH}}~\otimes~\PP_{\myNullOrtho{\HH}}) \CovEvo\big) \zz_{\max}^\transpose \vectorization{\CovMat_{t}^{\myPerp}} + \Tr(\ZZ_{\max} \CovMat_{\vv}^{\myPerp}) \nonumber\\
    & \geq \lambda_{\max}\big((\PP_{\myNullOrtho{\HH}}~\otimes~\PP_{\myNullOrtho{\HH}}) \CovEvo\big) \zz_{\max}^\transpose \vectorization{\CovMat_{t}^{\myPerp}},
\end{align}
where in the last step we used the fact that $ \ZZ^{\frac{1}{2}}_{\max} \CovMat_{\vv}^{\myPerp} \ZZ^{\frac{1}{2}}_{\max} $ is a PSD matrix, and thus
\begin{equation}\label{eq:Z and Cov_z inner prod}
    \Tr\left(\ZZ_{\max} \CovMat_{\vv}^{\myPerp}\right) = \Tr\left( \ZZ^{\frac{1}{2}}_{\max} \CovMat_{\vv}^{\myPerp} \ZZ^{\frac{1}{2}}_{\max} \right) \geq 0.
\end{equation}
Therefore, using \eqref{eq:1} $t$ times we have
\begin{align}\label{eq: z_max lower bound}
    \zz_{\max}^\transpose \vectorization{\CovMat_{t}^{\myPerp}}
    & \geq \lambda^t_{\max}\big((\PP_{\myNullOrtho{\HH}}\otimes\PP_{\myNullOrtho{\HH}})\CovEvo\big) \zz_{\max}^\transpose \vectorization{\CovMat_{0}} \nonumber \\
    & = \lambda^t_{\max}\big((\PP_{\myNullOrtho{\HH}}\otimes\PP_{\myNullOrtho{\HH}})\CovEvo\big),
\end{align}
where in the last step we used $ \CovMat_0 = \ZZ_{\max} $ and $ \| \ZZ_{\max} \|_{\mathrm{F}} = 1 $. Additionally, for all $ t > 0 $
\begin{align}\label{eq: z_max upper bound}
    \zz_{\max}^\transpose \vectorization{\CovMat_{t}^{\myPerp}}
    & \leq \norm{\zz_{\max}}_2 \norm{\vectorization{\CovMat_{t}^{\myPerp}}}_2 \nonumber\\
    & = \normF{\CovMat_{t}^{\myPerp}} \nonumber\\
    & =  \sqrt{\sum_{j = 1}^d \lambda^2_j\big(\CovMat_t^{\myPerp}\big) } \nonumber\\
    & \leq \sum_{j = 1}^d \lambda_j\big(\CovMat_t^{\myPerp}\big) \nonumber\\
    & =\Tr\big(\CovMat_t^{\myPerp}\big) \nonumber\\
    & = \E\left[ \left\| \params_t^{\myPerp} -\params^{*\myPerp} \right\|^2 \right].
\end{align}
Overall, combining \eqref{eq: z_max lower bound} and \eqref{eq: z_max upper bound} results with
\begin{equation}
    \lambda^t_{\max}\big((\PP_{\myNullOrtho{\HH}}\otimes\PP_{\myNullOrtho{\HH}})\CovEvo\big) \leq \E\left[ \left\| \params_t^{\myPerp} -\params^{*\myPerp}   \right\|^2 \right] .
\end{equation}
Since $ \E[ \| \params_t^{\myPerp} -\params^{*\myPerp}   \|^2 ] $ is bounded then $\lambda_{\max}\big((\PP_{\myNullOrtho{\HH}}\otimes\PP_{\myNullOrtho{\HH}})\CovEvo\big) \leq 1 $.

\paragraph{Third statement.}
Furthermore, if $ \zz_{\max}^\transpose \vectorization{\CovMat_{\vv}^{\myPerp}} \neq 0 $ we get from \eqref{eq:Z and Cov_z inner prod}
\begin{equation}
    \zz_{\max}^\transpose \vectorization{\CovMat_{\vv}^{\myPerp}} > 0.
\end{equation}
Assume by contradiction that $ \lambda_{\max}\big((\PP_{\myNullOrtho{\HH}}\otimes\PP_{\myNullOrtho{\HH}})\CovEvo\big) = 1 $, then \eqref{eq:1} gives
\begin{equation}
    \zz_{\max}^\transpose \vectorization{\CovMat_{t+1}^{\myPerp}}
     = \zz_{\max}^\transpose \vectorization{\CovMat_{t}^{\myPerp}} + \zz_{\max}^\transpose \vectorization{\CovMat_{\vv}^{\myPerp}}.
\end{equation}
Unrolling this equation gives $ \zz_{\max}^\transpose \vectorization{\CovMat_{t}^{\myPerp}} = t \zz_{\max}^\transpose \vectorization{\CovMat_{\vv}^{\myPerp}} $. Then, by \eqref{eq: z_max upper bound} we get
\begin{equation}
    \E\left[ \left\| \params_t -\params^*  \right\|^2 \right] \geq \zz_{\max}^\transpose \vectorization{\CovMat_{t}^{\myPerp}}
    = t \zz_{\max}^\transpose \vectorization{\CovMat_{\vv}^{\myPerp}}.
\end{equation}
Since $ \zz_{\max}^\transpose \vectorization{\CovMat_{\vv}^{\myPerp}} >0 $, then $ \E[ \| \params_t -\params^*  \|^2 ] \to \infty $ and we have a contradiction. Therefore $ \lambda_{\max}\big((\PP_{\myNullOrtho{\HH}}\otimes\PP_{\myNullOrtho{\HH}})\CovEvo\big) < 1 $.

\subsection{Proof of Lemma~\ref{lemma:spectral radii}}\label{app:spectral radii}

\begin{lemma}\label{lemma:spectral radii}
    Let $ \HH \in \PSDset{d} $ and $ \CovEvo $ defined as in \eqref{eq:Covariance evolution matrix}.  If $ \rho((\PP_{\myNullOrtho{\HH}} \otimes \PP_{\myNullOrtho{\HH}}) \CovEvo)<1 $ then $ \rho(\PP_{\myNullOrtho{\HH}} - \eta\HH)<1 $.
\end{lemma}
The matrix $ (\PP_{\myNullOrtho{\HH}} \otimes \PP_{\myNullOrtho{\HH}}) \CovEvo $ is
\begin{align}\label{eq:auxilary1}
	\left(\PP_{\myNullOrtho{\HH}} \otimes \PP_{\myNullOrtho{\HH}}\right) \CovEvo
	& = \left( \PP_{\myNullOrtho{\HH}} - \eta\HH \right) \otimes \left(\PP_{\myNullOrtho{\HH}} - \eta\HH \right) \nonumber \\
	& \quad + \eta^2 p \left( \frac{1}{n}\sum_{i = 1}^n \left( \HH_i - \HH\right)\otimes \left( \HH_i - \HH\right) \right),
\end{align}
where we used \eqref{eq:Covariance evolution matrix alternative 2} for the value of $ \CovEvo $. Note that $ \PP_{\myNullOrtho{\HH}} - \eta\HH $ is a symmetric matrix, and thus each of its dominant eigenvectors $ \tilde{\vv} \in \Sb^{d-1} $ satisfies $ \rho(\PP_{\myNullOrtho{\HH}} - \eta\HH)  = | \tilde{\vv}^{\transpose} (\PP_{\myNullOrtho{\HH}} - \eta\HH)\tilde{\vv} | $. Additionally, $ \| \tilde{\vv} \otimes \tilde{\vv} \| = \normF{\tilde{\vv} \tilde{\vv}^\transpose } = 1 $, \ie $ \tilde{\vv} \otimes \tilde{\vv} \in \Sb^{d^2-1} $. Now, let the spectral radius $ \rho\big((\PP_{\myNullOrtho{\HH}} \otimes \PP_{\myNullOrtho{\HH}}) \CovEvo\big)<1 $ then
\begin{align}
    1 & > \rho\big((\PP_{\myNullOrtho{\HH}} \otimes \PP_{\myNullOrtho{\HH}}) \CovEvo\big) \nonumber\\
    & \geq \left| [ \tilde{\vv} \otimes \tilde{\vv} ]^\transpose \big(\PP_{\myNullOrtho{\HH}} \otimes \PP_{\myNullOrtho{\HH}}\big) \CovEvo [ \tilde{\vv} \otimes \tilde{\vv} ] \right|  \nonumber\\
    & = \left( \tilde{\vv}^{\transpose}  \left( \PP_{\myNullOrtho{\HH}} - \eta\HH \right) \tilde{\vv} \right)^2 + \eta^2 p  \frac{1}{n} \sum_{i = 1}^n  \left( \tilde{\vv}^{\transpose}  \left( \HH_i - \HH \right) \tilde{\vv} \right)^2 \nonumber\\
    & = \rho^2 (\PP_{\myNullOrtho{\HH}} - \eta\HH) + \eta^2 p  \frac{1}{n} \sum_{i = 1}^n  \left( \tilde{\vv}^{\transpose}  \left( \HH_i - \HH \right) \tilde{\vv} \right)^2 \nonumber\\
    & \geq \rho^2 (\PP_{\myNullOrtho{\HH}} - \eta\HH),
\end{align}
where in the third step we used \eqref{eq:auxilary1}.

\section{Proof of Theorem~\ref{thm:Symmetric Kronecker systems}}\label{app:Symmetric Kronecker systems}
\paragraph{First statement.} Let $ \{\myMat_i\} $ be symmetric matrices in $\R^{d\times d}$, and let
\begin{equation}
    \CovEvo = \sum_{i=1}^M \myMat_i\otimes \myMat_i .
\end{equation}
First, note that $ \CovEvo \in \R^{d^2 \times d^2} $ is symmetric.
\begin{equation}
    \CovEvo^\transpose
    = \left( \sum_{i=1}^M \myMat_i\otimes \myMat_i \right)^\transpose
    = \sum_{i=1}^M \left(\myMat_i\otimes \myMat_i\right)^\transpose
    = \sum_{i=1}^M \myMat_i^\transpose\otimes \myMat_i^\transpose
    = \sum_{i=1}^M \myMat_i \otimes \myMat_i
    = \CovEvo,
\end{equation}
where in the third step we used \eqref{eq:KroneckerProperty2} property of the Kronecker product, and in the fourth we used the fact that $\{ \myMat_i \}$ are symmetric. Then, by the spectral theorem, we have that all its eigenvectors~$ \{ \zz_j \} $ and eigenvalues~$ \{ \lambda_j \} $ are real. Given an eigenvector $ \zz \in \R^{d^2} $  of $ \CovEvo $, we can examine its matrix form $ \ZZ = \mathrm{vec}^{-1}(\zz) $, where $ \ZZ \in \R^{d \times d} $. Here we show that $ \CovEvo $ always has a set of $d^2$ eigenvectors that correspond only to either symmetric or skew-symmetric matrices $ \{ \ZZ_j \}$. Let $ (\lambda,\zz) $ be an eigenpair of $ \CovEvo $, \ie $ \lambda \zz = \CovEvo \zz $, and set $ \ZZ = \mathrm{vec}^{-1}(\zz) $. Then, 
\begin{align}
    \lambda \ZZ
    = \mathrm{vec}^{-1}(\lambda \zz) 
    = \mathrm{vec}^{-1}(\CovEvo \zz) 
    = \mathrm{vec}^{-1}\left(\sum_{i=1}^M \myMat_i\otimes \myMat_i \zz\right) 
    & = \sum_{i=1}^M \mathrm{vec}^{-1}\left( \myMat_i\otimes \myMat_i \zz\right) \nonumber \\ 
    & = \sum_{i=1}^M  \myMat_i \ZZ \myMat_i^\transpose ,
\end{align}
where in the penultimate step we used \eqref{eq:KroneckerProperty1} property of the Kronecker product. By taking a transpose on both ends of this equation we have
\begin{equation}
    \lambda \ZZ^\transpose 
    = \left( \sum_{i=1}^M \myMat_i \ZZ \myMat_i^\transpose \right)^\transpose
    = \sum_{i=1}^M  \left( \myMat_i \ZZ \myMat_i^\transpose \right)^\transpose
    = \sum_{i=1}^M \myMat_i \ZZ^\transpose \myMat_i^\transpose.
\end{equation}
Thus, we have that $ \mathrm{vec}(\ZZ^\transpose) $ is also an eigenvector of $ \CovEvo $. If~$ \lambda $ has multiplicity one, then it must be that $ \ZZ^\transpose = \pm \ZZ $, \ie symmetric or skew-symmetric matrix. If the multiplicity is greater than one and $ \ZZ^\transpose \neq \pm \ZZ $, then any linear combination of $ \ZZ $ and $ \ZZ^\transpose $ is also an eigenvector corresponding to~$ \lambda $. In particular,
\begin{equation}
    \hat{\ZZ}_1 = \frac{1}{2}\left(\ZZ+\ZZ^\transpose\right)
    \qquad \text{and} \qquad
    \hat{\ZZ}_2 = \frac{1}{2}\left(\ZZ-\ZZ^\transpose\right).
\end{equation}
By construction, $ \hat{\ZZ}_1 $ and $ \hat{\ZZ}_2 $ are symmetric and skew-symmetric eigenvectors, corresponding to~$ \lambda $. This procedure can be repeated while projecting the next eigenvectors of $ \lambda $ onto the orthogonal complement of the already found vectors, until we find all eigenvectors of~$\lambda$. In this way, we can find a set of eigenvectors comprised solely of vectors that correspond to symmetric or skew-symmetric.

\paragraph{Second and third statement.}
Using the first statement, we can consider a complete set of eigenvectors for $ \CovEvo $ that is comprised solely of vectors that correspond to either symmetric or skew-symmetric matrices. Our next step is to show that there is at least one dominant eigenvector of $ \CovEvo $ that corresponds to a symmetric matrix. Our final step will be to show that among the dominant eigenvectors that correspond to symmetric matrices, at least one corresponds to a PDS matrix.

To this end, we first bring the eigenvalues of $ \CovEvo $, denoted by $ \{ \ZZ_j \} $, to a normal (canonical) form. Here we assume without loss of generality that the eigenvectors are normalized, that is $ \normF{\ZZ_j} = 1  $ for all $j\in [d^2]$. For symmetric matrix $ \ZZ $, we have the spectral decomposition theorem, and thus $ \ZZ = \VV \bS \VV^\transpose  $, where $ \VV $ is an orthogonal matrix and $ \bS $ is diagonal. We can also bring a skew-symmetric matrix to a similar form of $ \ZZ = \VV \bS \VV^\transpose  $ with orthogonal $\VV$, where $ \bS $ is a block diagonal matrix, with $ \lfloor d/2 \rfloor $ blocks of size $ {2 \times 2} $. Specifically, these blocks are in the form of \citep{zumino1962normal}
\begin{equation}
    \begin{bmatrix}
        0 & s_{\ell} \\
        -s_{\ell} & 0        
    \end{bmatrix}.
\end{equation}
If the dimension $ d $ is odd, then the last row and column of $ \bS $ are the zero vectors. For numerical purposes, this normal (canonical) form can be computed using the real Schur decomposition.

For symmetric matrices, we define the vector~$ \bs_{\text{sym}} \in \R^d $ to be the diagonal of $ \bS $, and for skew-symmetric matrices we define $ \bs_{\text{skew}} \in \R^{{\lfloor d/2 \rfloor}} $ to be $ [s_1, s_2,\cdots, s_{\lfloor d/2 \rfloor}]^\transpose $. In App.~\ref{app:Quadratic form calculation} we show that for a symmetric matrix $ \ZZ $, its corresponding vector form $ \zz $ satisfies
\begin{equation}\label{eq:symmetric eigenvalues}
    \zz^\transpose \CovEvo \zz = \bs_{\text{sym}}^\transpose \sum_{i=1}^M \MM_{i}^{ \odot 2 } \bs_{\text{sym}},
\end{equation}
where $\MM_i = \VV^\transpose \myMat_i \VV $, the superscript $ ^{ \odot k } $ denotes the Hadamard power and $ \| \bs_{\text{sym}} \| = 1 $. For skew-symmetric matrices, we define a set of matrices $ \{ \TT_{i} \} $ in $ \R^{ \lfloor d/2 \rfloor \times \lfloor d/2 \rfloor } $, where
\begin{equation}\label{eq:smallDeterminants}
    \TT_{i \, [\ell,p]} = \MM_{i \,[2\ell-1,2p-1]}\MM_{i \,[2\ell,2p]}-\MM_{i \,[2\ell-1,2p]}\MM_{i \,[2\ell,2p-1]},
\end{equation}
for all $1\leq \ell, p \leq \lfloor d/2 \rfloor $. Namely, $ \TT_{i} $ is the determinant of each $ 2 \times 2 $ block of $ \MM_i $ without overlap. We show in App.~\ref{app:Quadratic form calculation} that for a skew-symmetric matrix $ \ZZ $, its corresponding vector form $ \zz $ satisfies
\begin{equation}\label{eq:skew symmetric quadratic form}
    \zz^\transpose \CovEvo \zz =2 \bs_{\text{skew}}^\transpose \sum_{i=1}^M \TT_i \bs_{\text{skew}},
\end{equation}
where $ \| \bs_{\text{skew}} \| = 1/\sqrt{2} $. Let us define the projection matrix $ \PP \in \R^{\lfloor d/2 \rfloor \times d} $ as
\begin{equation}
    \PP = \frac{1}{\sqrt{2}}
    \begin{bmatrix}
        1 & 1 & 0 & 0 & 0 & \cdots & 0 & 0 \\
        0 & 0 & 1 & 1 & 0 & \cdots & 0 & 0 \\
        \vdots & & & \ddots & & & & \vdots \\
        0 & 0 & 0 & 0 & 0 & \cdots & 1 & 1
    \end{bmatrix}.
\end{equation}
If $d$ is odd, then the last column of $\PP$ is the zero vector. This matrix is semi-orthogonal, \ie it satisfies $ \PP \PP^\transpose = \Identity $. Note that
\begin{equation}\label{eq:smallFrobenius}
    \left[\PP\MM_i^{ \odot 2 }\PP^\transpose \right]_{[\ell,p]} = \frac{1}{2} \left( \MM_{i \, [2\ell-1,2p-1]}^2 + \MM_{i \, [2\ell-1,2p]}^2 + \MM_{i \, [2\ell,2p-1]}^2 + \MM_{i \, [2\ell,2p]}^2 \right).
\end{equation}
Therefore, for all $ 1 \leq \ell, p \leq \lfloor d/2 \rfloor $ and $i\in [M]$ we have
\begin{align}\label{eq:T-M inequality}
    \left| \TT_{i \, [\ell,p]} \right|
    & = \left| \MM_{i \,[2\ell-1,2p-1]}\MM_{i \,[2\ell,2p]}-\MM_{i \,[2\ell-1,2p]}\MM_{i \,[2\ell,2p-1]} \right| \nonumber \\
    & \leq \left| \MM_{i \,[2\ell-1,2p-1]}\MM_{i \,[2\ell,2p]}\right|+\left| \MM_{i \,[2\ell-1,2p]}\MM_{i \,[2\ell,2p-1]} \right| \nonumber \\
    & \leq  \frac{1}{2} \left( \MM_{i \, [2\ell-1,2p-1]}^2 + \MM_{i \, [2\ell,2p]}^2\right) + \frac{1}{2} \left(\MM_{i \, [2\ell-1,2p]}^2 + \MM_{i \, [2\ell,2p-1]}^2 \right) \nonumber \\
    & = \left[\PP\MM_i^{ \odot 2 }\PP^\transpose \right]_{[\ell,p]},
\end{align}
where in the first step we used \eqref{eq:smallDeterminants}, in the second we used the triangle inequality, in the third we used $ |ab|\leq \frac{1}{2}(a^2+b^2 $) twice, and in the last step we used \eqref{eq:smallFrobenius}.

Note that any pair of orthogonal matrix $ \VV $ and a vector $ \bs_{\text{skew}} \in \R^{\lfloor d/2 \rfloor} $, such that $ \| \bs_{\text{skew}} \| = 1/\sqrt{2} $, define a skew-symmetric matrix $ \ZZ_{\text{skew}} $ with a vectorization $ \zz_{\text{skew}} $. Similarly, any pair of orthogonal matrix $ \VV $ and a vector $ \bs_{\text{sym}} \in \R^{d} $, such that $ \| \bs_{\text{sym}} \| = 1$, correspond to a symmetric matrix $ \ZZ_{\text{sym}} $ with a vectorization $ \zz_{\text{sym}} $. Here we will show that given $ \VV $, there exists $ \bs_{\text{sym}} \in \Sb^{d-1} $ s.t.
\begin{equation}
    \left| \zz_{\text{skew}}^{\transpose} \CovEvo \zz_{\text{skew}} \right| \leq \zz_{\text{sym}}^{\transpose} \CovEvo \zz_{\text{sym}}
\end{equation}
for any $ \bs_{\text{skew}} \in \R^{\lfloor d/2 \rfloor} $ for which $\|\bs_{\text{skew}} \|=1/\sqrt{2} $. To this end, we set $\rr~=~\sqrt{2}\bs_{\text{skew}}~\in~\Sb^{\lfloor d/2 \rfloor} $,~then
\begin{align}\label{eq:Cauchy interlacing step}
    \left|2 \bs_{\text{skew}}^\transpose \sum_{i=1}^M \TT_i \bs_{\text{skew}} \right|
    & = \left| \rr^\transpose \sum_{i=1}^M \TT_i \rr \right| \nonumber \\
    & \leq \sum_{\ell = 1}^d \sum_{p = 1}^d \sum_{i=1}^M  \left| \rr_{[\ell]} \right| \left| \TT_{i \, [\ell,p]} \right| \left| \rr_{[p]} \right| \nonumber \\
    & \leq \sum_{\ell = 1}^d \sum_{p = 1}^d \sum_{i=1}^M  \left| \rr_{[\ell]} \right|  \left[\PP\MM_{i}^{ \odot 2 }\PP^\transpose\right]_{[\ell,p]} \left| \rr_{[p]} \right|  \nonumber \\
    & = \sum_{\ell = 1}^d \sum_{p = 1}^d  \left| \rr_{[\ell]} \right| \bigg[ \sum_{i=1}^M \PP\MM_{i}^{ \odot 2 }\PP^\transpose \bigg]_{[\ell,p]} \left| \rr_{[p]} \right| \nonumber \\
    & \leq \lambda_{\max} \left(  \sum_{i=1}^M  \PP\MM_{i}^{ \odot 2 }\PP^\transpose \right) \nonumber \\
    & = \lambda_{\max} \left( \PP \bigg(\sum_{i=1}^M  \MM_{i}^{ \odot 2 } \bigg) \PP^\transpose \right) \nonumber \\
    & \leq \lambda_{\max} \left(\sum_{i=1}^M  \MM_{i}^{ \odot 2 } \right),
\end{align}
where in the second step we used the triangle inequality, in the third step we used \eqref{eq:T-M inequality}, in the fifth we used the fact that $ \| \rr \| = 1 $ and bound the quadratic form with the top eigenvalue (note that $ \{ \MM_{i} \} $ are symmetric and thus the top eigenvalue is real), and in the last step we used the Cauchy interlacing theorem (a.k.a.\ Poincaré separation theorem). Now, take~$ \bs_{\text{sym}} \in \Sb^{d-1} $ to be the top eigenvector (normalized) of $ \sum_{i=1}^M  \MM_{i}^{ \odot 2 }  $ (note that $ \{ \MM_{i} \} $ are symmetric and thus this top eigenvector is real), and pair it with the same basis $ \VV $ of~$ \ZZ_{\text{skew}} $ to get a symmetric matrix~$ \ZZ_{\text{sym}} $ that satisfies
\begin{equation}
    \left| \zz_{\text{skew}}^\transpose \CovEvo \zz_{\text{skew}}  \right|
    = \left|2 \bs_{\text{skew}}^\transpose \sum_{i=1}^M \TT_i \bs_{\text{skew}} \right|
    \leq \lambda_{\max}\left( \sum_{i=1}^M  \MM_{i}^{ \odot 2 } \right)
    = \bs_{\text{sym}}^\transpose  \sum_{i=1}^M  \MM_{i}^{ \odot 2 } \bs_{\text{sym}}
    = \zz_{\text{sym}}^\transpose \CovEvo \zz_{\text{sym}},
\end{equation}
where in the first step we used \eqref{eq:skew symmetric quadratic form}, in the second we used \eqref{eq:Cauchy interlacing step}, in the third we used the fact that $ \bs_{\text{sym}} \in \Sb^{d-1} $ is the top eigenvector of $ \sum_{i=1}^M  \MM_{i}^{ \odot 2 }  $, and in the last step we used \eqref{eq:symmetric eigenvalues}. Since this is true for any orthogonal $ \VV $, we get that at least one dominant eigenvector of $ \CovEvo $ corresponds to a symmetric matrix rather than a skew-symmetric one. Hence, from here onwards we can use the fact that there exists a dominant eigenvector of $ \CovEvo $ that corresponds to a symmetric matrix.

Let $ \tilde{\zz} $ be a dominant eigenvector of $\CovEvo$ which correspond to a symmetric matrix, and let $ \tilde{\ZZ} = \tilde{\VV} \tilde{\bS} \tilde{\VV}^\transpose $ be its spectral decomposition. Set 
\begin{equation}
    \bPsi = \sum_{i=1}^M \tilde{\MM}_{i}^{ \odot 2 }, \qquad \text{s.t.} \qquad \tilde{\MM}_{i} = \tilde{\VV}^\transpose \myMat_i \tilde{\VV}.
\end{equation}
Since $ \CovEvo $ is symmetric, then by the spectral theorem we have that all its eigenvectors and eigenvalues are real, and they are given by the quadratic form using the corresponding eigenvectors. Thus,
\begin{align}
    \rho(\CovEvo)
    & = \left| \tilde{\zz}^\transpose \CovEvo \tilde{\zz} \right| \nonumber \\
    & = \left| \tilde{\bs}_{\text{sym}}^\transpose \sum_{i=1}^M \tilde{\MM}_{i}^{ \odot 2 } \tilde{\bs}_{\text{sym}} \right| \nonumber \\
    & = \left| \tilde{\bs}_{\text{sym}}^\transpose \bPsi \tilde{\bs}_{\text{sym}} \right| \nonumber \\
    & = \left| \sum_{\ell=1}^d \sum_{p=1}^d   \tilde{\bs}_{\text{sym}\,[\ell]} \bPsi_{[\ell,p]} \tilde{\bs}_{\text{sym}\,[p]} \right| \nonumber \\
    & \leq \sum_{\ell=1}^d \sum_{p=1}^d   \left| \tilde{\bs}_{\text{sym}\,[\ell]} \right| \bPsi_{[\ell,p]} \left| \tilde{\bs}_{\text{sym}\,[p]} \right| \nonumber \\
    & = \left[\tilde{\bs}_{\text{sym}}^{\text{abs}} \right] ^\transpose \bPsi \left[\tilde{\bs}_{\text{sym}}^{\text{abs}} \right] \nonumber \\
    & = \left[\tilde{\zz}^{\text{abs}} \right]^\transpose \CovEvo \left[\tilde{\zz}^{\text{abs}} \right],
\end{align}
where $\tilde{\bs}_{\text{sym}}^{\text{abs}}$ is the element-wise absolute value of~$ \tilde{\bs}_{\text{sym}} $, and $ \tilde{\zz}^{\text{abs}} = \mathrm{vec}( \tilde{\VV} \tilde{\bS}^{\text{abs}} \tilde{\VV}^\transpose ) $. Namely, the vector $ \tilde{\zz}^{\text{abs}}$ that corresponds to the matrix built from the element-wise absolute value of the spectrum of $ \tilde{\ZZ} $ yields a greater or equal result than the spectral radius of  $ \CovEvo $, while still having a unit Euclidean norm. Thus, either $\tilde{\bs}_{\text{sym}}^{\text{abs}} =  \tilde{\bs}_{\text{sym}} $, in which case $ \tilde{\ZZ} $ is PSD, or $\tilde{\bs}_{\text{sym}}^{\text{abs}} \neq  \tilde{\bs}_{\text{sym}} $, and then both $ \tilde{\zz} $ and $ \tilde{\zz}^{\text{abs}} $ are dominant eigenvectors (or else we get a contradiction).
Note that $ \mathrm{vec}^{-1}(\tilde{\zz}^{\text{abs}}) $ is in fact a PSD matrix. Therefore, there is always a dominant eigenvector for $ \CovEvo $ which corresponds to a PSD matrix. Additionally, since $ \max_j | \lambda_j(\CovEvo) | = \rho(\CovEvo) = [\tilde{\zz}^{\text{abs}}]^\transpose \CovEvo [\tilde{\zz}^{\text{abs}}] $, then $ [\tilde{\zz}^{\text{abs}}] $ is also a top eigenvector which corresponds to $ \lambda_{\max}(\CovEvo) $, \ie $ \rho(\CovEvo) = \lambda_{\max}(\CovEvo) $ ($\lambda_{\max}(\CovEvo)$ is a dominant eigenvalue).

\subsection{Quadratic form calculation for symmetric and skew-symmetric matrices}\label{app:Quadratic form calculation}
Let $ \zz = \mathrm{vec}(\ZZ) $, and assume that $ \ZZ = \VV \bS \VV^\transpose $ where $ \VV $ is orthogonal matrix. Then
\begin{align}
    \zz^\transpose \CovEvo \zz 
    & = \left[\vectorization{\VV \bS \VV^\transpose} \right]^\transpose \CovEvo \; \vectorization{\VV \bS \VV^\transpose} \nonumber \\
    & = \left[ \left(\VV \otimes \VV \right) \vectorization{\bS} \right]^\transpose \CovEvo \left(\VV \otimes \VV \right) \vectorization{\bS} \nonumber \\
    & = \left[ \vectorization{\bS} \right]^\transpose \left(\VV^\transpose \otimes \VV^\transpose \right)   \sum_{i=1}^M \myMat_i\otimes \myMat_i \Big(\VV \otimes \VV \Big) \vectorization{\bS} \nonumber \\
    & = \smash{\left[ \vectorization{\bS} \right]^\transpose   \sum_{i=1}^M \left( \VV^\transpose \myMat_i \VV \right) \otimes \left( \VV^\transpose \myMat_i \VV \right) \vectorization{\bS} }\nonumber \\
    & = \sum_{i=1}^M \left[ \vectorization{\bS} \right]^\transpose \MM_i \otimes \MM_i \; \vectorization{\bS}.
\end{align}
Writing the quadratic form explicitly for each $ i \in [M] $ we have
\begin{equation}
    \left[ \vectorization{\bS} \right]^\transpose \MM_i \otimes \MM_i \; \vectorization{\bS} 
    = \sum_{m = 1}^{d^2} \sum_{k = 1}^{d^2} \left[ \MM_i \otimes \MM_i \right]_{ [m,k]} \vectorization{\bS}_{[m]} \vectorization{\bS}_{[k]} .
\end{equation}
Set $ m = d(m_2-1)+m_1 $ and $ k = d(k_2-1)+k_1 $ where $ m_1,m_2,k_1,k_2 \in [d] $, then
\begin{equation}\label{eq:Kronecker indexing}
    \left[ \MM_i \otimes \MM_i \right]_{ [m,k]} = \left[ \MM_i \otimes \MM_i \right]_{ [d(m_2-1)+m_1,d(k_2-1)+k_1]} = \MM_{i \, [m_1,k_1]} \MM_{i \, [m_2,k_2]}.
\end{equation}
Moreover,
\begin{align}\label{eq:Vectorization indexing}
    & \left[ \vectorization{\bS} \right]_{[m]} = \left[ \vectorization{\bS} \right]_{[d(m_2-1)+m_1]} = \bS_{[m_1,m_2]}, \nonumber \\
    & \left[ \vectorization{\bS} \right]_{[k]} = \left[ \vectorization{\bS} \right]_{[d(k_2-1)+k_1]} = \bS_{[k_1,k_2]}.
\end{align}
Therefore,
\begin{align}\label{eq:quadratic form general case}
    & \left[ \vectorization{\bS} \right]^\transpose \MM_i \otimes \MM_i \; \vectorization{\bS} \nonumber \\
    & = \sum_{m = 1}^{d^2} \sum_{k = 1}^{d^2} \left[ \MM_i \otimes \MM_i \right]_{ [m,k]} \left[ \vectorization{\bS} \right]_{[m]} \left[ \vectorization{\bS} \right]_{[k]} \nonumber \\
    & = \sum_{m_2 = 1}^{d} \sum_{m_1 = 1}^{d} \sum_{k_2 = 1}^{d} \sum_{k_1 = 1}^{d} \left[ \MM_i \otimes \MM_i \right]_{ [d(m_2-1)+m_1,d(k_2-1)+k_1]} \left[ \vectorization{\bS} \right]_{[d(m_2-1)+m_1]} \left[ \vectorization{\bS} \right]_{[d(k_2-1)+k_1]} \nonumber \\ %%
    & = \sum_{m_2 = 1}^{d} \sum_{m_1 = 1}^{d} \sum_{k_2 = 1}^{d} \sum_{k_1 = 1}^{d} \MM_{i \, [m_1,k_1]} \MM_{i \, [m_2,k_2]} \bS_{[m_1,m_2]} \bS_{[k_1,k_2]},
\end{align}
where in the last step we used \eqref{eq:Kronecker indexing} and \eqref{eq:Vectorization indexing}.

\subsubsection{Symmetric eigenvectors}
Assume that $ \ZZ $ is symmetric, then $ \bS $ is a diagonal matrix. Therefore, we only need to consider the terms in the series above for which $ m_1=m_2 = \ell $ and $ k_1 = k_2 = p $. 
\begin{equation}
    \sum_{\ell = 1}^{d} \sum_{p = 1}^{d} \MM_{i \, [\ell,p]} \MM_{i \, [\ell,p]} \bS_{[\ell,\ell]} \bS_{[p,p]} 
    = \sum_{\ell = 1}^{d} \sum_{p = 1}^{d}  \MM^2_{i \, [\ell, p]}  \bs_{\text{sym}\, [\ell]} \bs_{\text{sym}\, [p]}
    = \bs_{\text{sym}}^\transpose  \MM_{i}^{ \odot 2 } \bs_{\text{sym}}.
\end{equation}
Overall,
\begin{equation}
    \zz^\transpose \CovEvo \zz = \bs_{\text{sym}}^\transpose \sum_{i=1}^M \MM_{i}^{ \odot 2 } \bs_{\text{sym}}.
\end{equation}

\subsubsection{Skew-symmetric eigenvectors}
Assume that $ \ZZ $ is skew-symmetric, then $ \bS $ is a block diagonal matrix, where each block is $ 2\times2 $ in the form of 
\begin{equation}
    \begin{bmatrix}
        0 & s_{\ell} \\
        -s_{\ell} & 0        
    \end{bmatrix}.
\end{equation}
If the dimension $ d $ is odd, then $ \bS $ has a row and column at the end filled with zeros. Here, the nonzero elements are located above and below the main diagonal of $\bS$ and they come in pairs. Specifically, the following relations hold.
\begin{equation}
    \bS_{[2\ell-1, 2\ell]} = s_{\ell} = \bs_{\text{skew}\, [\ell]}, \qquad \bS_{[2\ell, 2\ell-1]} = -s_{\ell} = -\bs_{\text{skew}\, [\ell]}.
\end{equation}
Therefore, in the skew-symmetric scenario, we have four different cases to consider in the last line of~\eqref{eq:quadratic form general case}. The four cases are as follows.

\paragraph{Case I: $  m_1 = 2\ell-1,\ m_2=2\ell , \ k_1=2p-1 ,\ k_2 =2p  $.}
\begin{equation}
    \sum_{\ell = 1}^{\lfloor d/2 \rfloor} 
    \sum_{p = 1}^{\lfloor d/2 \rfloor}     \MM_{i \, [2\ell-1, 2p-1]} \MM_{i \, [2\ell,2p]} \bS_{[2\ell-1, 2\ell]} \bS_{[2p-1,2p]}  =  \sum_{\ell = 1}^{\lfloor d/2 \rfloor} \sum_{p = 1}^{\lfloor d/2 \rfloor}  \MM_{i \, [2\ell-1, 2p-1]} \MM_{i \, [2\ell,2p]} \bs_{\text{skew}\, [\ell]} \bs_{\text{skew}\, [p]} .
\end{equation}

\paragraph{Case II: $  m_1 = 2\ell,\ m_2=2\ell-1 , \ k_1=2p-1 ,\ k_2 =2p  $.}
\begin{equation}
    \sum_{\ell = 1}^{\lfloor d/2 \rfloor} \sum_{p = 1}^{\lfloor d/2 \rfloor}  \MM_{i \, [2\ell, 2p-1]} \MM_{i \, [2\ell-1, 2p]} \bS_{[2\ell, 2\ell-1]} \bS_{[2p-1, 2p]} = -\sum_{\ell = 1}^{\lfloor d/2 \rfloor} \sum_{p = 1}^{\lfloor d/2 \rfloor}  \MM_{i \, [2\ell, 2p-1]} \MM_{i \, [2\ell-1, 2p]} \bs_{\text{skew}\, [\ell]} \bs_{\text{skew}\, [p]} .
\end{equation}

\paragraph{Case III: $  m_1 = 2\ell-1,\ m_2=2\ell , \ k_1=2p ,\ k_2 =2p-1  $.}
\begin{equation}
    \sum_{\ell = 1}^{\lfloor d/2 \rfloor} \sum_{p = 1}^{\lfloor d/2 \rfloor} \MM_{i \, [2\ell-1, 2p]} \MM_{i \, [ 2\ell, 2p-1]}  \bS_{[2\ell-1, 2\ell]} \bS_{[2p, 2p-1]} = -\sum_{\ell = 1}^{\lfloor d/2 \rfloor} \sum_{p = 1}^{\lfloor d/2 \rfloor} \MM_{i \, [2\ell-1, 2p]} \MM_{i \, [ 2\ell, 2p-1]} \bs_{\text{skew}\, [\ell]} \bs_{\text{skew}\, [p]} .
\end{equation}

\paragraph{Case IV: $  m_1 = 2\ell,\ m_2=2\ell-1 , \ k_1=2p,\ k_2 =2p-1  $.}
\begin{equation}
    \sum_{\ell = 1}^{\lfloor d/2 \rfloor} \sum_{p = 1}^{\lfloor d/2 \rfloor} \MM_{i \, [2\ell, 2p]} \MM_{i \, [2\ell-1, 2p-1]}  \bS_{[2\ell, 2\ell-1]} \bS_{[2p, 2p-1]} = \sum_{\ell = 1}^{\lfloor d/2 \rfloor} \sum_{p = 1}^{\lfloor d/2 \rfloor} \MM_{i \, [2\ell, 2p]} \MM_{i \, [2\ell-1, 2p-1]} \bs_{\text{skew}\, [\ell]} \bs_{\text{skew}\, [p]}.
\end{equation}

Summing over all these cases we get
\begin{align}
    \zz^\transpose \CovEvo \zz
    & = \sum_{i=1}^M  \sum_{\ell = 1}^{\lfloor d/2 \rfloor} \sum_{p = 1}^{\lfloor d/2 \rfloor} 2 \left( \MM_{i \,[2\ell-1,2p-1]}\MM_{i \,[2\ell,2p]}-\MM_{i \,[2\ell-1,2p]}\MM_{i \,[2\ell,2p-1]}\right)  \bs_{\text{skew}\, [\ell]} \bs_{\text{skew}\, [p]} \nonumber \\
    & = 2 \sum_{i=1}^M  \sum_{\ell = 1}^{\lfloor d/2 \rfloor} \sum_{p = 1}^{\lfloor d/2 \rfloor} \TT_{i \, [\ell, p]}  \bs_{\text{skew}\, [\ell]} \bs_{\text{skew}\, [p]} \nonumber \\
    & = 2 \bs_{\text{skew}}^\transpose  \sum_{i=1}^M \TT_i \bs_{\text{skew}}. 
\end{align}

\section{Proof of Proposition~\ref{thm:Monotonicity}}\label{app:MonotonicityProof}
Here we focus on interpolating minima for simplicity. A similar proof can be derived for regular minima. To begin with, note that (see \eqref{eq:eigenvalue equality})
\begin{equation}
    \lambda_{\max}\left( \CC^{\dagger}  \DD \right) = \lambda_{\max}\left( \left(\CC^{\frac{1}{2}}\right)^{\dagger} \DD \left(\CC^{\frac{1}{2}}\right)^{\dagger} \right).
\end{equation}
Additionally,
\begin{align}\label{eq:D and E}
    \DD 
    & = (1-p)\, \HH \otimes \HH + p\,\frac{1}{n}\sum_{i=1}^n\HH_i \otimes \HH_i \nonumber \\
    & = \HH \otimes \HH + p\left(\frac{1}{n}\sum_{i=1}^n\HH_i \otimes \HH_i -\HH \otimes \HH \right) \nonumber \\
    & = \HH \otimes \HH +p \left(\frac{1}{n}\sum_{i=1}^n(\HH_i-\HH) \otimes (\HH_i-\HH) \right) \nonumber \\
    & = \HH \otimes \HH +p \bE ,
\end{align}
where in the third step we used \eqref{eq:kronecker variance}, and at the last step $ \bE \triangleq \frac{1}{n}\sum_{i=1}^n(\HH_i-\HH) \otimes (\HH_i-\HH) $. Let $ \yy \in \Sb^{d^2-1} $ be the top eigenvector of $ (\CC^{\frac{1}{2}})^{\dagger}  \DD (\CC^{\frac{1}{2}})^{\dagger}  $, then since $ (\CC^{\frac{1}{2}})^{\dagger}  \DD (\CC^{\frac{1}{2}})^{\dagger}  $ is symmetric we have
\begin{align}
    \frac{\partial}{\partial p } \lambda_{\max}\left( \left(\CC^{\frac{1}{2}}\right)^{\dagger}  \DD \left(\CC^{\frac{1}{2}}\right)^{\dagger} \right)
    & = \yy^\transpose \left[ \left(\CC^{\frac{1}{2}}\right)^{\dagger} \left( \frac{\partial}{\partial p } \DD \right) \left(\CC^{\frac{1}{2}}\right)^{\dagger} \right] \yy \nonumber \\
    & = \yy^\transpose \left[ \left(\CC^{\frac{1}{2}}\right)^{\dagger} \bE \left(\CC^{\frac{1}{2}}\right)^{\dagger} \right] \yy.
\end{align}
In App.~\ref{sec:Eigenvector of CC^{dagger}^{1/2} DD CC^{dagger}^{1/2}} we show that $ \yy $ has the form of $ \yy = \CC^{\frac{1}{2}} \uu $ such that $ \mathrm{vec}^{-1}(\uu) \in \PSDset{d} $. Plugging this into the equation above we get
\begin{align}
    \yy^\transpose \left[ \left(\CC^{\frac{1}{2}}\right)^{\dagger} \bE \left(\CC^{\frac{1}{2}}\right)^{\dagger} \right] \yy
    & = \uu^\transpose \CC^{\frac{1}{2}} \left(\CC^{\frac{1}{2}}\right)^{\dagger} \bE \left(\CC^{\frac{1}{2}}\right)^{\dagger} \CC^{\frac{1}{2}} \uu \nonumber \\
    & = \uu^\transpose \PP_{\myNullOrtho{\CC}} \bE \PP_{\myNullOrtho{\CC}} \uu \nonumber \\
    & = \uu^\transpose \bE \uu,
\end{align}
where in the first step we used the fact that $ \CC $ is symmetric. Additionally, in the second step, we used the fact that $ \CC^{\frac{1}{2}} (\CC^{\frac{1}{2}})^{\dagger} $ and $ (\CC^{\frac{1}{2}})^{\dagger} \CC^{\frac{1}{2}} $ are projection matrices onto the column space of $ \CC $. Since the null space of $ \bE $ contains the null space of $ \CC $, we have that these projections can be removed (see App.~\ref{app:Null space}). Note that $ \mathrm{vec}^{-1}(\uu) $ is PSD, and let $ \VV \bS \VV^\transpose $ be its spectral decomposition, then in App.~\ref{app:Quadratic form calculation} we show that in this case
\begin{equation}\label{eq:Mono final step}
    \uu^\transpose \bE \uu = \bs^\transpose \sum_{i=1}^n \MM_{i}^{ \odot 2 } \bs,
\end{equation}
where $\MM_i = \VV^\transpose (\HH_i-\HH) \VV $ and $ \bs $ is a vector containing the eigenvalues of $ \mathrm{vec}^{-1}(\uu) $. Since $ \mathrm{vec}^{-1}(\uu) $ is PSD, we have the right-hand side of \eqref{eq:Mono final step} is a sum over nonnegative terms. Namely,
\begin{equation}
    \frac{\partial}{\partial p } \lambda_{\max}\left( \left(\CC^{\frac{1}{2}}\right)^{\dagger}  \DD \left(\CC^{\frac{1}{2}}\right)^{\dagger} \right)
    = \uu^\transpose \bE \uu = \bs^\transpose \sum_{i=1}^n \MM_{i}^{ \odot 2 } \bs \geq 0.
\end{equation}
Therefore, $ \lambda_{\max}\left( \CC^{\dagger} \DD  \right) $ is monotonically non-decreasing in $ p $, which means that $ \etaVar = 2/\lambda_{\max}\left( \CC^{\dagger} \DD  \right)  $ is monotonically non-decreasing with $B$.

\subsection{Top eigenvector of $ (\CC^{\frac{1}{2}})^{\dagger}  \DD (\CC^{\frac{1}{2}})^{\dagger} $}\label{sec:Eigenvector of CC^{dagger}^{1/2} DD CC^{dagger}^{1/2}}
Using the stability condition of \citet{ma2021on}, we have that $ \{ \E[ \| \params_t -\params  \|^2 ] \} $ is bounded if and only if (see proof in \citep{ma2021on})
\begin{equation}\label{eq:impicit SGD stability condition restated 2}
    \max_{\CovMat \in \PSDset{d}} \frac{  \norm{\CovEvo(\eta, B) \; \vectorization{\CovMat}}}{\normF{\CovMat}} \leq 1.
\end{equation}
Let us repeat the same steps from the proof of Thm.~\ref{thm:stability threshold expectation} in App.~\ref{app:stability threshold expectation proof} but \emph{without} relaxing the constraint of PSD matrices. Specifically, repeating the steps in equations in \eqref{eq:proof start}-\eqref{eq:StabilityConditionFirstMove} without invoking Thm.~\ref{thm:Symmetric Kronecker systems} gives us that
\begin{equation}
	\uu^\transpose  \CovEvo \uu
	 = 1 - 2\eta \uu^\transpose\CC\uu + \eta^2 \uu^\transpose\DD\uu \leq 1
\end{equation}
holds for any $ \uu \in \Sb^{d^2-1} $  such that $ \mathrm{vec}^{-1}(\uu) \in \PSDset{d}$ and $ \uu \notin \myNull{\DD} $, if and only if 
\begin{equation} \label{eq:etaVar alternative}
    \eta \leq \frac{2}{\lambda^*} = \etaVar,
\end{equation}
where
\begin{equation}
	\lambda^* = \sup_{\uu \in \Sb^{d^2-1}} \left\{ \frac{\uu^\transpose \DD \uu}{  \uu^\transpose  \CC \uu} \right\}\quad \text{s.t.} \quad  \mathrm{vec}^{-1}(\uu) \in \PSDset{d} \ \text{and} \  \uu \notin \myNull{\DD} .
\end{equation}
(Note that the case $ \uu \in \myNull{\DD} $ do not contribute any conditions on the learning rate, and therefore can be ignored - see App.~\ref{app:stability threshold expectation proof}). Using change of variables (see \eqref{eq:From frac to lambda max} and \eqref{eq:From frac to lambda max 2}) results with
\begin{equation}
    \lambda^* = \max_{ \yy \in \Sb^{d^2-1}} \yy^\transpose \left(\CC^{\frac{1}{2}}\right)^{\dagger}  \DD \left(\CC^{\frac{1}{2}}\right)^{\dagger} \yy  \quad \text{s.t.} \quad \yy = (\CC^{\frac{1}{2}} \uu) \ \text{and} \  \mathrm{vec}^{-1}(\uu) \in \PSDset{d}.
\end{equation}
Since the alternative form of $ \etaVar $ in \eqref{eq:etaVar alternative} has to be equal to the definition in \eqref{eq:etaVar definition} (or else we will get a contradiction), we get
\begin{equation}
    \lambda_{\max}\left( \left(\CC^{\frac{1}{2}}\right)^{\dagger}  \DD \left(\CC^{\frac{1}{2}}\right)^{\dagger} \right) = \lambda^{*}.
\end{equation}
Namely, the top eigenvector $ \yy $ of $ (\CC^{\frac{1}{2}})^{\dagger}  \DD (\CC^{\frac{1}{2}})^{\dagger} $ has the form of $ \yy = \CC^{\frac{1}{2}} \uu $ such that $ \mathrm{vec}^{-1}(\uu) \in \PSDset{d} $.

\section{Proof of Proposition~\ref{prop:equivalent algorithm}}\label{app:equivalent algorithm proof}
Here we focus on interpolating minima for simplicity. A similar proof can be derived for regular minima. Let $ \{\beta_t\} $ and $ \{ \kappa_t \} $ be i.i.d.\ random variables such that $ \beta_t \thicksim \mathrm{Bernoulli}(p) $ and $ \kappa_t \thicksim \mathcal{U}(\{ 1,...,n \}) $, then
\begin{equation}
    \batch_t =
    \begin{cases}
        \kappa_t & \text{if } \beta_t=1, \\
        \{ 1,...,n \} &  \text{otherwise} .
    \end{cases}
\end{equation}
Let us consider the following stochastic loss function
\begin{equation}
    \stochasticLoss_t(\params) = \frac{1}{| \batch_t |}\sum_{i \in \batch_t } \ell_i(\params),
\end{equation}
where $ | \batch_t | $ denotes the size of $ \batch_t $ (either $1$ or $n$), and define the following notation
\begin{equation}
    \mathcal{A}_t =  \Identity - \frac{\eta}{| \batch_t |} \sum_{i \in  \batch_t } \HH_i.
\end{equation}
First, for interpolating minima we have
\begin{equation}\label{eq:dynamics equivalent algo}
    \params_{t+1} - \params^* = \left( \Identity - \frac{\eta}{| \batch_t |} \sum_{i \in  \batch_t } \HH_i \right) (\params_t - \params^* ) = \mathcal{A}_t(\params_t - \params^* ).
\end{equation}
Thus,
\begin{align}
    \CovMat_{t+1}
    & = \EE\left[ \left(\params_{t+1} - \params^*\right)\left(\params_{t+1} - \params^*\right)^\transpose \right] \nonumber \\
    & =\EE\left[ \mathcal{A}_t \left(\params_{t} - \params^*\right)\left(\params_{t} - \params^*\right)^\transpose \mathcal{A}_t \right] \nonumber \\
    & = \EE\left[ \mathcal{A}_t \E\left[ \left(\params_{t} - \params^*\right)\left(\params_{t} - \params^*\right)^\transpose \middle| \batch_t  \right] \mathcal{A}_t \right] \nonumber \\
    & = \EE\left[ \mathcal{A}_t \E\left[ \left(\params_{t} - \params^*\right)\left(\params_{t} - \params^*\right)^\transpose\right] \mathcal{A}_t \right] \nonumber \\
    & = \EE\left[ \mathcal{A}_t \CovMat_t \mathcal{A}_t \right],
\end{align}
where in the second step we used \eqref{eq:dynamics equivalent algo}, in the third we used the law of total expectation, and in the fourth we used the fact that $ 
\params_t $ is statistically independent of $ \batch_t $. Using vectorization we get
\begin{align}\label{eq:covariance evolution equivalent algo}
    \vectorization{\CovMat_{t+1}}
    & = \EE\left[ \mathcal{A}_t \otimes \mathcal{A}_t \right] \vectorization{\CovMat_t}.
\end{align}
In \eqref{eq:Covariance evolution matrix alternative} we show that for any given (fixed) batch size,
\begin{equation}\label{eq:Covariance evolution matrix alternative restated}
\CovEvo(B, \eta)  =  \left(1-\frac{n-B}{B(n-1)}\right) (\Identity-\eta\HH)\otimes(\Identity-\eta\HH) + \frac{n-B}{B(n-1)} \frac{1}{n} \sum_{i = 1}^n (\Identity-\eta\HH_i)\otimes(\Identity-\eta\HH_i).
\end{equation}
Specifically, 
\begin{align}
    \CovEvo(\eta, B=1) = \frac{1}{n} \sum_{i = 1}^n (\Identity-\eta\HH_i)\otimes(\Identity-\eta\HH_i) , \qquad 
    \CovEvo(\eta, B=n) = (\Identity-\eta\HH)\otimes(\Identity-\eta\HH).
\end{align}
Using this result, let us compute the term in \eqref{eq:covariance evolution equivalent algo}.
\begin{align}
    \EE\left[\mathcal{A}_t  \otimes \mathcal{A}_t \right]
    & = \prob\left( \beta_t = 0\right)\EE\left[\mathcal{A}_t  \otimes \mathcal{A}_t \middle| \beta_t = 0 \right]+\prob \left( \beta_t = 1\right) \EE\left[\mathcal{A}_t  \otimes \mathcal{A}_t \middle| \beta_t = 1 \right] \nonumber \\
    & = (1-p) \CovEvo(\eta, B=n)+ p \CovEvo(\eta, B=1) \nonumber \\
    & = (1-p) \times (\Identity-\eta\HH)\otimes(\Identity-\eta\HH) + p \times \frac{1}{n} \sum_{i = 1}^n (\Identity-\eta\HH_i)\otimes(\Identity-\eta\HH_i).
\end{align}
This is the same matrix that we had for mini-batch SGD with batch size $B$ such that $ p=\frac{n-B}{B(n-1)} $ (see~\eqref{eq:Covariance evolution matrix alternative restated}). Namely, the covariance matrix of the parameters for the mixed process evolves in the same way as mini-batch SGD, with a corresponding batch size. Therefore, the stability threshold of both algorithms is the same.

\section{Proof of Proposition~\ref{prop:stability gap}}\label{app:stability gap proof}

First, since $ \etaVar $ is monotonically non-decreasing with $B$ (Thm.~\ref{thm:Monotonicity}), and for $ B=n $ we have $ \etaVar = \etaMean $ (App.~\ref{app:Recovering GD's stability condition}), we get that $ \etaVar \leq \etaMean $ for all values of $ B $. Now, set $ \varepsilon \in (0,1) $, then $(1-\varepsilon) \etaMean \leq \etaVar $ holds whenever
\begin{align}\label{eq:Gap excact condition}
    & (1-\varepsilon)\frac{2}{\lambda_{\max}(\HH)} \leq \frac{2}{\lambda_{\max}(\CC^{\dagger}\DD)}\nonumber\\
    \Leftrightarrow \quad & (1-\varepsilon)\lambda_{\max}(\CC^{\dagger}\DD) \leq \lambda_{\max}(\HH).
\end{align}
Note that $ \DD = \HH \otimes \HH +p \bE $ (see~\eqref{eq:D and E}), then
\begin{align}\label{eq:Gap first step}
    \lambda_{\max}\left(\CC^{\dagger}\DD\right)
    & = \lambda_{\max} \left(\CC^{\dagger}\HH \otimes \HH + p \CC^{\dagger} \bE \right) \nonumber \\
    & \leq \lambda_{\max} \left(\CC^{\dagger}\HH \otimes \HH \right) + p \lambda_{\max} \left( \CC^{\dagger} \bE \right) \nonumber \\
    & = \lambda_{\max}(\HH) + p \lambda_{\max} \left( \CC^{\dagger} \bE \right),
\end{align}
where we used $ \lambda_{\max} \left(\CC^{\dagger}\HH \otimes \HH \right) = \lambda_{\max}(\HH) $ (see App.~\ref{app:Recovering GD's stability condition}). Using the fact that $  \lambda_{\max} \left( \CC^{\dagger} \bE \right) $ is nonnegative (see App.~\ref{app:lambda_max(C^daggerE) is non-negative}) and that
\begin{equation}
    p = \frac{1}{B} \ \frac{n-B}{n-1} \leq \frac{1}{B},
\end{equation}
we can further bound \eqref{eq:Gap first step} from above by
\begin{align}\label{eq:Gap second step}
    \lambda_{\max}\left(\CC^{\dagger}\DD\right)
    & \leq \lambda_{\max}(\HH) + p \lambda_{\max} \left( \CC^{\dagger} \bE \right) \nonumber \\
    & \leq \lambda_{\max}(\HH) + \frac{1}{B} \lambda_{\max} \left( \CC^{\dagger} \bE \right).
\end{align}
Therefore, if
\begin{equation}\label{eq:Gap simplified}
    (1-\varepsilon)\bigg(  \lambda_{\max}(\HH) + \frac{1}{B} \lambda_{\max} \left( \CC^{\dagger} \bE \right) \bigg) \leq \lambda_{\max}(\HH)
\end{equation}
then \eqref{eq:Gap excact condition} holds. It is easy to show that \eqref{eq:Gap simplified} is equivalent to 
\begin{equation}
    B \geq \frac{1-\varepsilon}{\varepsilon} \  \frac{\lambda_{\max}(\CC^{\dagger}\bE)}{\lambda_{\max}(\HH)}.
\end{equation}
Overall, if the batch size satisfies this inequality then $(1-\varepsilon) \etaMean \leq \etaVar \leq \etaMean $.

\subsection{Proof that $ \lambda_{\max} \left( \CC^{\dagger} \bE \right) $ is nonnegative}\label{app:lambda_max(C^daggerE) is non-negative}
Since $ \etaVar \leq \etaMean $ for all values of $ B $ (see the beginning of this section), then $ \lambda_{\max}(\HH) \leq \lambda_{\max}(\CC^{\dagger}\DD) $. Therefore,
\begin{equation}
    0 \leq  \lambda_{\max}(\CC^{\dagger}\DD) - \lambda_{\max}(\HH) \leq  p \lambda_{\max} \left( \CC^{\dagger} \bE \right),
\end{equation}
where in the last step we used both ends of \eqref{eq:Gap first step}.

\section{Proof of Proposition~\ref{thm:stability expectation necessary}}\label{app:stability necessary proof}
The stability threshold given by Thm.~\ref{thm:stability threshold expectation} and Thm.~\ref{thm:stability threshold expectation regular minima} is 
\begin{equation}
    \etaVar =  \frac{2}{ \lambda_{\max}\left( \CC^{\dagger} \DD  \right) }
\end{equation}
where 
\begin{equation}
    \CC = \frac{1}{2} \HH \oplus \HH , \qquad
    \DD = (1-p)\, \HH \otimes \HH + p\,\frac{1}{n}\sum_{i=1}^n\HH_i \otimes \HH_i.
\end{equation}
This threshold corresponds to a necessary and sufficient condition for stability. Here we derive simplified necessary conditions for stability. In App.~\ref{app:stability threshold expectation proof} we show that (see \eqref{eq:StabilityConditionFirstMove})
\begin{equation}
    \frac{2}{ \lambda_{\max}\left( \CC^{\dagger} \DD  \right) } =  2 \inf_{\uu \in \Sb^{d^2-1} : \uu \notin \myNull{\DD} } \left\{ \frac{  \uu^\transpose  \CC \uu}{\uu^\transpose \DD \uu} \right\}.
\end{equation}
We shall upper bound the stability threshold by considering non-optimal yet interesting vectors $ \uu $.  Specifically, in the following we look at $ \uu = \vv_{\max} \otimes \vv_{\max} $,  where $\vv_{\max}$ is the top eigenvector of $ \HH $, and $ \uu = \mathrm{vec}(\Identity) $ to obtain the results of Proposition~\ref{thm:stability expectation necessary}.

\subsection{Setting $ \uu = \vv_{\max} \otimes \vv_{\max} $}\label{sec: evaluating v otimes v on Q}
Let $ \uu = \vv \otimes \vv \notin \myNull{\DD} $ where $ \| \vv \| = 1 $, then
\begin{align}
    \uu^\transpose \CC \uu
    & = \frac{1}{2} \uu^\transpose \left(\HH \otimes \Identity  + \Identity \otimes \HH \right)\uu \nonumber \\
    & = \frac{1}{2} \left[\left(\vv^\transpose \otimes \vv^\transpose\right) \left(\HH \otimes \Identity\right) \left(\vv \otimes \vv\right)  + \left(\vv^\transpose \otimes \vv^\transpose\right) \left(\Identity \otimes \HH\right) \left(\vv \otimes \vv\right) \right] \nonumber \\
    & = \frac{1}{2} \left[ \left(\vv^\transpose\HH \vv\right) \otimes \left( \vv^\transpose\vv \right)   + \left(\vv^\transpose\vv\right) \otimes \left( \vv^\transpose \HH \vv \right)  \right] \nonumber \\
    & = \frac{1}{2} \left[ \left( \vv^\transpose \HH \vv \right) \otimes 1   + 1 \otimes \left( \vv^\transpose \HH \vv \right)  \right] \nonumber \\
    & = \vv^\transpose \HH \vv .
\end{align}
Similarly,
\begin{align}
    \uu^\transpose \left(\HH \otimes \HH \right)\uu
    & = \left(\vv^\transpose \otimes \vv^\transpose\right) \left(\HH \otimes \HH\right) \left(\vv \otimes \vv\right) \nonumber \\
    & =  \left(\vv^\transpose\HH \vv\right) \otimes \left(\vv^\transpose\HH \vv\right) \nonumber \\
    & = \left( \vv^\transpose \HH \vv \right) \otimes \left( \vv^\transpose \HH \vv \right)\nonumber \\
    & = \left( \vv^\transpose \HH \vv \right)^2.
\end{align}
And again
\begin{align}
    \uu^\transpose \left(\frac{1}{n}\sum_{i=1}^n \HH_i \otimes \HH_i \right)\uu
    & = \frac{1}{n}\sum_{i=1}^n \left(\vv^\transpose \otimes \vv^\transpose\right) \left( \HH_i \otimes  \HH_i\right) \left(\vv \otimes \vv\right) \nonumber \\
    & = \frac{1}{n}\sum_{i=1}^n\left(\vv^\transpose \HH_i \vv\right) \otimes \left(\vv^\transpose \HH_i \vv\right) \nonumber \\
    & = \frac{1}{n}\sum_{i=1}^n \left(\vv^\transpose \HH_i \vv\right)^2 .
\end{align}
Thus,
\begin{align}
    \uu^\transpose \DD \uu
    & = (1-p)\uu^\transpose \left(\HH \otimes \HH \right)\uu + p\uu^\transpose \left(\frac{1}{n}\sum_{i=1}^n \HH_i \otimes \HH_i \right)\uu  \nonumber \\
    & = (1-p)\left( \vv^\transpose \HH \vv \right)^2 + p\frac{1}{n}\sum_{i=1}^n \left(\vv^\transpose \HH_i \vv\right)^2 \nonumber \\
    & = \left( \vv^\transpose \HH \vv \right)^2 + p\left[ \frac{1}{n}\sum_{i=1}^n \left(\vv^\transpose \HH_i \vv\right)^2 - \left( \vv^\transpose \HH \vv \right)^2\right] \nonumber \\
    & = \left( \vv^\transpose \HH \vv \right)^2 + p\left[ \frac{1}{n}\sum_{i=1}^n \left(\vv^\transpose \HH_i \vv\right)^2 -2 \left(\vv^\transpose \HH \vv\right)\left( \vv^\transpose \HH \vv \right)+\left( \vv^\transpose \HH \vv \right)^2\right] \nonumber \\
    & = \left( \vv^\transpose \HH \vv \right)^2 + p\frac{1}{n}\sum_{i=1}^n \left[ \left(\vv^\transpose \HH_i \vv\right)^2 -2 \left(\vv^\transpose \HH_i \vv\right)\left( \vv^\transpose \HH \vv \right)+\left( \vv^\transpose \HH \vv \right)^2\right] \nonumber \\
    & = \left( \vv^\transpose \HH \vv \right)^2 + p\frac{1}{n}\sum_{i=1}^n \left( \vv^\transpose \HH_i \vv-\left( \vv^\transpose \HH \vv \right)\right)^2.
\end{align}
Therefore, for general $ \uu = \vv \otimes \vv $ we get 
\begin{equation}\label{eq:necessary condition rank one general}
    \etaVar \leq 2 \frac{  \uu^\transpose  \CC \uu}{\uu^\transpose \DD \uu} =  \frac{ 2\vv^\transpose \HH \vv }{(\vv^\transpose \HH \vv)^2+\frac{p}{n} \sum_{i = 1}^n( \vv^\transpose \HH_i \vv - \vv^\transpose \HH \vv)^2 }.
\end{equation}
Specifically, for $ \uu = \vv_{\max} \otimes \vv_{\max} $ we get
\begin{equation}
    \etaVar \leq \frac{ 2 \lambda_{\max}(\HH) }{\lambda^2_{\max}(\HH) + p\frac{1}{n}\sum_{i=1}^n \left( \vv^\transpose_{\max} \HH_i \vv_{\max}-\lambda_{\max}(\HH)\right)^2}.
\end{equation}
Finally, from \eqref{eq:necessary condition rank one general} we get the following result which we used in Sec.~\ref{sec:Experiments}.
\begin{equation}
    \lambda_{\max}\left( \CC^{\dagger} \DD  \right) =  \frac{2}{\etaVar} \geq \vv^\transpose \HH \vv+p \frac{\frac{1}{n} \sum_{i = 1}^n ( \vv^\transpose \HH_i \vv - \vv^\transpose \HH \vv)^2 }{ \vv^\transpose \HH \vv } .
\end{equation}
Since this inequality holds for every $ \vv \notin \myNull{\HH} $, we can take the maximum to obtain \eqref{eq:optimized bound}.

\subsection{Setting $ \uu = \vectorization{\Identity}  $}
Let $ \uu = \mathrm{vec}(\Identity) \notin \myNull{\DD} $, then
\begin{align}
    \uu^\transpose \CC \uu
    & = \frac{1}{2} \uu^\transpose \left(\HH \otimes \Identity  + \Identity \otimes \HH \right)\uu \nonumber \\
    & = \frac{1}{2}\left( \left[\vectorization{\Identity}\right]^\transpose \left(\HH \otimes \Identity\right)\vectorization{\Identity}  + \left[\vectorization{\Identity}\right]^\transpose \left(\Identity \otimes \HH \right)\vectorization{\Identity} \right) \nonumber \\
    & = \frac{1}{2}\left( \Tr\big(\HH^\transpose\big) + \Tr\big(\HH\big) \right) \nonumber \\
    & = \Tr(\HH),
\end{align}
where in the third step we used \eqref{eq:quadratic form over kronecker}. Moreover, using \eqref{eq:quadratic form over kronecker} we have
\begin{equation}
    \uu^\transpose \left(\HH \otimes \HH \right)\uu
    = \left[\vectorization{\Identity}\right]^\transpose \left(\HH \otimes \HH \right) \vectorization{\Identity}
    = \Tr(\HH \HH^\transpose)
    = \normF{\HH}^2.
\end{equation}
Similarly,
\begin{equation}
    \uu^\transpose \left(\HH_i \otimes \HH_i \right)\uu
    = \left[\vectorization{\Identity}\right]^\transpose \left(\HH_i \otimes \HH_i \right) \vectorization{\Identity}
    = \Tr(\HH_i \HH_i^\transpose)
    = \normF{\HH_i}^2.
\end{equation}
Then,
\begin{align}
    \uu^\transpose \DD \uu
    & = (1-p)\uu^\transpose \left(\HH \otimes \HH \right)\uu + p \frac{1}{n}\sum_{i=1}^n \uu^\transpose \HH_i \otimes \HH_i \uu  \nonumber \\
    & = (1-p)\normF{\HH}^2 + p\frac{1}{n}\sum_{i=1}^n \normF{\HH_i}^2 .
\end{align}
Therefore,
\begin{equation}
    \etaVar \leq 2\frac{  \uu^\transpose  \CC \uu}{\uu^\transpose \DD \uu} = \frac{ 2 \Tr(\HH) }{(1-p)\normF{\HH}^2 + p\frac{1}{n}\sum_{i=1}^n \normF{\HH_i}^2 }.
\end{equation}

\section{Proof of Theorem~\ref{thm:covariance limit proposition}}\label{app:covariance limit proposition proof}
In this section, we use the following result on the Moore–Penrose inverse of a sum of two matrices.
\begin{theorem}[\citet{fill2000moore}, Thm.~3]\label{thm:pseudo inverse of matrix sum}
    Let $\XX, \YY \in \R^{p \times p }$ with $ \mathrm{rank}(\XX+\YY) = \mathrm{rank}(\XX)+\mathrm{rank}(\YY) $. Then
    \begin{equation}
        (\XX+\YY)^{\dagger} = (\Identity-\LL)\XX^{\dagger}(\Identity -\OO )+\LL \YY^{\dagger} \OO,
    \end{equation}
    where
    \begin{equation}
        \LL = \left( \PP_{\myRange{\YY^\transpose}} \PP_{\myRangeOrtho{\XX^\transpose}} \right)^{\dagger} \qquad \text{and} \qquad \OO = \left( \PP_{\myRangeOrtho{\XX}} \PP_{\myRange{\YY}}  \right)^{\dagger}.
    \end{equation}
\end{theorem}
Moreover, we use the following relations.
\begin{align}\label{eq:range of D and C}
    & \myRange{\DD} = \myRange{ \PP_{\myNullOrtho{\HH}} \otimes \PP_{\myNullOrtho{\HH}}  }, \nonumber \\
    & \myRange{\CC} = \myRange{ \Identity - \PP_{\myNull{\HH}} \otimes \PP_{\myNull{\HH}}  }, \nonumber \\
    & \myRange{ \PP_{\myNull{\DD}} \CC} = \myRange{\PP_{\myNull{\HH}} \otimes \PP_{\myNullOrtho{\HH}}+\PP_{\myNullOrtho{\HH}} \otimes \PP_{\myNull{\HH}}}.
\end{align}

The dynamics of $ \MeanVec_{t}^{\myPerp} $ and $\CovMat_{t}^{\myPerp}$ are given by (see \eqref{eq:projected joint dynamics})
\begin{equation}
    \begin{pmatrix}
        \MeanVec_{t+1}^{\myPerp} \\
        \vectorization{\CovMat_{t+1}^{\myPerp}}
    \end{pmatrix}
    =
    \bXi
    \begin{pmatrix}
        \MeanVec_{t}^{\myPerp} \\
        \vectorization{\CovMat_{t}^{\myPerp}}
    \end{pmatrix}
    +
    \begin{pmatrix}
        \zeroVec \\
        \vectorization{\CovMat_{\vv}^{\myPerp}}
    \end{pmatrix}.
\end{equation}
where
\begin{align}
    \bXi
    & = 
    \begin{pmatrix}
        \PP_{\myNullOrtho{\HH}} - \eta \HH  &  \zeroVec \\
        - \left(\PP_{\myNullOrtho{\HH}} \otimes \PP_{\myNullOrtho{\HH}} \right) (\E \left[ \vv_t^{\myPerp} \otimes \TransMat_t \right] + \E \left[ 	\TransMat_t \otimes \vv_t^{\myPerp} \right] )& \left(\PP_{\myNullOrtho{\HH}} \otimes \PP_{\myNullOrtho{\HH}} \right) \CovEvo
    \end{pmatrix}\nonumber \\
    & \triangleq
    \begin{pmatrix}
        \bXi_{1,1} & \bXi_{1,2} \\
        \bXi_{2,1} & \bXi_{2,2}
    \end{pmatrix}.
\end{align}
In App.~\ref{app:stability threshold expectation regular minima proof} we show that if $ 0< \eta < \etaVar $ then the spectral radius of $ \bXi $ is less then one. Therefore, the dynamical system is stable, and the asymptotic values of $ \MeanVec_{t}^{\myPerp} $ and $ \CovMat_{t}^{\myPerp} $ as $ t \to \infty $ are given by
\begin{equation}
	\lim_{t \to \infty} 
	\begin{pmatrix}
		\MeanVec_{t}^{\myPerp} \\
		\vectorization{\CovMat_{t}^{\myPerp}}
	\end{pmatrix}
	= \left( \Identity - \bXi \right)^{-1}
	\begin{pmatrix}
		\zeroVec \\
		\vectorization{\CovMat_{\vv}^{\myPerp}}
	\end{pmatrix}.
\end{equation}
Using the inversion formula for block matrix and the fact that $ \bXi_{1,2} = \zeroVec $ we have that
\begin{align}
	\left( \Identity - \bXi \right)^{-1}
	= & 
	\begin{pmatrix}
		\Identity-\bXi_{1,1} & -\bXi_{1,2} \\
		-\bXi_{2,1} & \Identity-\bXi_{2,2}
	\end{pmatrix}^{-1}\nonumber \\
	= & 
	\begin{pmatrix}
		\left(\Identity- \bXi_{1,1} - \bXi_{1,2}\left( \Identity- \bXi_{2,2}\right)^{-1} \bXi_{2,1} \right)^{-1}  & \zeroVec \\
		\zeroVec & \left(\Identity-\bXi_{2,2}- \bXi_{2,1}\left( \Identity- \bXi_{1,1} \right)^{-1}\bXi_{1,2} \right)^{-1}
	\end{pmatrix}\nonumber \\
	&
	\begin{pmatrix}
		\Identity & \bXi_{1,2} \left( \Identity-\bXi_{2,2} \right)^{-1} \\
		\bXi_{2,1}\left( \Identity-\bXi_{1,1} \right)^{-1} & \Identity
	\end{pmatrix}
	\nonumber \\
	= & \begin{pmatrix}
		(\Identity- \bXi_{1,1} )^{-1}  & \zeroVec \\
		\zeroVec & \left(\Identity- \bXi_{2,2} \right)^{-1}
	\end{pmatrix}
	\begin{pmatrix}
		\Identity & \zeroVec \\
		\bXi_{2,1}\left( \Identity-\bXi_{1,1} \right)^{-1} & \Identity
	\end{pmatrix}.
\end{align}
Therefore,
\begin{align}
	\lim_{t \to \infty} 
	\begin{pmatrix}
		\MeanVec_{t}^{\myPerp} \\
		\vectorization{\CovMat_{t}^{\myPerp}}
	\end{pmatrix}
	& = 
	\begin{pmatrix}
		(\Identity- \bXi_{1,1} )^{-1}  & \zeroVec \\
		\zeroVec & \left(\Identity- \bXi_{2,2} \right)^{-1}
	\end{pmatrix}
	\begin{pmatrix}
		\Identity & \zeroVec \\
		\bXi_{2,1}\left( \Identity-\bXi_{1,1} \right)^{-1} & \Identity
	\end{pmatrix}
	\begin{pmatrix}
		\zeroVec \\
		\vectorization{\CovMat_{\vv}^{\myPerp}}
	\end{pmatrix}\nonumber\\
	& = 
	\begin{pmatrix}
		(\Identity- \bXi_{1,1} )^{-1}  & \zeroVec \\
		\zeroVec & \left(\Identity- \bXi_{2,2} \right)^{-1}
	\end{pmatrix}
	\begin{pmatrix}
		\zeroVec \\
		\vectorization{\CovMat_{\vv}^{\myPerp}}
	\end{pmatrix}.
\end{align}
Namely,
\begin{equation}\label{eq:AsymptoticValue}
    \lim_{t \to \infty} \MeanVec_t^{\myPerp} = \zeroVec  \qquad \text{and} \qquad \lim_{t \to \infty}  \vectorization{\CovMat_{t}^{\myPerp}} = \left( \Identity - \PP_{\myNullOrtho{\DD}} \CovEvo  \right)^{-1} \vectorization{\CovMat_{\vv}^{\myPerp}}.
\end{equation}
Now,
\begin{align}
    \left( \Identity - \PP_{\myNullOrtho{\DD}} \CovEvo  \right)^{-1}
    & = \left( \Identity - \PP_{\myNullOrtho{\DD}} \CovEvo  \right)^{\dagger} \nonumber \\
    & = \left( \PP_{\myNull{\DD}}+\PP_{\myNullOrtho{\DD}} - \PP_{\myNullOrtho{\DD}} \CovEvo  \right)^{\dagger} \nonumber \\
    & = \left( \PP_{\myNull{\DD}}+\PP_{\myNullOrtho{\DD}} \left( \Identity - \CovEvo \right) \right)^{\dagger}.
\end{align}
Let us apply Thm.~\ref{thm:pseudo inverse of matrix sum} on $ ( \PP_{\DD}+\PP_{\myNullOrtho{\DD}} ( \Identity - \CovEvo ))^{\dagger} $. Here, $ \XX_1 = \PP_{\myNull{\DD}} $ and $ \YY_1 = \PP_{\myNullOrtho{\DD}} (\Identity - \CovEvo) $. Note that $\myRange{\XX_1} = \myRangeOrtho{\DD}$ and $ \myRange{\YY_1} = \myRange{\DD} $ and therefore $ \mathrm{rank}(\XX_1+\YY_1) = \mathrm{rank}(\XX_1)+\mathrm{rank}(\YY_1)  $. Additionally,
\begin{equation}
    \PP_{\myRange{\YY_1^\transpose}} = \PP_{\myNullOrtho{\DD}}, \quad  \PP_{\myRangeOrtho{\XX_1^\transpose}} = \PP_{\myNullOrtho{\DD}}, \quad  \PP_{\myRangeOrtho{\XX_1}} = \PP_{\myNullOrtho{\DD}},  \quad  \PP_{\myRange{\YY_1}} = \PP_{\myNullOrtho{\DD}}.
\end{equation}
Hence,
\begin{align}
    \LL_1 & = \left( \PP_{\myRange{\YY_1^\transpose}} \PP_{\myRangeOrtho{\XX_1^\transpose}} \right)^{\dagger} = \left(\PP_{\myNullOrtho{\DD}} \PP_{\myNullOrtho{\DD}} \right)^{\dagger} = \PP_{\myNullOrtho{\DD}},\nonumber \\
    \OO_1 & = \left( \PP_{\myRangeOrtho{\XX_1}} \PP_{\myRange{\YY_1}}  \right)^{\dagger} = \left(\PP_{\myNullOrtho{\DD}} \PP_{\myNullOrtho{\DD}} \right)^{\dagger} = \PP_{\myNullOrtho{\DD}}.
\end{align}
Therefore,
\begin{align}
    \left( \Identity - \PP_{\myNullOrtho{\DD}} \CovEvo  \right)^{-1}
    & = \left( \PP_{\myNull{\DD}}+\PP_{\myNullOrtho{\DD}} \left( \Identity - \CovEvo \right) \right)^{\dagger}  \nonumber \\
    & = (\XX_1+\YY_1)^{\dagger} \nonumber \\
    & = (\Identity-\LL_1)\XX_1^{\dagger}(\Identity -\OO_1 )+\LL_1 \YY_1^{\dagger} \OO_1 \nonumber \\
    & = (\Identity-\PP_{\myNullOrtho{\DD}})(\PP_{\myNull{\DD}})^{\dagger}(\Identity-\PP_{\myNullOrtho{\DD}}) \nonumber \\
    & \qquad +\PP_{\myNullOrtho{\DD}} \left( \PP_{\myNullOrtho{\DD}} (\Identity - \CovEvo) \right)^{\dagger} \PP_{\myNullOrtho{\DD}} \nonumber \\
    & = (\Identity-\PP_{\myNullOrtho{\DD}})\PP_{\myNull{\DD}}(\Identity-\PP_{\myNullOrtho{\DD}}) \nonumber \\
    & \qquad +\PP_{\myNullOrtho{\DD}} \left( \PP_{\myNullOrtho{\DD}} (\Identity - \CovEvo) \right)^{\dagger} \PP_{\myNullOrtho{\DD}} \nonumber \\
    & = \PP_{\myNull{\DD}}+\PP_{\myNullOrtho{\DD}} \left( \PP_{\myNullOrtho{\DD}} (\Identity - \CovEvo) \right)^{\dagger} \PP_{\myNullOrtho{\DD}},
\end{align}
where in the third step we used Thm.~\ref{thm:pseudo inverse of matrix sum}. Thus we get the following intermediate result
\begin{align}
    \lim_{t \to \infty}  \vectorization{\CovMat_{t}^{\myPerp}}
    & = \left( \Identity - \PP_{\myNullOrtho{\DD}} \CovEvo  \right)^{-1} \vectorization{\CovMat_{\vv}^{\myPerp}} \nonumber \\
    & = \left(\PP_{\myNull{\DD}}+\PP_{\myNullOrtho{\DD}} \left( \PP_{\myNullOrtho{\DD}} (\Identity - \CovEvo) \right)^{\dagger} \PP_{\myNullOrtho{\DD}} \right)\vectorization{\CovMat_{\vv}^{\myPerp}} \nonumber \\
    & = \PP_{\myNullOrtho{\DD}} \left( \PP_{\myNullOrtho{\DD}} (\Identity - \CovEvo) \right)^{\dagger} \PP_{\myNullOrtho{\DD}} \vectorization{\CovMat_{\vv}^{\myPerp}},
\end{align}
where in the final step we used $ \PP_{\myNull{\DD}} \mathrm{vec}(\CovMat_{\vv}^{\myPerp}) = \zeroVec $. Now, note that $ \myRange{ \PP_{\myNull{\DD}} \CC} = \myRange{\PP_{\myNull{\HH}} \otimes \PP_{\myNullOrtho{\HH}}+\PP_{\myNullOrtho{\HH}} \otimes \PP_{\myNull{\HH}}} $, whereas $ \mathrm{vec}(\CovMat_{\vv}^{\myPerp}) \in \myRange{\DD} = \myRange{ \PP_{\myNullOrtho{\HH}} \otimes \PP_{\myNullOrtho{\HH}}} $ and therefore $ (\PP_{\myNull{\DD}} \CC)^{\dagger} \mathrm{vec}(\CovMat_{\vv}^{\myPerp}) = \zeroVec $. Hence,
\begin{align}
    \lim_{t \to \infty}  \vectorization{\CovMat_{t}^{\myPerp}}
    & = \PP_{\myNullOrtho{\DD}} \left( \PP_{\myNullOrtho{\DD}} (\Identity - \CovEvo) \right)^{\dagger} \PP_{\myNullOrtho{\DD}} \vectorization{\CovMat_{\vv}^{\myPerp}} \nonumber \\
    & = \left(  \left( 2 \eta \PP_{\myNull{\DD}}\CC\right)^{\dagger} +\PP_{\myNullOrtho{\DD}} \left( \PP_{\myNullOrtho{\DD}} (\Identity - \CovEvo) \right)^{\dagger} \PP_{\myNullOrtho{\DD}} \right) \vectorization{\CovMat_{\vv}^{\myPerp}}.
\end{align}
Let us apply again Thm.~\ref{thm:pseudo inverse of matrix sum} but in the other direction. This time, $ \XX_2 = 2 \eta \PP_{\myNull{\DD}}\CC $ and $ \YY_2 = \PP_{\myNullOrtho{\DD}} (\Identity - \CovEvo) $. Note that $\myRange{\XX_2} = \myRange{\PP_{\myNull{\HH}} \otimes \PP_{\myNullOrtho{\HH}}+\PP_{\myNullOrtho{\HH}} \otimes \PP_{\myNull{\HH}}}$ and $ \myRange{\YY_2} = \myRange{\PP_{\myNullOrtho{\HH}} \otimes \PP_{\myNullOrtho{\HH}}} $ and therefore $ \mathrm{rank}(\XX_2+\YY_2) = \mathrm{rank}(\XX_2)+\mathrm{rank}(\YY_2)  $. Additionally,
\begin{align}
     & \PP_{\myRange{\YY_2^\transpose}} = \PP_{\myNullOrtho{\HH}} \otimes \PP_{\myNullOrtho{\HH}}, \nonumber \\
     & \PP_{\myRange{\YY_2}} = \PP_{\myNullOrtho{\HH}} \otimes \PP_{\myNullOrtho{\HH}}, \nonumber \\
     & \PP_{\myRangeOrtho{\XX_2^\transpose}} = \PP_{\myNull{\HH}} \otimes \PP_{\myNull{\HH}}+\PP_{\myNullOrtho{\HH}} \otimes \PP_{\myNullOrtho{\HH}}, \nonumber \\
     & \PP_{\myRangeOrtho{\XX_2}} = \PP_{\myNull{\HH}} \otimes \PP_{\myNull{\HH}}+\PP_{\myNullOrtho{\HH}} \otimes \PP_{\myNullOrtho{\HH}}.
\end{align}
Hence,
\begin{align}
    \LL_2
    & = \left( \PP_{\myRange{\YY_2^\transpose}} \PP_{\myRangeOrtho{\XX_2^\transpose}} \right)^{\dagger} \nonumber \\
    & = \left(\PP_{\myNullOrtho{\HH}} \otimes \PP_{\myNullOrtho{\HH}} \left( \PP_{\myNull{\HH}} \otimes \PP_{\myNull{\HH}}+\PP_{\myNullOrtho{\HH}} \otimes \PP_{\myNullOrtho{\HH}} \right)  \right)^{\dagger}  \nonumber \\
    & = \PP_{\myNullOrtho{\HH}} \otimes \PP_{\myNullOrtho{\HH}}  \nonumber \\
    & =\PP_{\myNullOrtho{\DD}},\nonumber \\
    \OO_2
    & = \left( \PP_{\myRangeOrtho{\XX_2}} \PP_{\myRange{\YY_2}}  \right)^{\dagger} \nonumber \\
    & = \left(\left( \PP_{\myNull{\HH}} \otimes \PP_{\myNull{\HH}}+\PP_{\myNullOrtho{\HH}} \otimes \PP_{\myNullOrtho{\HH}} \right) \PP_{\myNullOrtho{\HH}} \otimes \PP_{\myNullOrtho{\HH}}  \right)^{\dagger} \nonumber \\
    & = \PP_{\myNullOrtho{\HH}} \otimes \PP_{\myNullOrtho{\HH}} = \PP_{\myNullOrtho{\DD}}.
\end{align}
Moreover, since $\myRange{\XX_2} = \myRange{\PP_{\myNull{\HH}} \otimes \PP_{\myNullOrtho{\HH}}+\PP_{\myNullOrtho{\HH}} \otimes \PP_{\myNull{\HH}}}$ and $ \myRangeOrtho{\YY_2} = \myRange{\PP_{\myNull{\HH}} \otimes \PP_{\myNullOrtho{\HH}}+\PP_{\myNullOrtho{\HH}} \otimes \PP_{\myNull{\HH}}+\PP_{\myNullOrtho{\HH}} \otimes \PP_{\myNullOrtho{\HH}}} $ we have that $ \myRange{\XX_2} \subseteq \myRangeOrtho{\YY_2} = \myNull{\DD} $. Therefore,
\begin{equation}\label{eq:x_2 projection}
    (\Identity-\LL_2)\XX_2^{\dagger}(\Identity -\OO_2 ) = \left( \Identity - \PP_{\myNullOrtho{\DD}} \right) \XX_2^{\dagger}\left( \Identity - \PP_{\myNullOrtho{\DD}} \right) = \PP_{\myNull{\DD}} \XX_2^{\dagger} \PP_{\myNull{\DD}} = \XX_2^{\dagger}.
\end{equation}
Therefore, applying Thm.~\ref{thm:pseudo inverse of matrix sum} we get
\begin{align}
    \left( 2 \eta \PP_{\myNull{\DD}}\CC\right)^{\dagger} +\PP_{\myNullOrtho{\DD}} & \left( \PP_{\myNullOrtho{\DD}} (\Identity - \CovEvo) \right)^{\dagger} \PP_{\myNullOrtho{\DD}} \nonumber \\
    & = \XX_2^{\dagger}+\LL_2 \YY_2^{\dagger} \OO_2 \nonumber \\
    & = (\Identity-\LL_2)\XX_2^{\dagger}(\Identity -\OO_2 )+\LL_2 \YY_2^{\dagger} \OO_2 \nonumber \\
    & = (\XX_2+\YY_2)^{\dagger} \nonumber \\
    & = \left( 2 \eta \PP_{\myNull{\DD}} \CC+\PP_{\myNullOrtho{\DD}} \left( \Identity - \CovEvo \right) \right)^{\dagger}  \nonumber \\
    & = \left( 2 \eta \PP_{\myNull{\DD}} \CC+ 2 \eta  \PP_{\myNullOrtho{\DD}}\CC - \eta^2 \PP_{\myNullOrtho{\DD}} \DD \right)^{\dagger}  \nonumber \\
    & = \left( 2 \eta\CC - \eta^2 \DD \right)^{\dagger}.
\end{align}
where in the second step we used \eqref{eq:x_2 projection}, and in the third step we used Thm.~\ref{thm:pseudo inverse of matrix sum}. Overall, together with \eqref{eq:covariance of v} we get
\begin{align}
    \lim_{t \to \infty}  \vectorization{\CovMat_{t}^{\myPerp}}
    & = \left( 2 \eta\CC - \eta^2 \DD \right)^{\dagger}  \vectorization{\CovMat_{\vv}^{\myPerp}} \nonumber \\
    & = \left( \frac{1}{\eta}\left(2\CC - \eta \DD \right)^{\dagger} \right) \left( \eta^2 p  \ \vectorization{\CovMat_{\grad}^{\myPerp}} \right) \nonumber \\
    & = \eta p \left(2\CC - \eta \DD \right)^{\dagger} \vectorization{\CovMat_{\grad}^{\myPerp}}.
\end{align}

\section{Proof of Corollary~\ref{thm:covariance limit}}\label{app:covariance limit proof}
From Thm.~\ref{thm:covariance limit proposition} we have that if $ 0 < \eta < \etaVar $ then
\begin{equation}
    \lim_{t \to \infty}  \vectorization{\CovMat_{t}^{\myPerp}} = \eta p \left(2\CC - \eta \DD \right)^{\dagger} \vectorization{\CovMat_{\grad}^{\myPerp}},
\end{equation}
Using this result, we prove Corollary~\ref{thm:covariance limit}.

\paragraph{First statement.} If $ 0<\eta < \etaVar $ then by Prop.~\ref{thm:covariance limit proposition}
\begin{align}
    \lim_{t \to \infty} \E\left[ \left\| \params_t^\myPerp -\params^{*\myPerp}  \right\|^2 \right]
    & = (\vectorization{\Identity})^\transpose  \lim_{t \to \infty} \vectorization{\CovMat_{t}^{\myPerp}} \nonumber\\
    & = (\vectorization{\Identity})^\transpose  \left( \eta p \left(2\CC - \eta \DD \right)^{\dagger} \vectorization{\CovMat_{\grad}^{\myPerp}} \right) \nonumber\\
    & = \eta p (\vectorization{\Identity})^\transpose \big( 2 \CC -\eta \DD \big)^{\dagger} \vectorization{\CovMat_{\grad}^{\myPerp}}.
\end{align}

\paragraph{Second statement.} Similarly, let us compute the limit of the expected value of the loss function to  obtain point $2$. 
\begin{align}
    \lim_{t \to \infty} \E \left[\lossApprox(\params_t) \right] - \loss(\params^*)
    & = \frac{1}{2} \lim_{t \to \infty} \E \left[ (\params_t-\params^* )^\transpose \HH (\params_t - \params^*) \right] \nonumber\\
    & = \frac{1}{2} \lim_{t \to \infty} \E \left[ (\params_t-\params^* )^\transpose \PP_{\myNullOrtho{\HH}} \HH \PP_{\myNullOrtho{\HH}} (\params_t - \params^*) \right] \nonumber\\
    & = \frac{1}{2} \lim_{t \to \infty} \E \left[ (\params_t^{\myPerp}-\params^{*\myPerp})^\transpose \HH (\params_t^{\myPerp}-\params^{*\myPerp}) \right] \nonumber\\
    & = \frac{1}{2} \Tr \left( \HH \lim_{t \to \infty} \E \left[(\params_t^{\myPerp}-\params^{*\myPerp})(\params_t^{\myPerp}-\params^{*\myPerp})^\transpose  \right] \right) \nonumber\\
    & = \frac{1}{2} \Tr \left( \HH \lim_{t \to \infty} \CovMat_{t}^{\myPerp} \right) \nonumber\\
    & = \frac{1}{2} \left(\vectorization{\HH}\right)^\transpose \lim_{t \to \infty} \vectorization{\CovMat_{t}^{\myPerp}} \nonumber\\
    & = \frac{1}{2} (\vectorization{\HH})^\transpose  \left( \eta p \left(2\CC - \eta \DD \right)^{\dagger} \vectorization{\CovMat_{\grad}^{\myPerp}} \right) \nonumber\\
    & = \frac{1}{2} \eta p (\vectorization{\HH})^\transpose \big( 2 \CC -\eta \DD \big)^{\dagger} \vectorization{\CovMat_{\grad}^{\myPerp}}.
\end{align}

\paragraph{Third statement.} 
Finally, we prove point $3$. The gradient of the second-order Taylor expansion of the loss is given by
\begin{equation}
    \nabla \lossApprox(\params) = \HH \left( \params -\params^* \right).
\end{equation}
Therefore
\begin{align}
    \lim_{t \to \infty} \E\left[ \norm{\nabla \lossApprox (\params_t) }^2 \right]
    & = \lim_{t \to \infty} \E \left[ (\params_t-\params^* )^\transpose \HH^2 (\params_t - \params^*) \right] \nonumber\\
    & = \lim_{t \to \infty} \E \left[ (\params_t-\params^* )^\transpose \PP_{\myNullOrtho{\HH}} \HH^2 \PP_{\myNullOrtho{\HH}} (\params_t - \params^*) \right] \nonumber\\
    & = \lim_{t \to \infty} \E \left[ (\params_t^{\myPerp}-\params^{*\myPerp})^\transpose \HH^2 (\params_t^{\myPerp}-\params^{*\myPerp}) \right] \nonumber\\
    & = \Tr \left( \HH^2 \lim_{t \to \infty} \E \left[(\params_t^{\myPerp}-\params^{*\myPerp})(\params_t^{\myPerp}-\params^{*\myPerp})^\transpose  \right] \right) \nonumber\\
    & = \Tr \left( \HH^2 \lim_{t \to \infty} \CovMat_{t}^{\myPerp} \right) \nonumber\\
    & = \left(\vectorization{\HH^2}\right)^\transpose \lim_{t \to \infty} \vectorization{\CovMat_{t}^{\myPerp}} \nonumber\\
    & = (\vectorization{\HH^2})^\transpose  \left( \eta p \left(2\CC - \eta \DD \right)^{\dagger} \vectorization{\CovMat_{\grad}^{\myPerp}} \right) \nonumber\\
    & = \eta p (\vectorization{\HH^2})^\transpose \big( 2 \CC -\eta \DD \big)^{\dagger} \vectorization{\CovMat_{\grad}^{\myPerp}}.
\end{align}

\section{Recovering GD's stability condition}\label{app:Recovering GD's stability condition}
In this section, we show how our stability condition for SGD reduces to GD's when $ B = n $. In this case $ p=0 $ and thus
\begin{equation}
    \CC = \frac{1}{2} \HH \oplus \HH , \qquad
    \DD =  \HH \otimes \HH.
\end{equation}
Let $ \HH = \VV \bLambda \VV^\transpose $ be the eigenvalue decomposition of $ \HH $, where $ \VV \VV^\transpose =  \VV^\transpose \VV = \Identity$, then
\begin{align}\label{eq:C decomposition}
    \CC 
    & = \frac{1}{2} \HH \oplus \HH \nonumber \\
    & = \frac{1}{2} \left(\HH \otimes \Identity +  \HH \otimes \Identity \right)  \nonumber \\
    & = \frac{1}{2} \left(\left(\VV \bLambda \VV^\transpose\right) \otimes \left(\VV \VV^\transpose\right) +  \left(\VV \VV^\transpose\right) \otimes \left(\VV \bLambda \VV^\transpose\right) \right)  \nonumber \\
    & = \frac{1}{2} \left( 
    \left(\VV \otimes \VV \right) \left( \bLambda \otimes \Identity \right) \left(\VV^\transpose \otimes \VV^\transpose\right) +
    \left(\VV \otimes \VV \right) \left( \Identity \otimes \bLambda \right) \left(\VV^\transpose \otimes \VV^\transpose\right)
    \right)  \nonumber \\
    & = \left(\VV \otimes \VV \right) \left( \frac{1}{2} \bLambda \otimes \Identity +  \frac{1}{2}  \Identity \otimes \bLambda 
    \right) \left(\VV \otimes \VV\right)^\transpose.
\end{align}
Note that
\begin{equation}
    \left(\VV \otimes \VV \right)^\transpose\left(\VV \otimes \VV \right) = \left(\VV^\transpose \otimes \VV^\transpose \right)\left(\VV \otimes \VV \right) = \left(\VV^\transpose \VV\right)\otimes \left(\VV^\transpose \VV\right) = \Identity \otimes \Identity = \Identity,
\end{equation}
\ie $ (\VV \otimes \VV ) $ is an orthogonal matrix. Since $\frac{1}{2} (\bLambda \otimes \Identity +  \Identity \otimes \bLambda)  $ is diagonal, then the last result in \eqref{eq:C decomposition} is an eigenvalue decomposition of $ \CC $. Similarly,
\begin{align}\label{eq:D decomposition}
    \DD
    & = \HH \otimes \HH \nonumber \\
    & = \left(\VV \bLambda \VV^\transpose\right) \otimes \left(\VV \bLambda \VV^\transpose\right) \nonumber \\
    & =  
    \left(\VV \otimes \VV \right) \left( \bLambda \otimes \bLambda \right) \left(\VV^\transpose \otimes \VV^\transpose\right) \nonumber \\
    & = \left(\VV \otimes \VV \right) \left( \bLambda \otimes \bLambda 
    \right) \left(\VV \otimes \VV\right)^\transpose,
\end{align}
where the last result here is the eigenvalue decomposition of $\DD$. We have that $\CC $ and $ \DD $ have the same set of eigenvectors, given by $ \VV \otimes \VV $. This means that we can look only at the eigenvalues. Thus, set $ \lambda_{\ell} = \bLambda_{[\ell,\ell]} = \lambda_{\ell}(\HH) $, and define the Moore–Penrose inverse for scalars
\begin{equation}
    \forall x \in \R \qquad [x]^{\dagger} \triangleq
    \begin{cases}
        \frac{1}{x}, & x \neq 0,\\
        0, & x = 0.
    \end{cases}
\end{equation}
Then
\begin{equation}
    \lambda_{\max} \left(\CC^{\dagger} \DD  \right) = \max_{\ell,p \in [d]}
    \left\{ \lambda_{\ell}\lambda_{p} \left[\frac{1}{2}(\lambda_{\ell}+\lambda_{p})\right]^{\dagger} \right\}.
\end{equation}
Note that the objective vanishes whenever $ \lambda_{\ell} = 0 $ or $ \lambda_{p} = 0 $. Restricting to only positive eigenvalues gives 
\begin{equation}
    \lambda_{\max} \left(\CC^{\dagger} \DD \right) = \max_{\lambda_\ell, \lambda_p >0 } \left\{ \frac{\lambda_{\ell}\lambda_{p}}{\frac{1}{2}(\lambda_{\ell}+\lambda_{p})} \right\} .
\end{equation}
Additionally $ \sqrt{\lambda_{\ell}\lambda_{p}} \leq \frac{1}{2}(\lambda_{\ell}+\lambda_{p}) $ holds for all $ \lambda_\ell, \lambda_p >0 $, therefore
\begin{align}
    \frac{\lambda_{\ell}\lambda_{p}}{\frac{1}{2}(\lambda_{\ell}+\lambda_{p})} 
    & = \sqrt{\lambda_{\ell}\lambda_{p}} \frac{ \sqrt{\lambda_{\ell}\lambda_{p}} }{\frac{1}{2}(\lambda_{\ell}+\lambda_{p})} \nonumber \\
    & \leq \sqrt{\lambda_{\ell}\lambda_{p}} \nonumber \\
    & \leq \lambda_{\max}.
\end{align}
Yet for $  \lambda_{\ell} = \lambda_{p} = \lambda_{\max} $ we have that
\begin{equation}
    \frac{\lambda_{\ell}\lambda_{p}}{\frac{1}{2}(\lambda_{\ell}+\lambda_{p})} = \lambda_{\max}.
\end{equation}
Hence we have
\begin{equation}
    \lambda_{\max} \left(\CC^{\dagger} \DD  \right) = \max_{\lambda_\ell, \lambda_p >0}
    \left\{ \frac{\lambda_{\ell}\lambda_{p}}{\frac{1}{2}(\lambda_{\ell}+\lambda_{p})} 
    \right\} = \lambda_{\max}(\HH).
\end{equation}

\section{Additional experimental results and detail}\label{App:Additional experimental results}
\begin{figure}[t]%
    \centering
    \rotatebox[origin=l]{90}{\hspace{1.5cm} (Generalized) Sharpness}%
    % \subfigure[$ B = 8 $][b]{%
    % \includegraphics[trim=0 0 0 0.1in, clip, width=0.4\linewidth]{Figures/SGD_Stability_MNIST_Final_4.pdf}}%
    % \subfigure[$ B = 16 $][b]{%
    % \includegraphics[trim=0 0 0 0.1in, clip, width=0.4\linewidth]{Figures/SGD_Stability_MNIST_Final_5.pdf}}\\%
    % \rotatebox[origin=l]{90}{\hspace{1.5cm} (Generalized) Sharpness}%
    % \subfigure[$ B = 32 $][b]{%
    % \includegraphics[trim=0 0 0 0.1in, clip, width=0.4\linewidth]{Figures/SGD_Stability_MNIST_Final_6.pdf}}%
    \subfigure[$ B = 64 $][b]{%
    \includegraphics[trim=0 0 0 0.10in, clip, width=0.4\linewidth]{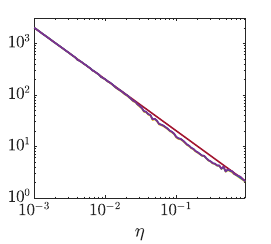}}%
    % \rotatebox[origin=l]{90}{\hspace{1.5cm} (Generalized) Sharpness}%
    \subfigure[$ B = 128 $][b]{%
    \includegraphics[trim=0 0 0 0.10in, clip, width=0.4\linewidth]{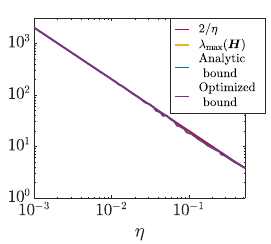}}%
    \caption{\textbf{Sharpness vs. learning rate.} Additional results for the experiment in Sec.~\ref{sec:Experiments}. These two figures complete the results of Fig.~\ref{Fig:A}. Here we see that SGD with big batch sizes behaves like GD.}%
    \label{Fig:Appendix}%
\end{figure}

In this section, we complete the technical detail of the experiment shown in Sec.~\ref{sec:Experiments}. For the experiment, we used a single-hidden layer ReLU network with fully connected layers (with bias vectors). The number of neurons is $ 1024 $, and the total number of parameters is $ 807,940 $. We used four classes from MNIST, $ 256 $ samples from each class, with a total of $1024$ samples. To get large initialization, we used standard torch initialization and multiplied the initial weights by a factor of $15$. The maximal number of epochs was set to $4\times 10^{4}$. If SGD did not converge within this number of epochs, then we removed this run from the plots. 

\end{document}